\documentclass{article}

\usepackage[final]{neurips_2023}

\usepackage[utf8]{inputenc} %
\usepackage[T1]{fontenc}    %
\usepackage{hyperref}       %
\usepackage{url}            %
\usepackage{booktabs}       %
\usepackage{amsfonts}       %
\usepackage{nicefrac}       %
\usepackage{microtype}      %
\usepackage{xcolor}         %
\usepackage{algorithm}
\usepackage[noend]{algorithmic}
\usepackage{selectp}
\usepackage{colortbl}
\usepackage{chngcntr}

\usepackage{multirow}

\usepackage{color}
\usepackage{amsmath}
\usepackage{amssymb}
\usepackage{amsthm}
\usepackage{mathtools}
\usepackage{thmtools}
\usepackage[mathscr]{eucal}
\usepackage{bm}
\usepackage{bbm}
\usepackage{relsize}
\usepackage{graphics}
\usepackage{wrapfig}
\usepackage{multirow}
\usepackage{enumitem} 

\usepackage[algo2e]{algorithm2e}

\usepackage[capitalize,noabbrev]{cleveref}

\usepackage{array}

\usepackage{pgf}
\usepackage{xinttools}
\usepackage{tikz}
\usetikzlibrary{arrows.meta}
\usepackage{textcomp}
\usetikzlibrary{calc,shapes,arrows,positioning,automata,trees}
\usepackage{placeins} 
\usepackage{tabularx}
\usepackage{graphicx}
\usepackage{subcaption} 
\usepackage{listings}
\usepackage{xargs,lipsum,caption,changepage,ifthen}

\usepackage{tikz} 

\usepackage{footmisc}

\usepackage{bigdelim}
\usepackage{colortbl} 
\definecolor{mColor1}{rgb}{0.95,0.95,0.95}
\newcolumntype{a}{>{\columncolor{mColor1}}c}

\usepackage[toc,page]{appendix}

\usepackage{tcolorbox}

\definecolor{solarized@base03}{HTML}{002B36}
\definecolor{solarized@base02}{HTML}{073642}
\definecolor{solarized@base01}{HTML}{586e75}
\definecolor{solarized@base00}{HTML}{657b83}
\definecolor{solarized@base0}{HTML}{839496}
\definecolor{solarized@base1}{HTML}{93a1a1}
\definecolor{solarized@base2}{HTML}{EEE8D5}
\definecolor{solarized@base3}{HTML}{FDF6E3}
\definecolor{solarized@yellow}{HTML}{B58900}
\definecolor{solarized@orange}{HTML}{CB4B16}
\definecolor{solarized@red}{HTML}{DC322F}
\definecolor{solarized@magenta}{HTML}{D33682}
\definecolor{solarized@blue}{HTML}{268BD2}
\definecolor{solarized@cyan}{HTML}{2AA198}
\definecolor{solarized@green}{HTML}{859900}

\newtcolorbox{importantresult}{colback=solarized@yellow!5!white,
colframe=solarized@yellow,parbox, left=0.5mm, right=0.5mm,top=0.5mm,bottom=0.5mm}

\newtcolorbox{importantresult_noparbox}{colback=solarized@yellow!5!white,
colframe=solarized@yellow,parbox=false, left=0.5mm, right=0.5mm,top=0.5mm,bottom=0.5mm}

\newtheorem{theorem}{Theorem}
\newtheorem{definition}{Definition}
\newtheorem{corollary}{Corollary}%
\newtheorem*{theorem*}{Theorem}
\newtheorem*{definition*}{Definition}

\newcommand\Bp{\bm{p}}
\newcommand\Bq{\bm{q}}

\newcommand\Bw{\bm{w}}
\newcommand\Bx{\bm{x}}
\newcommand\By{\bm{y}}

\newcommand\Bde{\bm{\delta}}

\newcommand\Bmu{\bm{\mu}}

\newcommand\Bth{\bm{\theta}}

\newcommand\BZo{\bm{0}}

 \newcommand{\rD}{\mathrm{D}}

 \newcommand{\rL}{\mathrm{L}}

 \newcommand{\cD}{\mathcal{D}}

 \newcommand{\cL}{\mathcal{L}}
 \newcommand{\cN}{\mathcal{N}}
 \newcommand{\cP}{\mathcal{P}}

\newcommand{\cW}{\mathcal{W}} \newcommand{\cX}{\mathcal{X}}
\newcommand{\cY}{\mathcal{Y}}

 \newcommand{\Rd}{\mathrm{d}}

\newcommand\EXP{\mathbf{\mathrm{E}}}
\newcommand\PR{\mathbf{\mathrm{Pr}}}

\DeclarePairedDelimiterX{\kldiv}[2]{(}{)}{%
  #1\,\delimsize\|\,#2%
}
\newcommand{\KL}{\mathrm{D}_\mathrm{KL}\kldiv}

\DeclarePairedDelimiterX{\mi}[2]{[}{]}{%
  #1\,\delimsize ; \,#2%
}
\newcommand{\MI}{\mathrm{I}\mi}

\DeclarePairedDelimiterX{\di}[2]{(}{)}{%
  #1\,\delimsize , \,#2%
}
\newcommand{\DI}{\mathrm{D}\di}

\DeclarePairedDelimiterX{\ce}[2]{[}{]}{%
  #1\,\delimsize , \,#2%
}
\newcommand{\CE}{\mathrm{CE}\ce}

\DeclarePairedDelimiter{\ent}{[}{]}
\newcommand{\ENT}{\mathrm{H}\ent}

\DeclarePairedDelimiterXPP{\mii}[3]%
   {_{\mathrm{#1}}}{(}{)}{}{#2\;\delimsize ; \;#3%
}

\renewcommand{\leq}{\leqslant}

\makeatletter
\newcommand{\dlmf}[1]{%
\citep[%
  \def\nextitem{\def\nextitem{, }}%
  \@for \el:=#1\do{\nextitem\href{http://dlmf.nist.gov/\el}{(\el)}}%
]{Olver:10}%
}
\makeatother

\usepackage{pdflscape} 

\newcolumntype{R}[1]{>{\raggedright\arraybackslash}p{#1}}
\newcolumntype{C}[1]{>{\centering\arraybackslash}p{#1}}
\newcolumntype{L}[1]{>{\raggedleft\arraybackslash}p{#1}}

\usepackage{bigdelim}
\usepackage{colortbl} 
\definecolor{mColor1}{rgb}{0.95,0.95,0.95}

\hypersetup{
   colorlinks=true,
   linkcolor=[RGB]{30, 30, 180},
   citecolor=[RGB]{30, 30, 180},
   urlcolor=magenta,
   pdfborder=0 0 0,
   pdftitle={},
   pdfsubject={}, 
   pdfkeywords={},
   pdfauthor={},%
   pdfstartview=FitH
}

\title{Quantification of Uncertainty with Adversarial Models}

\author{%
  Kajetan Schweighofer\footnotemark[1] \,\; Lukas Aichberger\footnotemark[1] \,\; Mykyta Ielanskyi\footnotemark[1] \\ {\bf Günter Klambauer \; Sepp Hochreiter} \\
  \\
  ~ELLIS Unit Linz and LIT AI Lab, Institute for Machine Learning, \\ 
                  ~~Johannes Kepler University Linz, Austria \\ 
  \footnotemark[1]~~Joint first authors\\
}

\begin{document}

\maketitle
\addtocontents{toc}{\protect\setcounter{tocdepth}{-1}} %

\begin{abstract}
  Quantifying uncertainty is important for actionable predictions in real-world applications.
  A crucial part of predictive uncertainty quantification is the estimation of epistemic uncertainty,
  which is defined as an integral of the product between a divergence function and the posterior.
  Current methods
  such as Deep Ensembles or MC dropout underperform at estimating the
  epistemic uncertainty, since they primarily consider the posterior when sampling models.
  We suggest Quantification of Uncertainty with Adversarial Models (QUAM) 
  to better estimate the epistemic uncertainty. 
  QUAM identifies regions where
  the whole product under the integral is large, not just the posterior.
  Consequently, QUAM
  has lower approximation error of the epistemic uncertainty compared to previous methods.
  Models for which the product is large correspond to
  adversarial models (not adversarial examples!).
  Adversarial models have both a high posterior as well as 
  a high divergence between their predictions 
  and that of a reference model.
  Our experiments 
  show that QUAM excels in capturing epistemic uncertainty for deep learning models
  and outperforms previous methods on challenging tasks in the vision domain.
\end{abstract}

\section{Introduction}
Actionable predictions
typically require risk assessment
based on predictive uncertainty quantification \citep{Apostolakis:90}.
This is of utmost importance in high stake applications,
such as medical diagnosis or drug discovery,
where human lives or extensive investments are at risk.
In such settings, even a single prediction has far-reaching 
real-world impact, thus necessitating the most precise quantification of the associated uncertainties.
Furthermore, foundation models or specialized models that are obtained externally
are becoming increasingly prevalent, also in high stake applications.
It is crucial to assess the robustness and reliability of 
those unknown models before applying them.
Therefore, the predictive uncertainty of given, pre-selected models at specific test points should be quantified,
which we address in this work.

We consider predictive uncertainty quantification (see Fig.~\ref{fig:adverarialModels}) 
for deep neural networks \citep{Gal:16thesis,Huellermeier:21}.
According to \cite{Vesely:84,Apostolakis:90,Helton:93,McKone:94,Helton:97},
predictive uncertainty can be categorized into two types.
First, \emph{aleatoric} (Type A, variability, stochastic, true, irreducible) uncertainty 
refers to the variability when drawing samples or when repeating
the same experiment.
Second, \emph{epistemic} (Type B, lack of knowledge, subjective, reducible) uncertainty 
refers to the lack of knowledge about the true model.
Epistemic uncertainty can result from
imprecision in parameter estimates, 
incompleteness in modeling, or
indefiniteness in the applicability of the model.
While aleatoric uncertainty cannot be reduced, 
epistemic uncertainty can be reduced by more data, 
better models, or more knowledge about the problem.
We follow \cite{Helton:97} and consider epistemic uncertainty 
as the imprecision or 
variability of parameters that determine the predictive distribution.
\cite{Vesely:84} calls this epistemic uncertainty ''parameter
uncertainty'',
which results from an imperfect learning algorithm or from insufficiently many
training samples.
Consequently, we consider
predictive uncertainty quantification as characterizing a probabilistic model of the world.
In this context, aleatoric uncertainty 
refers to the inherent stochasticity of
sampling outcomes from the predictive distribution of the model and
epistemic uncertainty refers to the uncertainty about model parameters.

\begin{figure}
\begin{center}
\includegraphics[width=\textwidth]{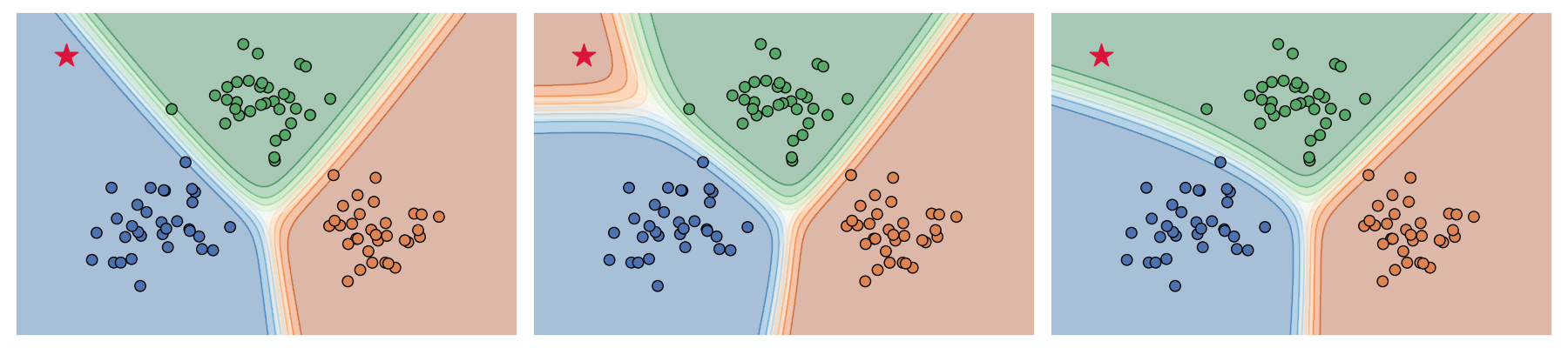}
\caption[]{Adversarial models. For the red test point,
  the predictive uncertainty is high as it is 
  far from the training data. 
  High uncertainties are detected by 
  different adversarial models that 
  assign the red test point 
  to different classes, 
  although all of them explain the training data equally well. 
  As a result, the true class of the test point remains ambiguous.}
\label{fig:adverarialModels}
\end{center}
\vspace{-0.2cm}
\end{figure}

Current uncertainty quantification methods 
such as Deep Ensembles \citep{Lakshminarayanan:17}
or Monte-Carlo (MC) dropout \citep{Gal:16}
underperform at estimating 
the epistemic uncertainty \citep{Wilson:20, ParkerHolder:20, DAngelo:21}, 
since they primarily consider the posterior when sampling models.
Thus they are prone to miss important posterior modes, where the whole integrand of the
integral defining the epistemic uncertainty is large.
We introduce Quantification of Uncertainty with Adversarial Models (QUAM) to identify those posterior modes.
QUAM searches for those posterior modes via adversarial models 
and uses them to reduce the approximation error when 
estimating the integral that defines the epistemic uncertainty.

Adversarial models are characterized by a large value of the integrand of
the integral defining the epistemic uncertainty. %
Thus, they considerably differ to the reference model's 
prediction at a test point
while having a similarly high posterior probability.
Consequently, they are counterexamples of the reference model
that predict differently for a new input, 
but explain the training data equally well.
Fig.~\ref{fig:adverarialModels} shows examples 
of adversarial models which assign different classes to a test point, 
but agree on the training data.
A formal definition is given by Def.~\ref{def:adversarial_models}.
It is essential to note that adversarial models are a new concept 
that is to be distinguished from other concepts that include the term 'adversarial' in their naming, such as adversarial examples \citep{Szegedy:13, Biggio:13}, adversarial training \citep{Goodfellow:15}, generative adversarial networks \citep{Goodfellow:14} or adversarial model-based RL \citep{Rigter:22}.

Our main contributions are:
\begin{itemize}[leftmargin=*, topsep=0pt, itemsep=0pt, partopsep=0pt, parsep=2pt]
\item We introduce QUAM as a framework for uncertainty quantification.
QUAM approximates the integral that defines the epistemic uncertainty 
substantially better
than previous methods, since it reduces the approximation error 
of the integral estimator.
\item  We introduce the concept of adversarial models %
for estimating posterior integrals with non-negative integrands.
For a given test point, adversarial models have considerably different
predictions than a reference model
while having similarly high posterior probability.
\item We introduce a new setting for uncertainty quantification,
where the uncertainty of a given, pre-selected model is quantified.
\end{itemize}

\section{Current Methods to Estimate the Epistemic Uncertainty} \label{sec:definition_uncertainty}

\paragraph{Definition of Predictive Uncertainty. }

Predictive uncertainty quantification is about describing a probabilistic model of 
the world, where aleatoric uncertainty 
refers to the inherent stochasticity of
sampling outcomes from the predictive distribution of the model and
epistemic uncertainty refers to the uncertainty about model parameters.
We consider two distinct settings of predictive uncertainty quantification.
Setting \textbf{(a)} concerns with the predictive uncertainty at a new test point
expected under all plausible models given the training dataset \citep{Gal:16thesis, Huellermeier:21}.
This definition of uncertainty comprises how differently possible models predict (epistemic)
and how confident each model is about its prediction (aleatoric).
Setting \textbf{(b)} concerns with the predictive uncertainty at a new test point
for a given, pre-selected model.
This definition of uncertainty comprises how likely this model is the true model that generated the training dataset (epistemic) \citep{Apostolakis:90, Helton:97}
and how confident this model is about its prediction (aleatoric).

As an example, assume we have initial data from an epidemic, but we do not know 
the exact infection rate, which is a parameter of a prediction model.
The goal is to predict the number of infected persons 
at a specific time in the future,
where each point in time is a test point. 
In setting (a), we are interested in the uncertainty of test point predictions 
of all models using infection rates that explain the initial data.
If all likely models agree for a given new test point, 
the prediction of any of those models can be trusted,
otherwise we can not trust the prediction regardless of 
which model is selected in the end.
In setting (b), we have selected a specific infection rate from the initial data
as parameter for our model to make predictions. 
We refer to this model as the given, pre-selected model.
However, we do not know the true infection rate of the epidemic.
All models with infection rates that are consistent with the 
initial data are likely to be the true model.
If all likely models agree with the given, pre-selected model for a given new test point, 
the prediction of the model can be trusted.

\subsection{Measuring Predictive Uncertainty}

We consider the predictive distribution of a single model $~{p(\By \mid \Bx, \Bw)}$,
which is a probabilistic model of the world.
Depending on the task, the predictive distribution of this probabilistic model 
can be a categorical distribution for classification or a Gaussian distribution for regression.
The Bayesian framework offers a principled way to treat the uncertainty about the parameters through the posterior $~{p(\Bw \mid \cD) \propto p(\cD \mid \Bw) p(\Bw)}$ for a given dataset $\cD$.
The Bayesian model average (BMA) predictive distribution is 
given by $~{p(\By \mid \Bx, \cD ) \ =  \int_{\cW }  p(\By \mid \Bx, \tilde{\Bw}) 
p(\tilde{\Bw} \mid \cD )  \Rd \tilde{\Bw}}$.
Following \cite{Gal:16thesis, Depeweg:18, Smith:18, Huellermeier:21}, the uncertainty of the BMA predictive distribution is commonly measured by the entropy $\ENT{p(\By \mid \Bx, \cD)}$. 
It refers to the total uncertainty, which can be decomposed into an aleatoric and an epistemic part.
The BMA predictive entropy is equal to the posterior expectation of the cross-entropy $\CE{\cdot}{\cdot}$
between the predictive distribution of candidate models and the BMA,
which corresponds to setting (a).
In setting (b), the cross-entropy is between the predictive distribution of the given, pre-selected model and candidate models.
Details about the entropy and cross-entropy as measures of uncertainty are given in Sec.~\ref{sec:theory:cross_entropy} in the appendix.
In the following, we formalize how to measure the notions of uncertainty in setting (a) and (b) using the expected cross-entropy over the posterior.

\paragraph{Setting (a): Expected uncertainty when selecting a model.} 
We estimate the predictive uncertainty at a test point $\Bx$
when selecting a model $\tilde{\Bw}$ given a training dataset $\cD$.
The total uncertainty is the expected cross-entropy between 
the predictive distribution of candidate models $p(\By \mid \Bx, \tilde{\Bw} )$ and 
the BMA predictive distribution $p(\By \mid \Bx, \cD )$, where the expectation is
with respect to the posterior:
\begin{align}  \label{eq:epistemic_setting_a}
\int_{\cW } & \CE{p(\By \mid \Bx, \tilde{\Bw} )}{p(\By \mid \Bx, \cD )}
 \ p(\tilde{\Bw} \mid \cD ) \ \Rd \tilde{\Bw} \ = \ \ENT{p(\By \mid \Bx, \cD )}\\ \nonumber
 &= \ \int_{\cW }    \ENT{p(\By \mid \Bx, \tilde{\Bw} )} \ p(\tilde{\Bw} \mid \cD ) \ \Rd \tilde{\Bw} \ + \ \MI{Y}{W \mid \Bx, \cD} \\ \nonumber
 &= \ \underbrace{\int_{\cW } \ENT{p(\By \mid \Bx, \tilde{\Bw} )} 
 \ p(\tilde{\Bw} \mid \cD ) \ \Rd \tilde{\Bw}}_{\text{aleatoric}} \ + \ \underbrace{\int_{\cW } \KL{p(\By \mid \Bx, \tilde{\Bw} )}{p(\By \mid \Bx, \cD )}
 \ p(\tilde{\Bw} \mid \cD ) \ \Rd \tilde{\Bw}}_{\text{epistemic}} \ .  
\end{align}
The aleatoric uncertainty characterizes the uncertainty due to the expected stochasticity of sampling outcomes from the predictive distribution of candidate models $p(\By \mid \Bx, \tilde{\Bw})$.
The epistemic uncertainty characterizes the uncertainty due to the mismatch between the predictive distribution of candidate models and the BMA predictive distribution.
It is measured by the mutual information $\MI{\cdot}{\cdot}$, between the prediction $Y$ and the model parameters $W$ for a given test point and dataset, 
which is equivalent to the posterior expectation of the KL-divergence $\KL{\cdot}{\cdot}$ 
between the predictive distributions of candidate models and 
the BMA predictive distribution.
Derivations are given in appendix Sec.~\ref{sec:theory:measures_of_uncertainty}.

\paragraph{Setting (b): Uncertainty of a given, pre-selected model.} 
We estimate the predictive uncertainty of a given, pre-selected
model $\Bw$ at a test point $\Bx$.
We assume that the dataset $\cD$ is produced according to 
the true distribution $p(\By \mid \Bx, \Bw^*)$ parameterized by $\Bw^*$.
The posterior $p(\tilde{\Bw} \mid \cD )$ is an estimate 
of how likely $\tilde{\Bw}$ match $\Bw^*$.
For epistemic uncertainty, we should measure the difference between the 
predictive distributions under $\Bw$ and $\Bw^*$, but $\Bw^*$ is unknown.
Therefore, we measure the expected difference between 
the predictive distributions under $\Bw$ and $\tilde{\Bw}$.
In accordance with \cite{Apostolakis:90} and \cite{Helton:97},
the total uncertainty is therefore the expected cross-entropy between
the predictive distributions of a given, pre-selected model $\Bw$ 
and candidate models $\tilde{\Bw}$, any of which could be 
the true model $\Bw^*$ according to the posterior:
\begin{align}  \label{eq:epistemic_setting_b}
 \int_{\cW } & \CE{p(\By \mid \Bx, \Bw )}{p(\By \mid \Bx, \tilde{\Bw} )}
 \ p(\tilde{\Bw} \mid \cD ) \ \Rd \tilde{\Bw} \\ \nonumber
 &= \ \underbrace{\ENT{p(\By \mid \Bx, \Bw )} \vphantom{\int_{\cW}}}_{\text{aleatoric}} \ + \ \underbrace{\int_{\cW }    \KL{p(\By \mid \Bx, \Bw )}{p(\By \mid \Bx, \tilde{\Bw} )} 
 \ p(\tilde{\Bw} \mid \cD ) \ \Rd \tilde{\Bw}}_{\text{epistemic}} \ .
\end{align}
The aleatoric uncertainty characterizes the uncertainty due to the stochasticity of sampling outcomes from the predictive distribution of the given, pre-selected model $p(\By \mid \Bx, \Bw)$.
The epistemic uncertainty characterizes the uncertainty due to the mismatch between the predictive distribution of the given, pre-selected model and the predictive distribution of candidate models that could be the true model.
Derivations and further details are given in appendix Sec.~\ref{sec:theory:measures_of_uncertainty}.

\subsection{Estimating the Integral for Epistemic Uncertainty }

\begin{figure*}[b!]
\centering
\begin{subfigure}{0.32\textwidth}
\includegraphics[width=\textwidth, trim={0 13pt 0 13pt}, clip]{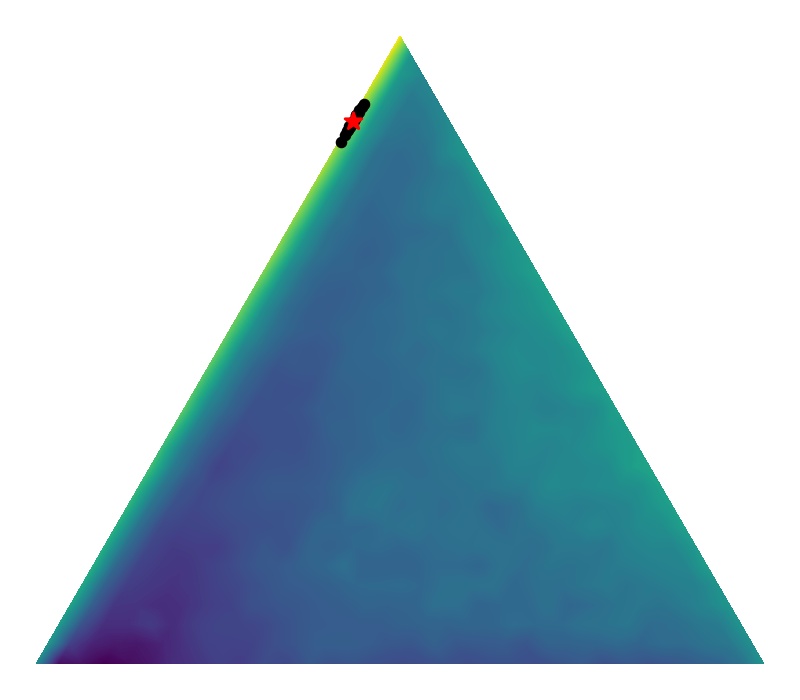}
\subcaption{Deep Ensembles}
\end{subfigure}
\begin{subfigure}{0.32\textwidth}
\includegraphics[width=\textwidth, trim={0 13pt 0 13pt}, clip]{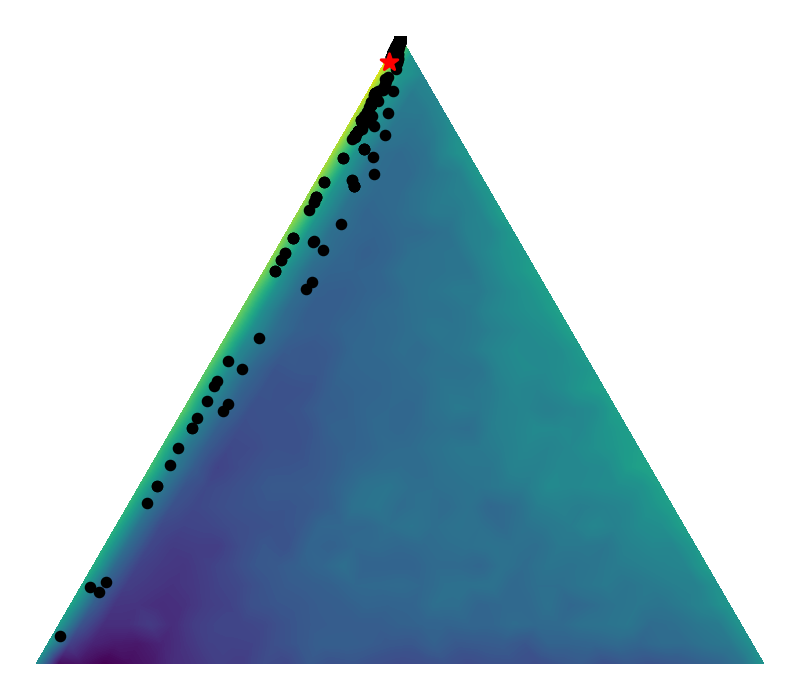}
\subcaption{MC dropout}
\end{subfigure}
\begin{subfigure}{0.32\textwidth}
\includegraphics[width=\textwidth, trim={0 13pt 0 13pt}, clip]{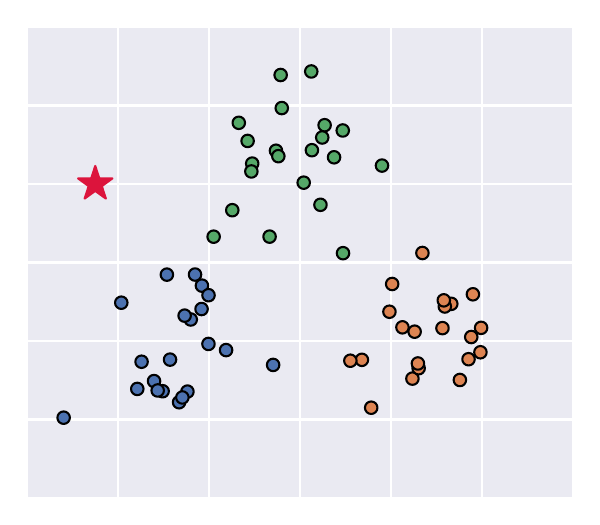}
\subcaption{Training data + new test point}
\end{subfigure}
\caption[]{Model prediction analysis.
Softmax outputs (black) of individual models of Deep Ensembles (a) 
and MC dropout (b), as well as their average output (red) on a probability simplex.
Models were selected on the training data, and evaluated on the new test point (red) depicted in (c).
The background color denotes the maximum likelihood of the training data that is achievable by a model having a predictive distribution (softmax values) equal to the respective location on the simplex.
Deep Ensembles and MC dropout fail to find models predicting the orange class, although there would be likely models that do so.
Details on the experimental setup are given in the appendix, Sec.~\ref{apx:sec:simplex}.}
\label{fig:simplex}
\end{figure*}

Current methods for predictive uncertainty quantification suffer from 
underestimating the epistemic uncertainty \citep{Wilson:20, ParkerHolder:20, DAngelo:21}.
The epistemic uncertainty is given by the respective terms in Eq.~\eqref{eq:epistemic_setting_a} 
for setting (a) and Eq.~\eqref{eq:epistemic_setting_b} 
for our new setting (b).
To estimate these integrals, almost all methods use gradient
descent on the training data.
Thus, posterior modes that are hidden from
the gradient flow remain undiscovered and the epistemic uncertainty is underestimated \citep{Shah:20, DAngelo:21}.
An illustrative example is depicted in Fig.~\ref{fig:simplex}.
Posterior expectations 
as in Eq.~\eqref{eq:epistemic_setting_a} 
and Eq.~\eqref{eq:epistemic_setting_b} 
that define the epistemic uncertainty
are generally approximated using Monte Carlo integration.
A good approximation of posterior integrals through Monte Carlo integration requires 
to capture all large values of the non-negative integrand \citep{Wilson:20}, which is not only
large values of the posterior, but also large values of the KL-divergence.

Variational inference \citep{Graves:11, Blundell:15, Gal:16} and 
ensemble methods \citep{Lakshminarayanan:17} estimate the posterior integral based on models with high posterior.
Posterior modes may be hidden from gradient
descent based techniques as they only discover
mechanistically similar models.
Two models are mechanistically similar if they 
rely on the same input attributes for making their predictions, that is,
they are invariant to the same input attributes \citep{Lubana:22}.
However, gradient descent will always start by extracting
input attributes that are highly correlated to the target as
they determine the steepest descent in the error landscape.
These input attributes create a large basin in the error landscape into
which the parameter vector is drawn via gradient descent.
Consequently, other modes further away from such basins are almost never found \citep{Shah:20, DAngelo:21}.
Thus, the epistemic uncertainty is underestimated.
Another reason that posterior modes may be hidden from gradient
descent is the presence of different labeling hypotheses.
If there is more than one way to
explain the training data, gradient descent will use all of them
as they give the steepest error descent \citep{Scimeca:21}.

Other work focuses on MCMC sampling according to the posterior distribution, 
which is approximated by stochastic gradient variants \citep{Welling:11, Chen:14} 
for large datasets and models.
Those are known to face issues to efficiently explore the highly complex 
and multimodal parameter space and escape local posterior modes.
There are attempts to alleviate the problem \citep{Li:16, Zhang:20}.
However, those methods
do not explicitly look for important posterior modes, where the predictive distributions 
of sampled models contribute strongly to the approximation of the posterior integral,
and thus have large values for the KL-divergence.

\section{Adversarial Models to Estimate the Epistemic Uncertainty}

\paragraph*{Intuition.}
The epistemic uncertainty in Eq.~\eqref{eq:epistemic_setting_a} 
for setting (a) compares possible
models with the BMA. 
Thus, the BMA is
used as reference model.
The epistemic uncertainty in Eq.~\eqref{eq:epistemic_setting_b} 
for our new setting (b) compares
models that are candidates for the true model with the given, pre-selected model. 
Thus, the given, pre-selected model is used as reference model.
If the reference model makes some prediction at the test point,
and if other models (the adversaries) make different predictions 
while explaining the training data equally well, 
then one should be uncertain about the prediction.
Adversarial models are plausible outcomes 
of model selection, while having a different prediction at the
test data point than the reference model.
In court, the same principle is used: 
if the prosecutor presents a scenario but the advocate
presents alternative equally plausible scenarios, the judges become 
uncertain about what happened and rule in favor of the defendant.
We use adversarial models to identify locations
where the integrand of the
integral defining the epistemic uncertainty in 
Eq.~\eqref{eq:epistemic_setting_a} or Eq.~\eqref{eq:epistemic_setting_b}
is large. 
These locations are used to construct a mixture distribution 
that is used for mixture importance sampling
to estimate the desired integrals. 
Using the mixture distribution for sampling, we aim to considerably 
reduce the approximation error of the estimator of the epistemic uncertainty.

\paragraph*{Mixture Importance Sampling.}
We estimate the integrals of epistemic uncertainty in
Eq.~\eqref{eq:epistemic_setting_a} and in Eq.~\eqref{eq:epistemic_setting_b}.
In the following, we focus on setting (b) with 
Eq.~\eqref{eq:epistemic_setting_b}, but all results hold for setting (a) with 
Eq.~\eqref{eq:epistemic_setting_a} as well.
Most methods sample from a distribution $q(\tilde{\Bw})$
to approximate the integral:
\begin{align} \label{eq:epistemic_importance_sampling}
    v \ &= \ \int_\cW \KL{p(\By \mid \Bx, \Bw )}{p(\By \mid \Bx, \tilde{\Bw} )} \ p(\tilde{\Bw} \mid \cD ) \ \Rd \tilde{\Bw} \ = \  \int_\cW \frac{u(\Bx, \Bw, \tilde{\Bw})}{q(\tilde{\Bw})} \ q(\tilde{\Bw}) \ \Rd \tilde{\Bw} \ ,
\end{align}
where $u(\Bx, \Bw, \tilde{\Bw})=
\KL{p(\By \mid \Bx, \Bw )}{p(\By \mid \Bx, \tilde{\Bw} )}  p(\tilde{\Bw} \mid \cD ) $.
As with Deep Ensembles or MC dropout,  
posterior sampling is often approximated  
by a sampling distribution $q(\tilde{\Bw})$ that is close 
to $p(\tilde{\Bw} \mid \cD )$. 
Monte Carlo (MC) integration estimates $v$ by  
\begin{align} \label{eq:epistemic_importance_sampling_estimator}
    \hat{v} &= \frac{1}{N} \sum_{n=1}^N \frac{u(\Bx, \Bw, \tilde{\Bw}_n)}{q(\tilde{\Bw}_n)} \ , \qquad \tilde{\Bw}_n \sim q(\tilde{\Bw}) \ .
\end{align}
If the posterior has different modes,
the estimate under a unimodal approximate distribution
has high variance and converges very slowly \citep{Steele:06}.
Thus, we use mixture importance 
sampling (MIS) \citep{Hesterberg:95}.
MIS utilizes a mixture distribution
instead of the unimodal distribution in 
standard importance sampling \citep{Owen:00}.
Furthermore, many MIS methods
iteratively enhance the
sampling distribution by incorporating new modes \citep{Raftery:10}.
In contrast to the usually applied 
iterative enrichment methods which find new modes by chance,
we have a much more favorable situation. 
We can explicitly search for posterior modes
where the KL divergence is large, as we can
cast it as a supervised learning problem.
Each of these modes determines the location of a 
mixture component of the mixture distribution.
\begin{theorem}
\label{th:mse}
    The expected mean squared error of importance sampling with $q(\tilde \Bw)$ can be bounded by
    \begin{equation}
    \label{eq:expected_mse}
        \EXP_{q(\tilde \Bw)}\left[ \left( \hat{v} \ - \ v \right)^2 \right] \leq  \EXP_{q(\tilde \Bw)}\left[ \left(\frac{u(\Bx, \Bw, \tilde{\Bw})}{q(\tilde \Bw)} \right)^2\right] \frac{4}{N} \ .
    \end{equation}
\end{theorem}
\begin{proof}
The inequality Eq.~\eqref{eq:expected_mse} follows from Theorem~1 in
\cite{Akyildiz:21}, when considering $0 \leq u(\Bx, \Bw, \tilde{\Bw})$ 
as an unnormalized distribution
and setting $\varphi=1$.
\end{proof}
Approximating only the posterior $p(\tilde{\Bw} \mid \cD )$ as done by 
Deep Ensembles or MC dropout is insufficient to guarantee a low 
expected mean squared error,
since the sampling variance cannot be bounded (see appendix Sec.~\ref{sec:importance_sampling_variance}).
\begin{corollary}
    With constant $c$, $\EXP_{q(\tilde \Bw)}\left[ \left( \hat{v} \ - \ v \right)^2 \right]  
    \leq  4  c^2 / N$
    holds if
    $u(\Bx, \Bw, \tilde{\Bw}) \leq c \ q(\tilde \Bw)$. 
\end{corollary}
Consequently, $q(\tilde \Bw)$ must have modes where 
$u(\Bx, \Bw, \tilde{\Bw})$ has modes even if the $q$-modes are a factor $c$ smaller.
The modes of $u(\Bx, \Bw, \tilde{\Bw})$ are models $\tilde{\Bw}$ 
with both high posterior and high KL-divergence.
We are searching for these modes to determine the locations $\breve{\Bw}_k$ of 
the components of a mixture distribution $q(\tilde \Bw)$:
\vspace{-0.2cm} %
\begin{align}
\label{eq:mixture_distribution}
  q(\tilde \Bw) \ = \ \sum_{k=1}^K \alpha_k \ \cP(\tilde \Bw \ ; \breve{\Bw}_k, \Bth) \ ,
\end{align}
with $\alpha_k = 1 /K$ for $K$ such models $\breve{\Bw}_k$ that determine a mode.
Adversarial model search finds the locations $\breve{\Bw}_k$ of the mixture components, 
where $\breve{\Bw}_k$ is an adversarial model.
The reference model does not define a mixture component, as it has
zero KL-divergence to itself.
We then sample from a distribution $\cP$ at the local posterior mode with mean $\breve{\Bw}_k$ and a set of shape parameters $\Bth$.
The simplest choice for $\cP$ is a Dirac delta distribution, but one could use e.g.\ a local Laplace approximation of the posterior \citep{MacKay:92b}, or a Gaussian distribution in some weight-subspace \citep{Maddox:19}.
Furthermore, one could use $\breve{\Bw}_k$ as starting point for SG-MCMC chains \citep{Welling:11, Chen:14, Zhang:20, Zhang:22}.
More details regarding MIS are given in the appendix in Sec.~\ref{sec:importance_sampling_variance}.
In the following, we propose an algorithm to find those models with both high posterior and high KL-divergence to the predictive distribution of the reference model.

\paragraph*{Adversarial Model Search.}
Adversarial model search is the concept of 
searching for a model that has a large distance / divergence to
the reference predictive distribution and at the same time a high posterior.
We call such models ''adversarial models'' as they 
act as adversaries to the reference model by contradicting its prediction.
A formal definition of an adversarial model is given by Def.~\ref{def:adversarial_models}:

\begin{definition}
    Given are a new test data point $\Bx$, a reference conditional probability model
    $p(\By \mid \Bx, \Bw)$ from a model class parameterized by $\Bw$, 
    a divergence or distance measure $\rD(\cdot,\cdot)$ for probability distributions,
     $~{\gamma>0}$, $~{\Lambda > 0}$, and a dataset $\cD$.
    Then a model with parameters $\tilde{\Bw}$ 
    that satisfies the inequalities
    $~{|\log  p(\Bw \mid \cD) - \log p(\tilde{\Bw} \mid \cD)  | \leq \gamma}$ 
    and
    $~{\rD({p(\By \mid \Bx,\Bw), p(\By \mid \Bx, \tilde{\Bw}))} \geq \Lambda}$ 
    is called an $(\gamma, \Lambda)-${adversarial model}.
    \label{def:adversarial_models}
\end{definition}

 Adversarial model search corresponds to the following optimization problem:
\begin{align}
\label{eq:minProb}
 &\max_{\Bde \in \Delta}  \ \DI{p(\By \mid \Bx, \Bw )}{p(\By \mid \Bx, \Bw \ + \ \Bde )}   \quad  \mbox{s.t.} \ \log p(\Bw \mid \cD) \ - \ \log p(\Bw \ + \ \Bde \mid \cD) \ \leq \ \gamma \ .
\end{align}
We are searching for a weight perturbation $\Bde$ that maximizes the
distance $\DI{\cdot}{\cdot}$ to the reference distribution without decreasing the log posterior more than $\gamma$.
The search for adversarial models is restricted to $\Bde \in \Delta$,
for example by only optimizing the last layer of the reference model 
or by bounding the norm of $\Bde$.
This optimization problem can be rewritten as:
\begin{align}
  &\max_{\Bde \in \Delta} \ \DI{p(\By \mid \Bx, \Bw )}{p(\By \mid \Bx, \Bw \ + \ \Bde )}   \ + \ c \ \left( \log p(\Bw \ + \ \Bde \mid \cD)  \ - \
    \log p(\Bw \mid \cD) \ + \ \gamma \right) \ .
\end{align}
where $c$ is a hyperparameter.
According to the {\em Karush-Kuhn-Tucker (KKT) theorem} \citep{Karush:39,Kuhn:50,May:20,Luenberger:16}:
If $\Bde^*$ is the solution to the problem Eq.~\eqref{eq:minProb}, then
there exists a $c^* \geq 0$ with
$\nabla_{\Bde} \cL(\Bde^*,c^*) = \BZo$ 
($ \cL$ is the Lagrangian) and
$c^* \ \left( \log p(\Bw \mid \cD)  -  \log p(\Bw  +  \Bde^* \mid \cD)  -  \gamma \right) = 0$.
This is a necessary condition for an optimal point according to 
Theorem on Page 326 of \citet{Luenberger:16}.

We solve this optimization problem by
the penalty method, which relies on the KKT theorem \citep{Zangwill:67}.
A penalty algorithm solves a series of unconstrained problems, solutions of which converge to the solution of the original constrained problem (see e.g. \cite{Fiacco:90}).
The unconstrained problems are constructed by adding a weighted penalty function measuring the constraint violation to the objective function. At every step, the weight of the penalty is increased, thus the constraints are less violated.
If exists, the solution to the constraint optimization problem is an adversarial model that is located within a posterior mode but has a different predictive distribution compared to the reference model.
We summarize the adversarial model search in Alg.~\ref{alg:ams}.

\clearpage

\renewcommand{\algorithmicensure}{\textbf{Supplies:}}
\renewcommand{\algorithmicrequire}{\textbf{Requires:}}
\begin{algorithm}[h!]
  \caption{Adversarial Model Search (used in QUAM)}
  \label{alg:ams}
  \begin{algorithmic}[1]
    \ENSURE Adversarial model $\breve{\Bw}$ with maximum\ $\rL_{\text{adv}}$ and $\rL_{\text{pen}} \leq 0$
    \REQUIRE Test point $\Bx$,
    training dataset $\cD = \{(\Bx_k, \By_k)\}_{k=1}^{K}$, reference model $\Bw$, loss function $l$, loss of reference model on the training dataset $\rL_{\text{ref}} = \frac{1}{K} \sum_{k=1}^K l(p(\By \mid \Bx_k, \Bw), \By_k)$, 
    minimization procedure MINIMIZE, number of penalty iterations $M$, initial penalty parameter $c_0$, penalty parameter increase scheduler $\eta$, slack parameter $\gamma$, distance / divergence measure $\DI{\cdot}{\cdot}$.    
    \STATE $\breve \Bw \leftarrow \Bw ; \;\; \tilde \Bw \leftarrow \Bw ; \;\; c \leftarrow c_0$
    \FOR{$m \leftarrow 1$ to $M$}
    \STATE $\rL_{\text{pen}} \leftarrow \frac{1}{K} \sum_{k=1}^K l(p(\By \mid \Bx_k, \tilde{\Bw}  ), \By_k) \ - \ \left( \rL_{\text{ref}} \ + \ \gamma \right)$
    \STATE $\rL_{\text{adv}} \leftarrow - \DI{p(\By \mid \Bx, \Bw )}{p(\By \mid \Bx, \tilde{\Bw}  )}$ 
    \STATE $\rL \leftarrow  \rL_{\text{adv}} + c \ \rL_{\text{pen}}$
    \STATE $\tilde \Bw \leftarrow \text{MINIMIZE}(\rL(\tilde \Bw))$
    \IF{$\rL_{\text{adv}}$ larger than all previous and $\rL_{\text{pen}} \leq 0$}
    \STATE $\breve \Bw \leftarrow \tilde \Bw$ 
    \ENDIF
    \STATE $c \leftarrow \eta(c)$
    \ENDFOR
    \STATE \textbf{return} $\breve \Bw$ 
  \end{algorithmic}
  \label{algo:ams}
\end{algorithm}

\vspace{-0.2cm}
\paragraph{Practical Implementation.}
Empirically, we found that directly executing the optimization procedure defined in Alg.~\ref{alg:ams} tends to result in adversarial models with similar predictive distribution for a given input across multiple searches. 
The vanilla implementation of Alg.~\ref{alg:ams} corresponds to an \emph{untargeted} attack, known from the literature on adversarial attacks \citep{Szegedy:13, Biggio:13}. To prevent the searches from converging to a single solution, we optimize the cross-entropy loss for one specific class during each search, which corresponds to a \emph{targeted} attack. Each resulting adversarial model represents a local optimum of Eq.~\eqref{eq:minProb}.
We execute as many adversarial model searches as there are classes, dedicating one search to each class, unless otherwise specified. 
To compute Eq.~\eqref{eq:epistemic_importance_sampling_estimator},
we use the predictive distributions $p(\By \mid \Bx, \tilde\Bw)$ of all models $\tilde{\Bw}$ encountered during each penalty iteration of all searches, weighted by their posterior probability.
The posterior probability is approximated with the negative exponential training loss, the likelihood, of models $\tilde{\Bw}$.
This approximate posterior probability is scaled with a temperature parameter, set as a hyperparameter.
Further details are given in the appendix Sec.~\ref{apx:sec:ams_details}.

\section{Experiments}
\vspace{-0.1cm}
In this section, we compare previous uncertainty quantification methods 
and our method QUAM in a set of experiments.
First, we assess the considered methods on a synthetic benchmark, 
on which it is feasible to compute a ground truth epistemic uncertainty.
Then, we conduct challenging out-of-distribution (OOD) detection, 
adversarial example detection, misclassification detection and selective prediction experiments in the vision domain.
We compare (1) QUAM,
(2) cyclical Stochastic Gradient Hamiltonian Monte Carlo (cSG-HMC) \citep{Zhang:20}, 
(3) an efficient Laplace approximation (Laplace) \citep{Daxberger:21}, 
(4) MC dropout (MCD) \citep{Gal:16} and 
(5) Deep Ensembles (DE) \citep{Lakshminarayanan:17} 
on their ability to estimate the epistemic uncertainty.
Those baseline methods, especially Deep Ensembles, 
are persistently among the best 
performing uncertainty quantification methods 
across various benchmark tasks \citep{Filos:19, Ovadia:19, Caldeira:20, Band:22}

\subsection{Epistemic Uncertainty on Synthetic Dataset}

\begin{figure*}
\captionsetup[subfigure]{aboveskip=-0.5pt,belowskip=-1pt}
\centering
\begin{subfigure}{0.325\textwidth}
\includegraphics[width=\textwidth]{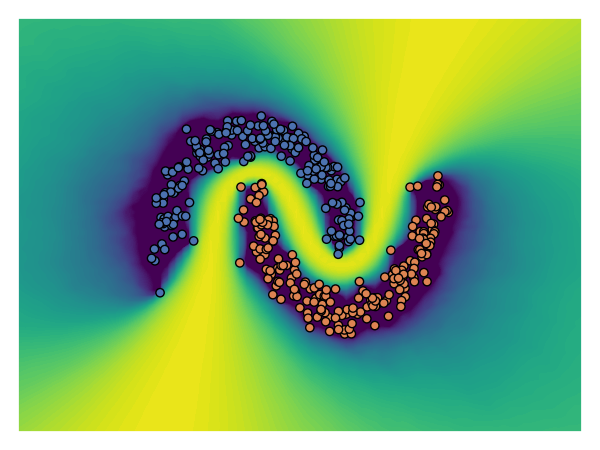}
\subcaption{\textbf{Ground Truth} - HMC}
\end{subfigure}
\begin{subfigure}{0.325\textwidth}
\includegraphics[width=\textwidth]{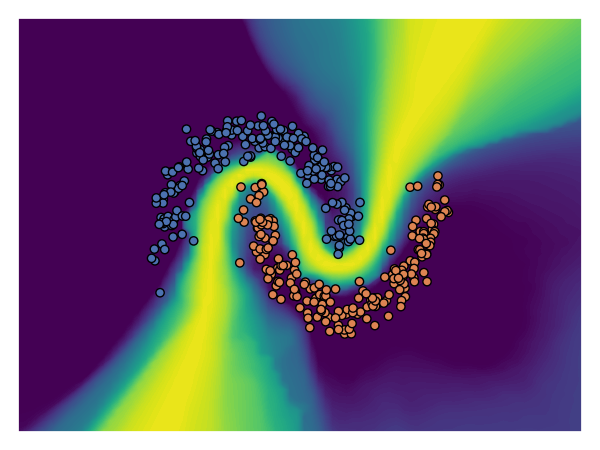}
\subcaption{cSG-HMC}
\end{subfigure}
\begin{subfigure}{0.325\textwidth}
\includegraphics[width=\textwidth]{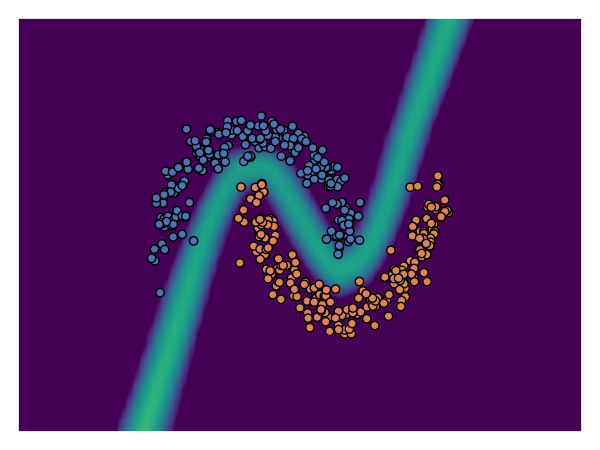}
\subcaption{Laplace}
\end{subfigure}
\begin{subfigure}{0.325\textwidth}
\includegraphics[width=\textwidth]{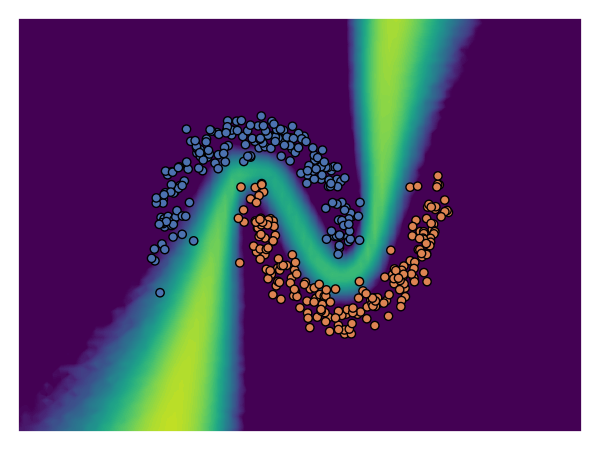}
\subcaption{MC dropout}
\end{subfigure}
\begin{subfigure}{0.325\textwidth}
\includegraphics[width=\textwidth]{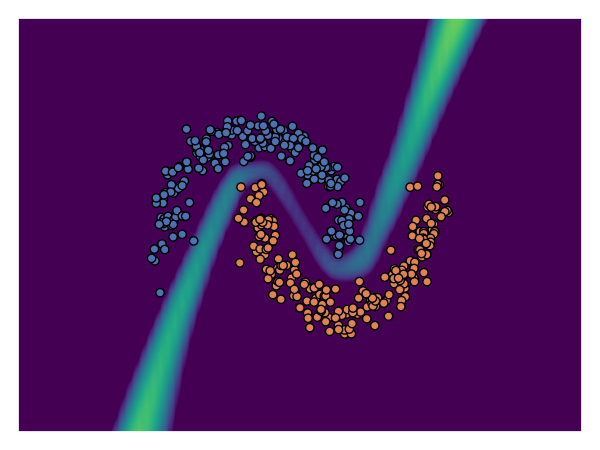}
\subcaption{Deep Ensembles}
\end{subfigure}
\begin{subfigure}{0.325\textwidth}
\includegraphics[width=\textwidth]{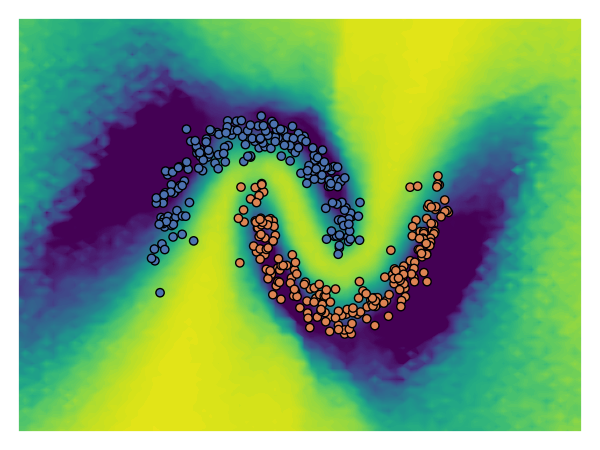}
\subcaption{\textbf{Our Method} - QUAM}
\end{subfigure}
\caption[]{Epistemic uncertainty as in Eq.~\eqref{eq:epistemic_setting_a} for two-moons. Yellow denotes high epistemic uncertainty. Purple denotes low epistemic uncertainty. HMC is considered as ground truth \citep{Izmailov:21}
and is most closely matched by QUAM.
Artifacts for QUAM arise because it is applied to each test point individually, whereas other methods use the same sampled models for all test points.}
\label{fig:res:classification_twomoon_setting_a}
\vspace{-0.2cm}
\end{figure*}

We evaluated all considered methods on the two-moons dataset, created using the implementation of \cite{Pedregosa:11}. %
To obtain the ground truth uncertainty, we utilized Hamiltonian Monte Carlo (HMC) \citep{Neal:96}.
HMC is regarded as the most precise algorithm to approximate posterior expectations \citep{Izmailov:21}, but necessitates extreme computational expenses to be applied to models and datasets of practical scale.
The results are depicted in Fig.~\ref{fig:res:classification_twomoon_setting_a}.
QUAM most closely matches the uncertainty estimate of the ground truth 
epistemic uncertainty obtained by HMC and 
excels especially on the regions further away from the decision boundary 
such as in the top left and bottom right of the plots.
All other methods fail to capture the epistemic uncertainty in those regions 
as gradient descent on the training set fails to capture posterior modes 
with alternative predictive distributions in those parts 
and misses the important integral components.
Experimental details and results for the epistemic 
uncertainty as in Eq.~\eqref{eq:epistemic_setting_b} 
are given in the appendix Sec.~\ref{apx:sec:synthetic}.

\subsection{Epistemic Uncertainty on Vision Datasets}
We benchmark the ability of different methods to estimate the epistemic uncertainty 
of a given, pre-selected model 
(setting (b) as in Eq.~\eqref{eq:epistemic_setting_b}) 
in the context of 
(i) out-of-distribution (OOD) detection, 
(ii) adversarial example detection, 
(iii) misclassification detection
and (iv) selective prediction.
In all experiments, we assume to have access to a pre-trained model 
on the in-distribution (ID) training dataset, which we refer to as reference model.
The epistemic uncertainty is expected to be higher
for OOD samples, as they can be assigned to multiple ID classes,
depending on the utilized features.
Adversarial examples indicate that the model is misspecified 
on those inputs, thus we expect a higher epistemic uncertainty, the uncertainty about the model parameters.
Furthermore, we expect higher epistemic uncertainty for misclassified samples than for correctly classified samples. 
Similarly, we expect the classifier to perform better on a subset of more certain samples.
This is tested by evaluating the accuracy of the classifier on retained subsets of a certain fraction of samples with the lowest epistemic uncertainty \citep{Filos:19, Band:22}.
We report the AUROC for classifying the ID vs. OOD samples (i),
the ID vs. the adversarial examples (ii), or the correctly classified vs. the misclassified samples (iii), 
using the epistemic uncertainty as score to distinguish the two
classes respectively.
For the selective prediction experiment (iv), we report the AUC of the accuracy vs. fraction of retained samples, using the epistemic uncertainty to determine the retained subsets.

\paragraph{MNIST.}
We perform OOD detection on the FMNIST \citep{Xiao:17}, KMNIST \citep{Clanuwat:18}, EMNIST \citep{Cohen:17} and OMNIGLOT \citep{Lake:15}
test datasets as OOD datasets, 
using the LeNet \citep{LeCun:98} architecture.
The test dataset of MNIST \citep{LeCun:98} is used as ID dataset.
We utilize the aleatoric uncertainty of the reference model 
(as in Eq.~\eqref{eq:epistemic_setting_b}) 
as a baseline to assess the added value 
of estimating the epistemic uncertainty
of the reference model.
The results are listed in Tab.~\ref{tab:res:ood_mnist}.
QUAM outperforms all other methods on this task, with Deep Ensembles being the runner
up method on all dataset pairs. 
Furthermore, we observed, that only the epistemic uncertainties obtained by Deep Ensembles and QUAM are able to surpass the performance of using the aleatoric uncertainty of the reference model.

\setlength{\tabcolsep}{7pt}
\renewcommand{\arraystretch}{1.2}
\begin{table}[t!]
\centering
\caption[]{MNIST results: AUROC using the epistemic uncertainty of a given, pre-selected model (as in Eq.~\eqref{eq:epistemic_setting_b}) as a score to distinguish between ID (MNIST) and OOD samples.
We also report the AUROC when using the aleatoric uncertainty of the 
reference model (Reference).\\
\vspace{-0.2cm}}
\label{tab:res:ood_mnist}
\begin{tabular}{ccccccc}
\hline
$\cD_{\text{ood}}$ & Reference       & cSG-HMC             & Laplace             & MCD          & DE         & QUAM                \\ \hline
FMNIST              &            $.986_{\pm .005}$         &         $.977_{\pm .004}$            &               $.978_{\pm .004}$      & $.978_{\pm .005}$ & $.988_{\pm .001}$ & $\boldsymbol{.994}_{\pm .001}$ \\
KMNIST              &         $.966_{\pm .005}$            &        $.957_{\pm .005}$             &                   $.959_{\pm .006}$  & $.956_{\pm .006}$ & $.990_{\pm .001}$ & $\boldsymbol{.994}_{\pm .001}$ \\
EMNIST              & $.888_{\pm .007}$ & $.869_{\pm .012}$ & $.877_{\pm .011}$ & $.876_{\pm .008}$ & $.924_{\pm .003}$ & $\boldsymbol{.937}_{\pm .008}$ \\
OMNIGLOT            &         $.973_{\pm .003}$            &          $.963_{\pm .004}$           & $.963_{\pm .003}$ &           $.965_{\pm .003}$          & $.983_{\pm .001}$ &  $\boldsymbol{.992}_{\pm .001}$ \\ \hline
\end{tabular}
\vspace{-0.2cm}
\end{table}
\renewcommand{\arraystretch}{1}

\setlength{\tabcolsep}{5pt} %
\renewcommand{\arraystretch}{1.2}
\begin{table}[t!]
\centering
\caption[]{ImageNet-1K results: AUROC using the epistemic uncertainty of a given, pre-selected model (as in Eq.~\eqref{eq:epistemic_setting_b}) to distinguish between ID (ImageNet-1K) and OOD samples. 
Furthermore, we report the AUROC when using the epistemic uncertainty for misclassification detection and the AUC of accuracy over fraction of retained predictions on the ImageNet-1K validation dataset.
We also report results for all experiments, using the aleatoric uncertainty of the reference model (Reference).\\
\vspace{-0.2cm}}
\label{tab:res:imagenet}
\begin{tabular}{ccccccc}
\hline
$\cD_{\text{ood}}$ // Task & Reference & cSG-HMC            & MCD                & DE (LL)           & DE (all)          & QUAM                \\ \hline
ImageNet-O & $.626_{\pm .004}$ & $.677_{ \pm .005}$ & $.680_{\pm .003}$  & $.562_{\pm .004}$ & $.709_{\pm .005}$ & $\boldsymbol{.753}_{\pm .011}$ \\
ImageNet-A & $.792_{\pm .002}$ & $.799_{\pm .001}$  & $.827_{\pm .002}$  & $.686_{\pm .001}$ & $\boldsymbol{.874}_{\pm .004}$ & $\boldsymbol{.872}_{\pm .003}$ \\ \hline \hline
Misclassification  & $.867_{\pm .007}$ & $.772_{\pm .011}$  & $.796_{\pm .014}$  & $.657_{\pm .009}$ & $.780_{\pm .009}$ & $\boldsymbol{.904}_{\pm .008}$ \\ 
Selective prediction  & $.958_{\pm .003}$ & $.931_{\pm .003}$  & $.935_{\pm .006}$  & $.911_{\pm .004}$ & $.950_{\pm .002}$ & $\boldsymbol{.969}_{\pm .002}$ \\\hline
\end{tabular}
\vspace{-0.3cm}
\end{table}
\renewcommand{\arraystretch}{1}

\paragraph{ImageNet-1K.}
We conduct OOD detection, adversarial example detection, misclassification detection and selective prediction experiments on ImageNet-1K \citep{Deng:09}.
As OOD dataset, we use ImageNet-O \citep{Hendrycks:21}, 
which is a challenging OOD dataset that was explicitly
created to be classified as an ID dataset with high confidence by conventional ImageNet-1K classifiers.
Similarly, ImageNet-A \citep{Hendrycks:21} is a dataset 
consisting of natural adversarial examples,
which belong to the ID classes of ImageNet-1K, 
but are misclassified with high confidence
by conventional ImageNet-1K classifiers.
Furthermore, we evaluated the utility of the uncertainty score 
for misclassification detection 
of predictions of the reference model on the ImageNet-1K validation dataset.
On the same dataset, we evaluated the accuracy of the reference model when only predicting on fractions of samples with the lowest epistemic uncertainty.

All ImageNet experiments were performed on variations of the EfficientNet architecture \citep{Tan:19}.
Recent work by \cite{Kirichenko:22} showed that typical ImageNet-1K classifiers 
learn desired features of the data even if they rely on simple, 
spurious features for their prediction.
Furthermore, they found, that last layer retraining on a dataset without the spurious correlation is 
sufficient to re-weight the importance that the classifier places on 
different features. 
This allows the classifier to ignore the spurious features and utilize the desired features for its prediction.
Similarly, we apply QUAM on the last layer of the reference model.
We compare against cSG-HMC applied to the last layer, MC dropout and Deep Ensembles. 
MC dropout was applied to the last layer as well, since the EfficientNet architectures utilize dropout only before the last layer.
Two versions of Deep Ensembles were considered.
First, Deep Ensembles aggregated from pre-trained EfficientNets of different network sizes (DE (all)).
Second, Deep Ensembles of retrained last layers on the same encoder network (DE (LL)). 
We further utilize the aleatoric uncertainty of the reference model 
(as in Eq.~\eqref{eq:epistemic_setting_b}) 
as a baseline to assess the additional benefit of estimating the epistemic uncertainty of the reference model. The Laplace approximation was not feasible to compute on our hardware, 
even only for the last layer.

The results are listed in Tab.~\ref{tab:res:imagenet}.
Plots showing the respective curves of each experiment are depicted in Fig.~\ref{fig:roc} in the appendix.
We observe that using the epistemic uncertainty provided by DE (LL) has the worst performance throughout all experiments. 
While DE (all) performed second best on most tasks, MC dropout outperforms it on OOD detection on the ImageNet-O dataset.
QUAM outperforms all other methods on all tasks we evaluated, except for ImageNet-A, where it performed on par with DE (all).
Details about all experiments and additional results are given in the appendix Sec.~\ref{sec:apx:vision_experiments}.

\paragraph{Compute Efficiency.}

As an ablation study, we investigate the performance of QUAM under a restricted computational budget. 
Therefore, the searches for adversarial models were performed on only a subset of classes instead of each eligible class, specifically the top $N$ most probable classes according to the predictive distribution of the given, pre-selected model. 
The computational budget between QUAM and MC dropout was matched by accounting for the number of forward pass equivalents required by each method. In this context, we assume that the backward pass corresponds to the computational cost of two forward passes.
The results depicted in Fig.~\ref{fig:topx} show that QUAM outperforms MC dropout even under a very limited computational budget.
Furthermore, training a single additional ensemble member for Deep Ensembles requires more compute than evaluating the entire ImageNet-O and ImageNet-A datasets with QUAM when performed on all $1000$ classes.

\begin{figure*}
\centering
\begin{subfigure}{0.49\textwidth}
\includegraphics[width=\textwidth, trim={12pt 12pt 0 12pt}, clip]{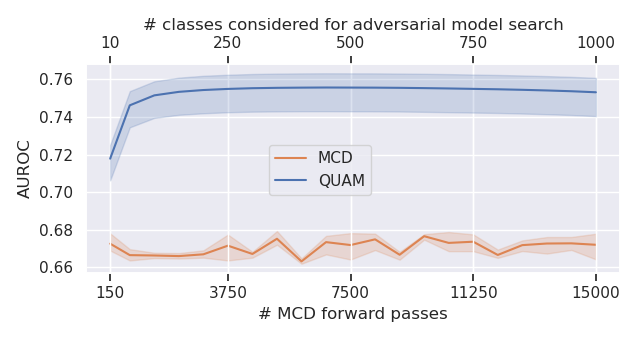}
\end{subfigure}
\begin{subfigure}{0.49\textwidth}
\includegraphics[width=\textwidth, trim={0 12pt 12pt 12pt}, clip]{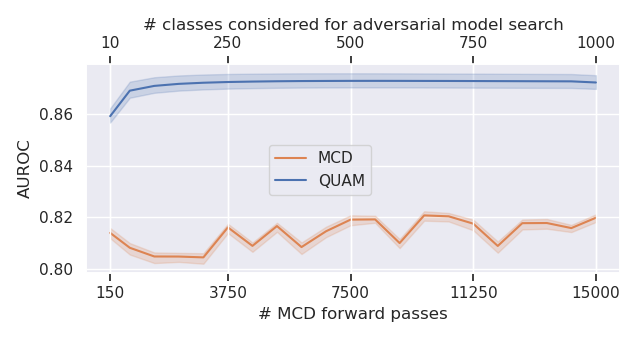}
\end{subfigure}
\caption[]{Inference speed vs. performance. MCD and QUAM evaluated on equal computational budget in terms of forward pass equivalents on ImageNet-O (left) and ImageNet-A (right) tasks.}
\label{fig:topx}
\vspace{-0.2cm}
\end{figure*}

\section{Related Work}
Quantifying predictive uncertainty, especially for deep learning models,
is an active area of research.
Classical uncertainty quantification methods such as
Bayesian Neural Networks (BNNs) \citep{MacKay:92b, Neal:96}
are challenging for deep learning, 
since (i) the Hessian or maximum-a-posterior (MAP) is difficult to estimate and (ii)
regularization and
normalization techniques cannot be treated \citep{Antoran:22}.
Epistemic neural networks \citep{Osband:21}
add a variance term (the epinet) to the output only.
Bayes By Backprop \citep{Blundell:15} and variational neural networks \citep{Oleksiienko:22} 
work only for small models as they require considerably more parameters.
MC dropout \citep{Gal:16}
casts applying dropout during inference as sampling from an approximate distribution.
MC dropout was generalized to MC dropconnect \citep{Mobiny:21}.
Deep Ensembles \citep{Lakshminarayanan:17} are often the
best-performing uncertainty quantification method \citep{Ovadia:19, Wursthorn:22}.
Masksembles or Dropout Ensembles
combine ensembling with MC dropout \citep{Durasov:21}.
Stochastic Weight Averaging approximates the posterior over the weights \citep{Maddox:19}.
Single forward pass methods are fast and they 
aim to capture a different notion of epistemic uncertainty through the distribution or distances of latent representations \citep{Bradshaw:17, Liu:20, Mukhoti:21, vAmersfoort:21, Postels:21} rather than through posterior integrals.
For further methods and a general overview of uncertainty estimation 
see e.g. \cite{Huellermeier:21, Abdar:21} and \cite{Gawlikowski:21}.

\section{Conclusion}

We have introduced QUAM, a novel method that 
quantifies predictive uncertainty using adversarial models.
Adversarial models identify important posterior modes 
that are missed by previous uncertainty quantification methods.
We conducted various experiments on deep neural networks, 
for which epistemic uncertainty is challenging to estimate.
On a synthetic dataset, we highlighted the strength of our method to capture 
epistemic uncertainty.
Furthermore, we conducted experiments on large-scale benchmarks in the vision domain,
where QUAM outperformed all previous methods.

Searching for adversarial models is computationally expensive
and has to be done for each new test point. 
However, more efficient versions can be utilized. 
One can search for 
adversarial models while restricting the search 
to a subset of the parameters, e.g.\ to the last layer 
as was done for the ImageNet experiments, to the normalization parameters, or to the bias weights. 
Furthermore, there have been several advances for efficient 
fine-tuning of large models \citep{Houlsby:19, Hu:21}.
Utilizing those for more efficient versions of our algorithm 
is an interesting direction for future work.

Nevertheless, high stake applications justify this effort
to obtain the best estimate of predictive uncertainty for each new test point.
Furthermore, QUAM is applicable to quantify the predictive uncertainty of 
any single given model, regardless of whether 
uncertainty estimation was considered during the modeling process.
This allows to assess the predictive uncertainty of foundation models or 
specialized models that are obtained externally.

\section*{Acknowledgements}

We would like to thank Angela Bitto-Nemling for providing relevant literature, organizing meetings, and giving feedback on this research project.
Furthermore, we would like to thank Angela Bitto-Nemling, Daniel Klotz, and Sebastian Lehner for insightful discussions and provoking questions.
The ELLIS Unit Linz, the LIT AI Lab, the Institute for Machine Learning, are supported by the Federal State Upper Austria. We thank the projects AI-MOTION (LIT-2018-6-YOU-212), DeepFlood (LIT-2019-8-YOU-213), Medical Cognitive Computing Center (MC3), INCONTROL-RL (FFG-881064), PRIMAL (FFG-873979), S3AI (FFG-872172), DL for GranularFlow (FFG-871302), EPILEPSIA (FFG-892171), AIRI FG 9-N (FWF-36284, FWF-36235), AI4GreenHeatingGrids(FFG- 899943), INTEGRATE (FFG-892418), ELISE (H2020-ICT-2019-3 ID: 951847), Stars4Waters (HORIZON-CL6-2021-CLIMATE-01-01). We thank Audi.JKU Deep Learning Center, TGW LOGISTICS GROUP GMBH, Silicon Austria Labs (SAL), FILL Gesellschaft mbH, Anyline GmbH, Google, ZF Friedrichshafen AG, Robert Bosch GmbH, UCB Biopharma SRL, Merck Healthcare KGaA, Verbund AG, GLS (Univ. Waterloo) Software Competence Center Hagenberg GmbH, T\"{U}V Austria, Frauscher Sensonic, TRUMPF and the NVIDIA Corporation.

\bibliography{arxiv}

\begin{thebibliography}{104}
\providecommand{\natexlab}[1]{#1}
\providecommand{\url}[1]{\texttt{#1}}
\expandafter\ifx\csname urlstyle\endcsname\relax
  \providecommand{\doi}[1]{doi: #1}\else
  \providecommand{\doi}{doi: \begingroup \urlstyle{rm}\Url}\fi

\bibitem[Abdar et~al.(2021)Abdar, Pourpanah, Hussain, Rezazadegan, Liu, Ghavamzadeh, Fieguth, Cao, Khosravi, Acharya, Makarenkov, and Nahavandi]{Abdar:21}
M.~Abdar, F.~Pourpanah, S.~Hussain, D.~Rezazadegan, L.~Liu, M.~Ghavamzadeh, P.~Fieguth, X.~Cao, A.~Khosravi, U.~R. Acharya, V.~Makarenkov, and S.~Nahavandi.
\newblock A review of uncertainty quantification in deep learning: Techniques, applications and challenges.
\newblock \emph{Information Fusion}, 76:\penalty0 243--297, 2021.

\bibitem[Adler et~al.(2008)Adler, Youmaran, and Lionheart]{Adler:08}
A.~Adler, R.~Youmaran, and W.~R.~B. Lionheart.
\newblock A measure of the information content of {EIT} data.
\newblock \emph{Physiological Measurement}, 29\penalty0 (6):\penalty0 S101--S109, 2008.

\bibitem[Akyildiz and Míguez(2021)]{Akyildiz:21}
\"{O}.~D. Akyildiz and J.~Míguez.
\newblock Convergence rates for optimised adaptive importance samplers.
\newblock \emph{Statistics and Computing}, 31\penalty0 (12), 2021.

\bibitem[Angelo and Fortuin(2021)]{DAngelo:21}
F.~D\textquotesingle Angelo and V.~Fortuin.
\newblock Repulsive deep ensembles are bayesian.
\newblock In M.~Ranzato, A.~Beygelzimer, Y.~Dauphin, P.S. Liang, and J.~Wortman Vaughan, editors, \emph{Advances in Neural Information Processing Systems}, volume~34, pages 3451--3465. Curran Associates, Inc., 2021.

\bibitem[Antor\'{a}an et~al.(2022)Antor\'{a}an, Janz, Allingham, Daxberger, Barbano, Nalisnick, and Hern\'{a}ndez-Lobato]{Antoran:22}
J.~Antor\'{a}an, D.~Janz, J.~U. Allingham, E.~Daxberger, R.~Barbano, E.~Nalisnick, and J.~M. Hern\'{a}ndez-Lobato.
\newblock Adapting the linearised {Laplace} model evidence for modern deep learning.
\newblock \emph{ArXiv}, 2206.08900, 2022.

\bibitem[Apostolakis(1991)]{Apostolakis:90}
G.~Apostolakis.
\newblock The concept of probability if safety assessments of technological systems.
\newblock \emph{Science}, 250\penalty0 (4986):\penalty0 1359--1364, 1991.

\bibitem[Band et~al.(2022)Band, Rudner, Feng, Filos, Nado, Dusenberry, Jerfel, Tran, and Gal]{Band:22}
N.~Band, T.~G.~J. Rudner, Q.~Feng, A.~Filos, Z.~Nado, M.~W. Dusenberry, G.~Jerfel, D.~Tran, and Y.~Gal.
\newblock Benchmarking bayesian deep learning on diabetic retinopathy detection tasks.
\newblock \emph{ArXiv}, 2211.12717, 2022.

\bibitem[Biggio et~al.(2013)Biggio, Corona, Maiorca, Nelson, {\v{S}}rndi{\'{c}}, Laskov, Giacinto, and Roli]{Biggio:13}
B.~Biggio, I.~Corona, D.~Maiorca, B.~Nelson, N.~{\v{S}}rndi{\'{c}}, P.~Laskov, G.~Giacinto, and F.~Roli.
\newblock Evasion attacks against machine learning at test time.
\newblock In H.~Blockeel, K.~Kersting, S.~Nijssen, and F.~{\v{Z}}elezn{\'y}, editors, \emph{Machine Learning and Knowledge Discovery in Databases}, pages 387--402, Berlin, Heidelberg, 2013. Springer Berlin Heidelberg.

\bibitem[Blundell et~al.(2015)Blundell, Cornebise, Kavukcuoglu, and Wierstra]{Blundell:15}
C.~Blundell, J.~Cornebise, K.~Kavukcuoglu, and D.~Wierstra.
\newblock Weight uncertainty in neural network.
\newblock In F.~Bach and D.~Blei, editors, \emph{Proceedings of the 32nd International Conference on Machine Learning}, volume~37 of \emph{Proceedings of Machine Learning Research}, pages 1613--1622, 2015.

\bibitem[Bradshaw et~al.(2017)Bradshaw, de~G.~Matthews, and Ghahramani]{Bradshaw:17}
J.~Bradshaw, A.~G. de~G.~Matthews, and Z.~Ghahramani.
\newblock Adversarial examples, uncertainty, and transfer testing robustness in gaussian process hybrid deep networks.
\newblock \emph{ArXiv}, 1707.02476, 2017.

\bibitem[Caldeira and Nord(2020)]{Caldeira:20}
J.~Caldeira and B.~Nord.
\newblock Deeply uncertain: comparing methods of uncertainty quantification in deep learning algorithms.
\newblock \emph{Machine Learning: Science and Technology}, 2\penalty0 (1):\penalty0 015002, 2020.

\bibitem[Capp\'{e} et~al.(2004)Capp\'{e}, Guillin, Marin, and Robert]{Cappe:04}
O.~Capp\'{e}, A.~Guillin, J.-M. Marin, and C.~P. Robert.
\newblock Population {Monte Carlo}.
\newblock \emph{Journal of Computational and Graphical Statistics}, 13\penalty0 (4):\penalty0 907--929, 2004.

\bibitem[Chen et~al.(2014)Chen, Fox, and Guestrin]{Chen:14}
T.~Chen, E.~Fox, and C.~Guestrin.
\newblock Stochastic gradient {Hamiltonian} {Monte Carlo}.
\newblock In \emph{International Conference on Machine Learning}, pages 1683--1691. Proceedings of Machine Learning Research, 2014.

\bibitem[Clanuwat et~al.(2018)Clanuwat, Bober-Irizar, Kitamoto, Lamb, Yamamoto, and Ha]{Clanuwat:18}
T.~Clanuwat, M.~Bober-Irizar, A.~Kitamoto, A.~Lamb, K.~Yamamoto, and D.~Ha.
\newblock Deep learning for classical japanese literature.
\newblock \emph{ArXiv}, 1812.01718, 2018.

\bibitem[Cobb and Jalaian(2021)]{Cobb:21}
A.~D. Cobb and B.~Jalaian.
\newblock Scaling hamiltonian monte carlo inference for bayesian neural networks with symmetric splitting.
\newblock \emph{Uncertainty in Artificial Intelligence}, 2021.

\bibitem[Cohen et~al.(2017)Cohen, Afshar, Tapson, and Schaik]{Cohen:17}
G.~Cohen, S.~Afshar, J.~Tapson, and A.~Van Schaik.
\newblock Emnist: Extending mnist to handwritten letters.
\newblock In \emph{2017 international joint conference on neural networks}. IEEE, 2017.

\bibitem[Cover and Thomas(2006)]{Cover:06}
T.~M. Cover and J.~A. Thomas.
\newblock \emph{Elements of Information Theory}.
\newblock Wiley Series in Telecommunications and Signal Processing. Wiley-Interscience, 2nd edition, 2006.
\newblock ISBN 0471241954.

\bibitem[Daxberger et~al.(2021)Daxberger, Kristiadi, Immer, Eschenhagen, Bauer, and Hennig]{Daxberger:21}
E.~Daxberger, A.~Kristiadi, A.~Immer, R.~Eschenhagen, M.~Bauer, and P.~Hennig.
\newblock Laplace redux-effortless bayesian deep learning.
\newblock \emph{Advances in Neural Information Processing Systems}, 34:\penalty0 20089--20103, 2021.

\bibitem[Deng et~al.(2009)Deng, Dong, Socher, Li, Li, and Fei-Fei]{Deng:09}
J.~Deng, W.~Dong, R.~Socher, L.-J. Li, K.~Li, and L.~Fei-Fei.
\newblock Imagenet: A large-scale hierarchical image database.
\newblock In \emph{2009 IEEE conference on computer vision and pattern recognition}. IEEE, 2009.

\bibitem[Depeweg et~al.(2018)Depeweg, Hern{\'{a}}ndez{-}Lobato, Doshi{-}Velez, and Udluft]{Depeweg:18}
S.~Depeweg, J.~M. Hern{\'{a}}ndez{-}Lobato, F.~Doshi{-}Velez, and S.~Udluft.
\newblock Decomposition of uncertainty in {Bayesian} deep learning for efficient and risk-sensitive learning.
\newblock In \emph{Proceedings of the 35th International Conference on Machine Learning}, volume~80 of \emph{Proceedings of Machine Learning Research}, pages 1192--1201, 2018.

\bibitem[Durasov et~al.(2021)Durasov, Bagautdinov, Baque, and Fua]{Durasov:21}
N.~Durasov, T.~Bagautdinov, P.~Baque, and P.~Fua.
\newblock Masksembles for uncertainty estimation.
\newblock In \emph{Proceedings of the IEEE/CVF Conference on Computer Vision and Pattern Recognition}, pages 13539--13548, 2021.

\bibitem[Elvira et~al.(2015)Elvira, Martino, Luengo, and Bugallo]{Elvira:15}
V.~Elvira, L.~Martino, D.~Luengo, and M.~F. Bugallo.
\newblock Efficient multiple importance sampling estimators.
\newblock \emph{{IEEE} Signal Processing Letters}, 22\penalty0 (10):\penalty0 1757--1761, 2015.

\bibitem[Elvira et~al.(2019)Elvira, Martino, Luengo, and Bugallo]{Elvira:19}
V.~Elvira, L.~Martino, D.~Luengo, and M.~F. Bugallo.
\newblock Generalized multiple importance sampling.
\newblock \emph{Statistical Science}, 34\penalty0 (1), 2019.

\bibitem[Eschenhagen et~al.(2021)Eschenhagen, Daxberger, Hennig, and Kristiadi]{Eschenhagen:21}
R.~Eschenhagen, E.~Daxberger, P.~Hennig, and A.~Kristiadi.
\newblock Mixtures of laplace approximations for improved post-hoc uncertainty in deep learning.
\newblock \emph{arXiv}, 2111.03577, 2021.

\bibitem[Fiacco and McCormick(1990)]{Fiacco:90}
A.~V. Fiacco and G.~P. McCormick.
\newblock \emph{Nonlinear programming: sequential unconstrained minimization techniques}.
\newblock Society for Industrial and Applied Mathematics, 1990.

\bibitem[Filos et~al.(2019)Filos, Farquhar, Gomez, Rudner, Kenton, Smith, Alizadeh, Kroon, and Gal]{Filos:19}
A.~Filos, S.~Farquhar, A.~N. Gomez, T.~G.~J. Rudner, Z.~Kenton, L.~Smith, M.~Alizadeh, A.~De Kroon, and Y.~Gal.
\newblock A systematic comparison of bayesian deep learning robustness in diabetic retinopathy tasks.
\newblock \emph{ArXiv}, 1912.10481, 2019.

\bibitem[Fort et~al.(2019)Fort, Hu, and Lakshminarayanan]{Fort:19}
S.~Fort, H.~Hu, and B.~Lakshminarayanan.
\newblock Deep ensembles: A loss landscape perspective.
\newblock \emph{ArXiv}, 1912.02757, 2019.

\bibitem[Gal(2016)]{Gal:16thesis}
Y.~Gal.
\newblock \emph{Uncertainty in Deep Learning}.
\newblock PhD thesis, Department of Engineering, University of Cambridge, 2016.

\bibitem[Gal and Ghahramani(2016)]{Gal:16}
Y.~Gal and Z.~Ghahramani.
\newblock Dropout as a {Bayesian} approximation: Representing model uncertainty in deep learning.
\newblock In \emph{Proceedings of the 33nd International Conference on Machine Learning}, 2016.

\bibitem[Gawlikowski et~al.(2021)Gawlikowski, Tassi, Ali, Lee, Humt, Feng, Kruspe, Triebel, Jung, Roscher, Shahzad, Yang, Bamler, and Zhu]{Gawlikowski:21}
J.~Gawlikowski, C.~R.~N. Tassi, M.~Ali, J.~Lee, M.~Humt, J.~Feng, A.~Kruspe, R.~Triebel, P.~Jung, R.~Roscher, M.~Shahzad, W.~Yang, R.~Bamler, and X.~X. Zhu.
\newblock A survey of uncertainty in deep neural networks.
\newblock \emph{ArXiv}, 2107.03342, 2021.

\bibitem[Goodfellow et~al.(2014)Goodfellow, Pouget-Abadie, Mirza, Xu, Warde-Farley, Ozair, Courville, and Bengio]{Goodfellow:14}
I.~J. Goodfellow, J.~Pouget-Abadie, M.~Mirza, B.~Xu, D.~Warde-Farley, S.~Ozair, A.~Courville, and Y.~Bengio.
\newblock Generative adversarial nets.
\newblock In Z.~Ghahramani, M.~Welling, C.~Cortes, N.~Lawrence, and K.Q. Weinberger, editors, \emph{Advances in Neural Information Processing Systems}, volume~27. Curran Associates, Inc., 2014.

\bibitem[Goodfellow et~al.(2015)Goodfellow, Shlens, and Szegedy]{Goodfellow:15}
I.~J. Goodfellow, J.~Shlens, and C.~Szegedy.
\newblock Explaining and harnessing adversarial examples.
\newblock \emph{arXiv}, 1412.6572, 2015.

\bibitem[Graves(2011)]{Graves:11}
A.~Graves.
\newblock Practical variational inference for neural networks.
\newblock In J.~Shawe-Taylor, R.~Zemel, P.~Bartlett, F.~Pereira, and K.Q. Weinberger, editors, \emph{Advances in Neural Information Processing Systems}, volume~24. Curran Associates, Inc., 2011.

\bibitem[Guo et~al.(2017)Guo, Pleiss, Sun, and Weinberger]{Guo:17}
C.~Guo, G.~Pleiss, Y.~Sun, and K.~Q. Weinberger.
\newblock On calibration of modern neural networks.
\newblock In D.~Precup and Y.~W. Teh, editors, \emph{Proceedings of the 34th International Conference on Machine Learning}, volume~70 of \emph{Proceedings of Machine Learning Research}, pages 1321--1330, 2017.

\bibitem[Gustafsson et~al.(2020)Gustafsson, Danelljan, and Sch{\"o}n]{Gustafsson:20}
F.~K. Gustafsson, M.~Danelljan, and T.~B. Sch{\"o}n.
\newblock Evaluating scalable bayesian deep learning methods for robust computer vision.
\newblock In \emph{Proceedings of the IEEE/CVF Conference on Computer Vision and Pattern Recognition (CVPR) Workshops}, June 2020.

\bibitem[Hale(2001)]{Hale:01}
J.~Hale.
\newblock A probabilistic earley parser as a psycholinguistic model.
\newblock In \emph{Proceedings of the Annual Conference of the North American Chapter of the Association for Computational Linguistics}, pages 1--8. Association for Computational Linguistics, 2001.

\bibitem[Helton(1993)]{Helton:93}
J.~C. Helton.
\newblock Risk, uncertainty in risk, and the {EPA} release limits for radioactive waste disposal.
\newblock \emph{Nuclear Technology}, 101\penalty0 (1):\penalty0 18--39, 1993.

\bibitem[Helton(1997)]{Helton:97}
J.~C. Helton.
\newblock Uncertainty and sensitivity analysis in the presence of stochastic and subjective uncertainty.
\newblock \emph{Journal of Statistical Computation and Simulation}, 57\penalty0 (1-4):\penalty0 3--76, 1997.

\bibitem[Hendrycks et~al.(2021)Hendrycks, Zhao, Basart, Steinhardt, and Song]{Hendrycks:21}
D.~Hendrycks, K.~Zhao, S.~Basart, J.~Steinhardt, and D.~Song.
\newblock Natural adversarial examples.
\newblock \emph{IEEE/CVF Conference on Computer Vision and Pattern Recognition}, 2021.

\bibitem[Hesterberg(1995)]{Hesterberg:95}
T.~Hesterberg.
\newblock Weighted average importance sampling and defensive mixture distributions.
\newblock \emph{Technometrics}, 37\penalty0 (2):\penalty0 185--194, 1995.

\bibitem[Hesterberg(1996)]{Hesterberg:96}
T.~Hesterberg.
\newblock Estimates and confidence intervals for importance sampling sensitivity analysis.
\newblock \emph{Mathematical and Computer Modelling}, 23\penalty0 (8):\penalty0 79--85, 1996.

\bibitem[Houlsby et~al.(2011)Houlsby, Huszar, Ghahramani, and Lengyel]{Houlsby:11}
N.~Houlsby, F.~Huszar, Z.~Ghahramani, and M.~Lengyel.
\newblock Bayesian active learning for classification and preference learning.
\newblock \emph{ArXiv}, 1112.5745, 2011.

\bibitem[Houlsby et~al.(2019)Houlsby, Giurgiu, Jastrzebski, Morrone, Laroussilhe, Gesmundo, Attariyan, and Gelly]{Houlsby:19}
N.~Houlsby, A.~Giurgiu, S.~Jastrzebski, B.~Morrone, Q.~De Laroussilhe, A.~Gesmundo, M.~Attariyan, and S.~Gelly.
\newblock Parameter-efficient transfer learning for nlp.
\newblock In \emph{International Conference on Machine Learning}. Proceedings of Machine Learning Research, 2019.

\bibitem[Hu et~al.(2021)Hu, Shen, Wallis, Allen-Zhu, Li, Wang, Wang, and Chen]{Hu:21}
E.~J. Hu, Y.~Shen, P.~Wallis, Z.~Allen-Zhu, Y.~Li, S.~Wang, L.~Wang, and W.~Chen.
\newblock Lora: Low-rank adaptation of large language models.
\newblock \emph{ArXiv}, 2106.09685, 2021.

\bibitem[Huang et~al.(2017)Huang, Li, Pleiss, Liu, Hopcroft, and Weinberger]{Huang:17}
G.~Huang, Y.~Li, G.~Pleiss, Z.~Liu, J.~E. Hopcroft, and K.~Q. Weinberger.
\newblock Snapshot ensembles: Train 1, get m for free.
\newblock \emph{ArXiv}, 1704.00109, 2017.

\bibitem[H\"{u}llermeier and Waegeman(2021)]{Huellermeier:21}
E.~H\"{u}llermeier and W.~Waegeman.
\newblock Aleatoric and epistemic uncertainty in machine learning: an introduction to concepts and methods.
\newblock \emph{Machine Learning}, 3\penalty0 (110):\penalty0 457--506, 2021.

\bibitem[Izmailov et~al.(2021)Izmailov, Vikram, Hoffman, and Wilson]{Izmailov:21}
P.~Izmailov, S.~Vikram, M.~D. Hoffman, and A.~G. Wilson.
\newblock What are {Bayesian} neural network posteriors really like?
\newblock In \emph{Proceedings of the 38th International Conference on Machine Learning}, pages 4629--4640, 2021.

\bibitem[Jaynes(1957)]{Jaynes:57}
E.~T. Jaynes.
\newblock Information theory and statistical mechanics.
\newblock \emph{Phys. Rev.}, 106:\penalty0 620--630, 1957.

\bibitem[Kapoor(2023)]{Kapoor:23}
S.~Kapoor.
\newblock torch-sgld: Sgld as pytorch optimizer.
\newblock \url{https://pypi.org/project/torch-sgld/}, 2023.
\newblock Accessed: 12-05-2023.

\bibitem[Karush(1939)]{Karush:39}
W.~Karush.
\newblock Minima of functions of several variables with inequalities as side conditions.
\newblock Master's thesis, University of Chicago, 1939.

\bibitem[Kendall and Gal(2017)]{Kendall:17}
A.~Kendall and Y.~Gal.
\newblock What uncertainties do we need in {Bayesian} deep learning for computer vision?
\newblock In \emph{Advances in Neural Information Processing Systems}, volume~30, 2017.

\bibitem[Kingma and Ba(2014)]{Kingma:14b}
D.~P. Kingma and J.~Ba.
\newblock Adam: A method for stochastic optimization.
\newblock \emph{ArXiv}, 1412.6980, 2014.

\bibitem[Kingma and Welling(2014)]{Kingma:14}
D.~P. Kingma and M.~Welling.
\newblock Auto-encoding variational {Bayes}.
\newblock In \emph{2nd International Conference on Learning Representations}, 2014.

\bibitem[Kirichenko et~al.(2022)Kirichenko, Izmailov, and Wilson]{Kirichenko:22}
P.~Kirichenko, P.~Izmailov, and A.~G. Wilson.
\newblock Last layer re-training is sufficient for robustness to spurious correlations.
\newblock \emph{ArXiv}, 2204.02937, 2022.

\bibitem[Kuhn and Tucker(1950)]{Kuhn:50}
H.~W. Kuhn and A.~W. Tucker.
\newblock Nonlinear programming.
\newblock In J.~Neyman, editor, \emph{Second Berkeley Symposium on Mathematical Statistics and Probability}, pages 481--492, Berkeley, 1950. University of California Press.

\bibitem[Kviman et~al.(2022)Kviman, Melin, Koptagel, Elvira, and Lagergren]{Kviman:22}
O.~Kviman, H.~Melin, H.~Koptagel, V.~Elvira, and J.~Lagergren.
\newblock Multiple importance sampling {ELBO }and deep ensembles of variational approximations.
\newblock In \emph{Proceedings of the 25th International Conference on Artificial Intelligence and Statistics}, volume 151 of \emph{Proceedings of Machine Learning Research}, pages 10687--10702, 2022.

\bibitem[Lake et~al.(2015)Lake, Salakhutdinov, and Tenenbaum]{Lake:15}
B.~M. Lake, R.~Salakhutdinov, and J.~B. Tenenbaum.
\newblock Human-level concept learning through probabilistic program induction.
\newblock \emph{Science}, 2015.

\bibitem[Lakshminarayanan et~al.(2017)Lakshminarayanan, Pritzel, and Blundell]{Lakshminarayanan:17}
B.~Lakshminarayanan, A.~Pritzel, and C.~Blundell.
\newblock Simple and scalable predictive uncertainty estimation using deep ensembles.
\newblock In \emph{Proceedings of the 31st International Conference on Neural Information Processing Systems}, page 6405–6416. Curran Associates Inc., 2017.

\bibitem[LeCun et~al.(1998)LeCun, Bottou, Bengio, and Haffner]{LeCun:98}
Y.~LeCun, L.~Bottou, Y.~Bengio, and P.~Haffner.
\newblock Gradient-based learning applied to document recognition.
\newblock \emph{Proceedings of the IEEE}, 86\penalty0 (11):\penalty0 2278--2324, 1998.

\bibitem[Li et~al.(2016)Li, Chen, Carlson, and Carin]{Li:16}
C.~Li, C.~Chen, D.~Carlson, and L.~Carin.
\newblock Preconditioned stochastic gradient {Langevin} dynamics for deep neural networks.
\newblock In \emph{Proceedings of the Thirtieth AAAI Conference on Artificial Intelligence (AAAI)}, page 1788–1794. AAAI Press, 2016.

\bibitem[Liu et~al.(2020)Liu, Lin, Padhy, Tran, Weiss, and Lakshminarayanan]{Liu:20}
J.~Liu, Z.~Lin, S.~Padhy, D.~Tran, T.~Bedrax Weiss, and B.~Lakshminarayanan.
\newblock Simple and principled uncertainty estimation with deterministic deep learning via distance awareness.
\newblock \emph{Advances in Neural Information Processing Systems}, 33:\penalty0 7498--7512, 2020.

\bibitem[Lubana et~al.(2022)Lubana, Bigelow, Dick, Krueger, and Tanaka]{Lubana:22}
E.~S. Lubana, E.~J. Bigelow, R.~P. Dick, D.~Krueger, and H.~Tanaka.
\newblock Mechanistic mode connectivity.
\newblock \emph{ArXiv}, 2211.08422, 2022.

\bibitem[Luenberger and Ye(2016)]{Luenberger:16}
D.~G. Luenberger and Y.~Ye.
\newblock \emph{Linear and nonlinear programming}.
\newblock International Series in Operations Research and Management Science. Springer, 2016.

\bibitem[MacKay(1992)]{MacKay:92b}
D.~J.~C. MacKay.
\newblock A practical {Bayesian} framework for backprop networks.
\newblock \emph{Neural Computation}, 4:\penalty0 448--472, 1992.

\bibitem[Maddox et~al.(2019)Maddox, Izmailov, Garipov, Vetrov, and Wilson]{Maddox:19}
W.~J. Maddox, P.~Izmailov, T.~Garipov, D.~P. Vetrov, and A.~G. Wilson.
\newblock A simple baseline for {Bayesian} uncertainty in deep learning.
\newblock In \emph{Advances in Neural Information Processing Systems}, 2019.

\bibitem[Malinin and Gales(2018)]{Malinin:18}
A.~Malinin and M.~Gales.
\newblock Predictive uncertainty estimation via prior networks.
\newblock \emph{Advances in neural information processing systems}, 31, 2018.

\bibitem[May(2020)]{May:20}
R.~May.
\newblock A simple proof of the {Karush-Kuhn-Tucker} theorem with finite number of equality and inequality constraints.
\newblock \emph{ArXiv}, 2007.12483, 2020.

\bibitem[McKone(1994)]{McKone:94}
T.~E. McKone.
\newblock Uncertainty and variability in human exposures to soil contaminants through home-grown food: A monte carlo assessment.
\newblock \emph{Risk Analysis}, 14, 1994.

\bibitem[Mobiny et~al.(2021)Mobiny, Yuan, Moulik, Garg, Wu, and VanNguyen]{Mobiny:21}
A.~Mobiny, P.~Yuan, S.~K. Moulik, N.~Garg, C.~C. Wu, and H.~VanNguyen.
\newblock {DropConnect} is effective in modeling uncertainty of {Bayesian} deep networks.
\newblock \emph{Scientific Reports}, 11:\penalty0 5458, 2021.

\bibitem[Mukhoti et~al.(2021)Mukhoti, Kirsch, vanAmersfoort, Torr, and Gal]{Mukhoti:21}
J.~Mukhoti, A.~Kirsch, J.~vanAmersfoort, P.~H.~S. Torr, and Y.~Gal.
\newblock Deep deterministic uncertainty: A simple baseline.
\newblock \emph{ArXiv}, 2102.11582, 2021.

\bibitem[Neal(1996)]{Neal:96}
R.~Neal.
\newblock \emph{Bayesian Learning for Neural Networks}.
\newblock Springer Verlag, New York, 1996.

\bibitem[Oleksiienko et~al.(2022)Oleksiienko, Tran, and Iosifidis]{Oleksiienko:22}
I.~Oleksiienko, D.~T. Tran, and A.~Iosifidis.
\newblock Variational neural networks.
\newblock \emph{ArXiv}, 2207.01524, 2022.

\bibitem[Osband et~al.(2021)Osband, Wen, Asghari, Dwaracherla, Ibrahimi, Lu, and VanRoy]{Osband:21}
I.~Osband, Z.~Wen, S.~M. Asghari, V.~Dwaracherla, M.~Ibrahimi, X.~Lu, and B.~VanRoy.
\newblock Weight uncertainty in neural networks.
\newblock \emph{ArXiv}, 2107.08924, 2021.

\bibitem[Ovadia et~al.(2019)Ovadia, Fertig, Ren, Nado, Sculley, Nowozin, Dillon, Lakshminarayanan, and Snoek]{Ovadia:19}
Y.~Ovadia, E.~Fertig, J.~Ren, Z.~Nado, D.~Sculley, S.~Nowozin, J.~Dillon, B.~Lakshminarayanan, and J.~Snoek.
\newblock Can you trust your model's uncertainty? evaluating predictive uncertainty under dataset shift.
\newblock In \emph{Advances in Neural Information Processing Systems}, volume~32, 2019.

\bibitem[Owen and Zhou(2000)]{Owen:00}
A.~Owen and Y.~Zhou.
\newblock Safe and effective importance sampling.
\newblock \emph{Journal of the American Statistical Association}, 95\penalty0 (449):\penalty0 135--143, 2000.

\bibitem[Parker-Holder et~al.(2020)Parker-Holder, Metz, Resnick, Hu, Lerer, Letcher, Peysakhovich, Pacchiano, and Foerster]{ParkerHolder:20}
J.~Parker-Holder, L.~Metz, C.~Resnick, H.~Hu, A.~Lerer, A.~Letcher, A.~Peysakhovich, A.~Pacchiano, and J.~Foerster.
\newblock Ridge rider: Finding diverse solutions by following eigenvectors of the hessian.
\newblock In H.~Larochelle, M.~Ranzato, R.~Hadsell, M.F. Balcan, and H.~Lin, editors, \emph{Advances in Neural Information Processing Systems}, volume~33, pages 753--765. Curran Associates, Inc., 2020.

\bibitem[Paszke et~al.(2019)Paszke, Gross, Massa, Lerer, Bradbury, Chanan, Killeen, Lin, Gimelshein, Antiga, et~al.]{Paszke:19}
A.~Paszke, S.~Gross, F.~Massa, A.~Lerer, J.~Bradbury, G.~Chanan, T.~Killeen, Z.~Lin, N.~Gimelshein, L.~Antiga, et~al.
\newblock Pytorch: An imperative style, high-performance deep learning library.
\newblock \emph{Advances in Neural Information Processing Systems}, 32, 2019.

\bibitem[Pedregosa et~al.(2011)Pedregosa, Varoquaux, Gramfort, Michel, Thirion, Grisel, Blondel, Prettenhofer, Weiss, Dubourg, et~al.]{Pedregosa:11}
F.~Pedregosa, G.~Varoquaux, A.~Gramfort, V.~Michel, B.~Thirion, O.~Grisel, M.~Blondel, P.~Prettenhofer, R.~Weiss, V.~Dubourg, et~al.
\newblock Scikit-learn: Machine learning in python.
\newblock \emph{The Journal of Machine Learning Research}, 12:\penalty0 2825--2830, 2011.

\bibitem[Postels et~al.(2021)Postels, Segu, Sun, Gool, Yu, and Tombari]{Postels:21}
J.~Postels, M.~Segu, T.~Sun, L.~Van Gool, F.~Yu, and F.~Tombari.
\newblock On the practicality of deterministic epistemic uncertainty.
\newblock \emph{ArXiv}, 2107.00649, 2021.

\bibitem[Raftery and Bao(2010)]{Raftery:10}
A.~E. Raftery and L.~Bao.
\newblock Estimating and projecting trends in {HIV/AIDS} generalized epidemics using incremental mixture importance sampling.
\newblock \emph{Biometrics}, 66, 2010.

\bibitem[Rigter et~al.(2022)Rigter, Lacerda, and Hawes]{Rigter:22}
M.~Rigter, B.~Lacerda, and N.~Hawes.
\newblock Rambo-rl: Robust adversarial model-based offline reinforcement learning.
\newblock In S.~Koyejo, S.~Mohamed, A.~Agarwal, D.~Belgrave, K.~Cho, and A.~Oh, editors, \emph{Advances in Neural Information Processing Systems}, volume~35, pages 16082--16097. Curran Associates, Inc., 2022.

\bibitem[Scimeca et~al.(2022)Scimeca, Oh, Chun, Poli, and Yun]{Scimeca:21}
L.~Scimeca, S.~J. Oh, S.~Chun, M.~Poli, and S.~Yun.
\newblock Which shortcut cues will dnns choose? a study from the parameter-space perspective.
\newblock \emph{arXiv}, 2110.03095, 2022.

\bibitem[Seidenfeld(1986)]{Seidenfeld:86}
T.~Seidenfeld.
\newblock Entropy and uncertainty.
\newblock \emph{Philosophy of Science}, 53\penalty0 (4):\penalty0 467--491, 1986.

\bibitem[Shah et~al.(2020)Shah, Tamuly, Raghunathan, Jain, and Netrapalli]{Shah:20}
H.~Shah, K.~Tamuly, A.~Raghunathan, P.~Jain, and P.~Netrapalli.
\newblock The pitfalls of simplicity bias in neural networks.
\newblock In H.~Larochelle, M.~Ranzato, R.~Hadsell, M.F. Balcan, and H.~Lin, editors, \emph{Advances in Neural Information Processing Systems}, volume~33, pages 9573--9585. Curran Associates, Inc., 2020.

\bibitem[Shannon and Elwood(1948)]{Shannon:48}
C.~E. Shannon and C.~Elwood.
\newblock A mathematical theory of communication.
\newblock \emph{The Bell System Technical Journal}, 27:\penalty0 379--423, 1948.

\bibitem[Smith and Gal(2018)]{Smith:18}
L.~Smith and Y.~Gal.
\newblock Understanding measures of uncertainty for adversarial example detection.
\newblock In A.~Globerson and R.~Silva, editors, \emph{Proceedings of the Thirty-Fourth Conference on Uncertainty in Artificial Intelligence}, pages 560--569. AUAI Press, 2018.

\bibitem[Steele et~al.(2006)Steele, Raftery, and Emond]{Steele:06}
R.~J. Steele, A.~E. Raftery, and M.~J. Emond.
\newblock Computing normalizing constants for finite mixture models via incremental mixture importance sampling {(IMIS)}.
\newblock \emph{Journal of Computational and Graphical Statistics}, 15, 2006.

\bibitem[Sundararajan et~al.(2017)Sundararajan, Taly, and Yan]{Sundararajan:17}
M.~Sundararajan, A.~Taly, and Q.~Yan.
\newblock Axiomatic attribution for deep networks.
\newblock In \emph{International Conference on Machine Learning}, 2017.

\bibitem[Szegedy et~al.(2013)Szegedy, Zaremba, Sutskever, Bruna, Erhan, Goodfellow, and Fergus]{Szegedy:13}
C.~Szegedy, W.~Zaremba, I.~Sutskever, J.~Bruna, D.~Erhan, I.~J. Goodfellow, and R.~Fergus.
\newblock Intriguing properties of neural networks.
\newblock \emph{ArXiv}, 1312.6199, 2013.

\bibitem[Tan and Le(2019)]{Tan:19}
M.~Tan and Q.~Le.
\newblock Efficientnet: Rethinking model scaling for convolutional neural networks.
\newblock In \emph{International conference on machine learning}, pages 6105--6114. Proceedings of Machine Learning Research, 2019.

\bibitem[Tribus(1961)]{Tribus:61}
M.~Tribus.
\newblock \emph{Thermostatics and Thermodynamics: An Introduction to Energy, Information and States of Matter, with Engineering Applications}.
\newblock University series in basic engineering. Van Nostrand, 1961.

\bibitem[van Amersfoort et~al.(2020)van Amersfoort, Smith, Teh, and Gal]{vAmersfoort:20}
J.~van Amersfoort, L.~Smith, Y.~W. Teh, and Y.~Gal.
\newblock Uncertainty estimation using a single deep deterministic neural network.
\newblock In \emph{International Conference on Machine Learning}, pages 9690--9700. Proceedings of Machine Learning Research, 2020.

\bibitem[van Amersfoort et~al.(2021)van Amersfoort, Smith, Jesson, Key, and Gal]{vAmersfoort:21}
J.~van Amersfoort, L.~Smith, A.~Jesson, O.~Key, and Y.~Gal.
\newblock On feature collapse and deep kernel learning for single forward pass uncertainty.
\newblock \emph{ArXiv}, 2102.11409, 2021.

\bibitem[Veach and Guibas(1995)]{Veach:95}
E.~Veach and L.~J. Guibas.
\newblock Optimally combining sampling techniques for {Monte Carlo} rendering.
\newblock In \emph{Proceedings of the 22nd Annual Conference on Computer Graphics and Interactive Techniques}, pages 419--428. Association for Computing Machinery, 1995.

\bibitem[Vesely and Rasmuson(1984)]{Vesely:84}
W.~E. Vesely and D.~M. Rasmuson.
\newblock Uncertainties in nuclear probabilistic risk analyses.
\newblock \emph{Risk Analysis}, 4, 1984.

\bibitem[Weinzierl(2000)]{Weinzierl:00}
S.~Weinzierl.
\newblock Introduction to {Monte Carlo} methods.
\newblock \emph{ArXiv}, hep-ph/0006269, 2000.

\bibitem[Welling and Teh(2011)]{Welling:11}
M.~Welling and Y.~W. Teh.
\newblock Bayesian learning via stochastic gradient {Langevin} dynamics.
\newblock In \emph{Proceedings of the 28th International Conference on Machine Learning}, page 681–688, Madison, WI, USA, 2011. Omnipress.

\bibitem[Wilson and Izmailov(2020)]{Wilson:20}
A.~G. Wilson and P.~Izmailov.
\newblock Bayesian deep learning and a probabilistic perspective of generalization.
\newblock \emph{Advances in Neural Information Processing Systems}, 33:\penalty0 4697--4708, 2020.

\bibitem[Wursthorn et~al.(2022)Wursthorn, Hillemann, and M.~Ulrich]{Wursthorn:22}
K.~Wursthorn, M.~Hillemann, and M.~M.~Ulrich.
\newblock Comparison of uncertainty quantification methods for {CNN}-based regression.
\newblock \emph{The International Archives of the Photogrammetry, Remote Sensing and Spatial Information Sciences}, XLIII-B2-2022:\penalty0 721--728, 2022.

\bibitem[Xiao et~al.(2017)Xiao, Rasul, and Vollgraf]{Xiao:17}
H.~Xiao, K.~Rasul, and R.~Vollgraf.
\newblock Fashion-mnist: a novel image dataset for benchmarking machine learning algorithms.
\newblock \emph{ArXiv}, 1708.07747, 2017.

\bibitem[Zangwill(1967)]{Zangwill:67}
W.~I. Zangwill.
\newblock Non-linear programming via penalty functions.
\newblock \emph{Management Science}, 13\penalty0 (5):\penalty0 344--358, 1967.

\bibitem[Zhang et~al.(2020)Zhang, Li, Zhang, Chen, and Wilson]{Zhang:20}
R.~Zhang, C.~Li, J.~Zhang, C.~Chen, and A.~G. Wilson.
\newblock Cyclical stochastic gradient {MCMC} for {Bayesian} deep learning.
\newblock In \emph{International Conference on Learning Representations}, 2020.

\bibitem[Zhang et~al.(2022)Zhang, Wilson, and DeSa]{Zhang:22}
R.~Zhang, A.~G. Wilson, and C.~DeSa.
\newblock Low-precision stochastic gradient {Langevin} dynamics.
\newblock \emph{ArXiv}, 2206.09909, 2022.

\bibitem[Zidek and vanEeden(2003)]{Zidek:03}
J.~V. Zidek and C.~vanEeden.
\newblock Uncertainty, entropy, variance and the effect of partial information.
\newblock \emph{Lecture Notes-Monograph Series}, 42:\penalty0 155–167, 2003.

\end{thebibliography}
\bibliographystyle{plainnat}

\newpage
\appendix
\counterwithin{figure}{section}
\counterwithin{table}{section}

\addtocontents{toc}{\protect\setcounter{tocdepth}{3}}

\section*{Appendix}

This is the appendix of the paper ''\textbf{Quantification of Uncertainty with Adversarial Models}''.
It consists of three sections.
In view of the increasing influence of
contemporary machine learning research on the broader public,
section \ref{apx:sec:societal_impact} gives a societal impact statement.
Following to this, section \ref{apx:sec:theory} gives details of our theoretical results,
foremost about the measure of uncertainty used throughout our work.
Furthermore, Mixture Importance Sampling for variance reduction is discussed.
Finally, section \ref{apx:sec:experimental_details} gives details about the experiments
presented in the main paper, as well as further experiments.

\renewcommand*{\contentsname}{Contents of the Appendix}
\tableofcontents
\listoffigures
\listoftables %

\section{Societal Impact Statement} \label{apx:sec:societal_impact}

In this work, we have focused on improving the predictive uncertainty estimation for machine learning models, specifically deep learning models. Our primary goal is to enhance the robustness and reliability of these predictions, which we believe have several positive societal impacts.

\begin{enumerate}
\item {\bf Improved decision-making:} By providing more accurate predictive uncertainty estimates, we enable a broad range of stakeholders to make more informed decisions. This could have implications across various sectors, including healthcare, finance, and autonomous vehicles, where decision-making based on machine learning predictions can directly affect human lives and economic stability.

\item {\bf Increased trust in machine learning systems:} By enhancing the reliability of machine learning models, our work may also contribute to increased public trust in these systems. This could foster greater acceptance and integration of machine learning technologies in everyday life, driving societal advancement.

\item {\bf Promotion of responsible machine learning:} Accurate uncertainty estimation is crucial for the responsible deployment of machine learning systems. By advancing this area, our work promotes the use of those methods in an ethical, transparent, and accountable manner.
\end{enumerate}

While we anticipate predominantly positive impacts, it is important to acknowledge potential negative impacts or challenges.

\begin{enumerate}
\item {\bf Misinterpretation of uncertainty:} Even with improved uncertainty estimates, there is a risk that these might be misinterpreted or misused, potentially leading to incorrect decisions or unintended consequences. It is vital to couple advancements in this field with improved education and awareness around the interpretation of uncertainty in AI systems.

\item {\bf Increased reliance on machine learning systems:} While increased trust in machine learning systems is beneficial, there is a risk it could lead to over-reliance on these systems, potentially resulting in reduced human oversight or critical thinking. It's important that robustness and reliability improvements don't result in blind trust.

\item {\bf Inequitable distribution of benefits:} As with any technological advancement, there is a risk that the benefits might not be evenly distributed, potentially exacerbating existing societal inequalities. We urge policymakers and practitioners to consider this when implementing our findings.
\end{enumerate}

In conclusion, while our work aims to make significant positive contributions to society, we believe it is essential to consider these potential negative impacts and take steps to mitigate them proactively.

\clearpage
\section{Theoretical Results} \label{apx:sec:theory}

\subsection{Measuring Predictive Uncertainty}\label{sec:theory:measures_of_uncertainty}

In this section, we first discuss the usage of the entropy and the cross-entropy
as measures of predictive uncertainty.
Following this, we introduce the two settings (a) and (b) 
(see Sec.~\ref{sec:definition_uncertainty}) in detail for the 
predictive distributions of probabilistic models in classification and regression.
Finally, we discuss Mixture Importance Sampling for variance reduction of the uncertainty estimator.

\subsubsection{Entropy and Cross-Entropy as Measures of Predictive Uncertainty}\label{sec:theory:cross_entropy}

\cite{Shannon:48} defines the entropy $\ENT{\Bp} = - \sum_{i=1}^N p_i \log p_i$ 
as a measure of the amount of uncertainty of a discrete probability
distribution $\Bp=(p_1,\ldots,p_N)$ and states that it measures 
how much ''choice'' is involved in the selection of a class $i$.
See also \cite{Jaynes:57,Cover:06} for an elaboration on this topic.
The value $- \log p_i$ has been called "surprisal" \citep{Tribus:61} 
(page 64, Subsection 2.9.1) and has been
used in computational linguistics \citep{Hale:01}.
Hence, the entropy is the expected or mean surprisal.
Instead of ''surprisal'' also the terms ''information content'', 
''self-information'', or ''Shannon information'' are used.

The cross-entropy $ \CE{\Bp}{\Bq} = - \sum_{i=1}^N p_i \log q_i$ between 
two discrete probability
distributions $~{\Bp=(p_1,\ldots,p_N)}$ and $\Bq=(q_1,\ldots,q_N)$ 
measures the expectation of the surprisal of $\Bq$ under distribution $\Bp$.
Like the entropy, the cross-entropy is a mean of surprisals, therefore
can be considered as a measure to quantify uncertainty.
The higher surprisals are on average, the higher the uncertainty.
The cross-entropy has increased uncertainty compared to the entropy since
more surprising events are expected when selecting events 
via $\Bp$ instead of $\Bq$.
Only if those distributions coincide, there is no additional surprisal and the cross-entropy
is equal to the entropy of the distributions.
The cross-entropy depends on the uncertainty of the two distributions and 
how different they are. 
In particular, high surprisal of $q_i$ and low
surprisal of $p_i$ strongly increase the cross-entropy since unexpected events
are more frequent, that is, we are more often surprised.
Thus, the cross-entropy does not only measure the uncertainty under distribution $\Bp$, 
but also the difference of the distributions.
The average surprisal via the cross-entropy depends on the uncertainty 
of $\Bp$ and the difference between $\Bp$ and $\Bq$:
\begin{align}
 \CE{\Bp}{\Bq} \ &= \ - \ \sum_{i=1}^N p_i \ \log q_i \\ \nonumber
  &= \ 
  - \sum_{i=1}^N p_i \ \log p_i \ + \   \sum_{i=1}^N p_i \ \log \frac{p_i}{q_i} 
  \\ \nonumber
  &= \ \ENT{\Bp} \ + \ \KL{\Bp}{\Bq} \ ,
\end{align}
where the Kullback-Leibler divergence $\KL{\cdot}{\cdot}$ is
\begin{align}
 \KL{\Bp}{\Bq} \ &= \ \sum_{i=1}^N p_i \log \frac{p_i}{q_i} \  .
\end{align}
The Kullback-Leibler divergence measures the difference in the distributions
via their average difference of surprisals.
Furthermore, it measures the decrease in uncertainty 
when shifting from the estimate $\Bp$ to the true $\Bq$ \citep{Seidenfeld:86,Adler:08}.

Therefore, the cross-entropy can serve to measure the total uncertainty,
where the entropy is used as aleatoric uncertainty 
and the difference of distributions is used as the epistemic uncertainty.
We assume that $\Bq$ is the true distribution that is estimated by
the distribution $\Bp$. We quantify the total uncertainty of
$\Bp$ as the sum of the entropy of $\Bp$ (aleatoric uncertainty) 
and the Kullback-Leibler divergence to $\Bq$ (epistemic uncertainty).
In accordance with \citet{Apostolakis:90} and \citet{Helton:97},
the aleatoric uncertainty
measures the stochasticity of sampling from $\Bp$, while
the epistemic uncertainty 
measures the deviation of the parameters $\Bp$
from the true parameters $\Bq$.

In the context of quantifying uncertainty through probability distributions, 
other measures such as the variance have been proposed \citep{Zidek:03}.
For uncertainty estimation in the context of deep learning systems, e.g. \cite{Gal:16thesis, Kendall:17, Depeweg:18} proposed to use the variance of the BMA predictive distribution 
as a measure of uncertainty.
Entropy and variance capture different notions of uncertainty and investigating measures based
on the variance of the predictive distribution is an interesting avenue for future work.

\subsubsection{Classification}

\paragraph{Setting (a): Expected uncertainty when selecting a model.}

We assume to have training data $\cD$ and an input $\Bx$.
We want to know the uncertainty in predicting a class $\By$ from $\Bx$
when we first choose a model $\tilde{\Bw}$  
based on the posterior $p(\tilde{\Bw} \mid \cD )$ 
an then use the chosen model $\tilde{\Bw}$ to choose a class for input $\Bx$
according to the predictive distribution $p(\By \mid \Bx, \tilde{\Bw} )$.
The uncertainty in predicting the class
arises from choosing a model (epistemic) and from
choosing a class using this probabilistic model (aleatoric).

Through Bayesian model averaging, we obtain the following probability of selecting a class:
\begin{align}
 p(\By \mid \Bx, \cD) \ &= \ 
   \int_{\cW }    p(\By \mid \Bx, \tilde{\Bw} )
 \ p(\tilde{\Bw} \mid \cD ) \ \Rd \tilde{\Bw}  \ .
\end{align}
The total uncertainty is commonly measured as the entropy of this probability distribution \citep{Houlsby:11,Gal:16thesis,Depeweg:18,Huellermeier:21}:
\begin{align}
   \ENT{p(\By \mid \Bx, \cD )}  \ .  
\end{align}
We can reformulate the total uncertainty as the expected cross-entropy:
\begin{align}
 \ENT{p(\By \mid \Bx, \cD )} \ &= \ - \ \sum_{\By \in \cY}   p(\By \mid \Bx, \cD ) \ \log p(\By \mid \Bx, \cD ) \\ \nonumber
 &= \ - \ \sum_{\By \in \cY} \log p(\By \mid \Bx, \cD ) 
 \int_{\cW } p(\By \mid \Bx, \tilde{\Bw} ) \ 
 \ p(\tilde{\Bw} \mid \cD ) \ \Rd \tilde{\Bw}  \\ \nonumber
 &= \ \int_{\cW } \left( - \ \sum_{\By \in \cY}  p(\By \mid \Bx, \tilde{\Bw} ) \ \log p(\By \mid \Bx, \cD ) \right) 
 \ p(\tilde{\Bw} \mid \cD )  \ \Rd \tilde{\Bw}  \\ \nonumber
 &= \int_{\cW }   \CE{p(\By \mid \Bx, \tilde{\Bw} )}{p(\By \mid \Bx, \cD )}
 \ p(\tilde{\Bw} \mid \cD ) \ \Rd \tilde{\Bw} \ .  
\end{align}
We can split the total uncertainty into the aleatoric and epistemic uncertainty \citep{Houlsby:11, Gal:16thesis, Smith:18}:
\begin{align}
\label{eq:ce_h_kl}
 \int_{\cW }  & \CE{p(\By \mid \Bx, \tilde{\Bw} )}{p(\By \mid \Bx, \cD )}
 \ p(\tilde{\Bw} \mid \cD ) \ \Rd \tilde{\Bw}  \\ \nonumber
 &= \
\int_{\cW }   \left( \ENT{p(\By \mid \Bx, \tilde{\Bw} )} \ + \ \KL{p(\By \mid \Bx, \tilde{\Bw} )}{p(\By \mid \Bx, \cD )} \right) 
 \ p(\tilde{\Bw} \mid \cD ) \ \Rd \tilde{\Bw} \\ \nonumber
&= \ \int_{\cW }    \ENT{p(\By \mid \Bx, \tilde{\Bw} )} 
 \ p(\tilde{\Bw} \mid \cD ) \ \Rd \tilde{\Bw}
\ + \ 
\int_{\cW } \KL{p(\By \mid \Bx, \tilde{\Bw} )}{p(\By \mid \Bx, \cD )}
 \ p(\tilde{\Bw} \mid \cD ) \ \Rd \tilde{\Bw} \\\nonumber
 &= \ \int_{\cW }    \ENT{p(\By \mid \Bx, \tilde{\Bw} )} 
 \ p(\tilde{\Bw} \mid \cD ) \ \Rd \tilde{\Bw}
\ + \ \MI{Y}{W \mid \Bx, \cD} \ .
\end{align}

We verify the last equality in Eq.~\eqref{eq:ce_h_kl}, i.e. that the Mutual Information is equal to the expected Kullback-Leibler divergence:
\begin{align}
\MI{Y}{W \mid \Bx, \cD} \ &= \
 \int_{\cW} \sum_{\By \in \cY}
 p(\By,\tilde{\Bw }\mid \Bx, \cD) \ \log \frac{p(\By,\tilde{\Bw }\mid \Bx, \cD)}
  {p(\By \mid \Bx, \cD) \ p(\tilde{\Bw }\mid \cD)}   \ \Rd \tilde{\Bw} \\ \nonumber
  &= \
 \int_{\cW} \sum_{\By \in \cY}
  p(\By \mid \Bx, \tilde{\Bw}) \   p(\tilde{\Bw }\mid \cD)  \ \log \frac{ p(\By \mid \Bx, \tilde{\Bw}) \   p(\tilde{\Bw }\mid \cD) }
  {p(\By \mid \Bx, \cD) \ p(\tilde{\Bw }\mid \cD)}   \ \Rd \tilde{\Bw} \\ \nonumber
  &= \
 \int_{\cW} \sum_{\By \in \cY}
  p(\By \mid \Bx, \tilde{\Bw})   \ \log \frac{ p(\By \mid \Bx, \tilde{\Bw})  }
  {p(\By \mid \Bx, \cD)}  \   p(\tilde{\Bw }\mid \cD) \ \Rd \tilde{\Bw} 
  \\ \nonumber
  &= \
 \int_{\cW}  \KL{p(\By \mid \Bx, \tilde{\Bw})}
  { p(\By \mid \Bx, \cD)}  \  p(\tilde{\Bw }\mid \cD) \ \Rd \tilde{\Bw} \ .
\end{align}
This is possible because the label is dependent on the selected model.
First, a model is selected, then a label is chosen with the selected model.
To summarize, the predictive uncertainty is measured by:
\begin{align} \label{eq:setting_a_all_versions}
   \ENT{p(\By \mid \Bx, \cD )}
   \ &= \ \int_{\cW }    \ENT{p(\By \mid \Bx, \tilde{\Bw} )} 
 \ p(\tilde{\Bw} \mid \cD ) \ \Rd \tilde{\Bw}
\ + \ \MI{Y}{W \mid \Bx, \cD} \\ \nonumber
&= \ \int_{\cW }    \ENT{p(\By \mid \Bx, \tilde{\Bw} )} 
 \ p(\tilde{\Bw} \mid \cD ) \ \Rd \tilde{\Bw}
\\ \nonumber
&+ \ 
\int_{\cW } \KL{p(\By \mid \Bx, \tilde{\Bw} )}{p(\By \mid \Bx, \cD )}
 \ p(\tilde{\Bw} \mid \cD ) \ \Rd \tilde{\Bw} \\\nonumber
 &= \int_{\cW } \CE{p(\By \mid \Bx, \tilde{\Bw} )}{p(\By \mid \Bx, \cD )}
 \ p(\tilde{\Bw} \mid \cD ) \ \Rd \tilde{\Bw} \ .  \nonumber
\end{align}
The total uncertainty is given by the entropy of the Bayesian model average predictive distribution, which we showed is equal to the expected
cross-entropy between the predictive distributions of candidate models $\tilde{\Bw}$ selected
according to the posterior and the Bayesian model average predictive distribution.
The aleatoric uncertainty is the expected entropy of candidate models 
drawn from the posterior, which can also be interpreted as the
entropy we expect when selecting a model according to the posterior.
Therefore, if all models likely under the posterior have low surprisal,
the aleatoric uncertainty in this setting is low.
The epistemic uncertainty is the expected KL divergence between the
the predictive distributions of candidate models and the Bayesian model average predictive distribution.
Therefore, if all models likely under the posterior have low divergence
of their predictive distribution to the Bayesian model average predictive distribution, 
the epistemic uncertainty in this setting is low.

\paragraph{Setting (b): Uncertainty of a given, pre-selected model.}

We assume to have training data $\cD$, an input $\Bx$,
and a given, pre-selected model with parameters $\Bw$ and predictive distribution $p(\By \mid \Bx, \Bw )$.
Using the predictive distribution of the model, a class $\By$ is selected based on $\Bx$,
therefore there is uncertainty about which $\By$ is selected.
Furthermore, we assume that the true model with predictive distribution $p(\By \mid \Bx, \Bw^* )$ 
and parameters $\Bw^*$ has generated the training data $\cD$ and will also
generate the observed (real world) $\By^*$ from $\Bx$ that we want to predict.
The true model is only revealed later, e.g.\ via more samples or 
by receiving knowledge about $\Bw^*$.
Hence, there is uncertainty about the parameters of the true model.
Revealing the true model is viewed as drawing a true model from
all possible true models according to their agreement with $\cD$.
Note, to reveal the true model is not necessary in our framework but helpful
for the intuition of drawing a true model.
We neither consider uncertainty about the model class nor the modeling 
nor about the training data.
In summary, there is uncertainty about drawing a class from the 
predictive distribution of the given, pre-selected model and
uncertainty about drawing the true parameters of the model distribution.

According to \citet{Apostolakis:90} and \citet{Helton:97},
the aleatoric uncertainty is the variability of
selecting a class $\By$ via $p(\By \mid \Bx, \Bw )$.
Using the entropy,
the aleatoric uncertainty is
\begin{align}
 &\ENT{p(\By \mid \Bx, \Bw )} \ .
\end{align}
Also according to \citet{Apostolakis:90} and \citet{Helton:97},
the epistemic uncertainty is the uncertainty about the
parameters $\Bw$ of the distribution, that is, a difference measure
between $\Bw$ and the true parameters $\Bw^*$.
We use as a measure for the epistemic uncertainty the Kullback-Leibler divergence:
\begin{align}
 &\KL{p(\By \mid \Bx, \Bw )}{p(\By \mid \Bx, \Bw^* )} \ .
\end{align}
The total uncertainty is the aleatoric uncertainty plus the epistemic uncertainty,
which is the cross-entropy between
$p(\By \mid \Bx, \Bw )$ and $p(\By \mid \Bx, \Bw^*)$:
\begin{align}
 \CE{p(\By \mid \Bx, \Bw )}{p(\By \mid \Bx, \Bw^* )} \ &= \ 
  \ENT{p(\By \mid \Bx, \Bw )} \ + \ 
  \KL{p(\By \mid \Bx, \Bw )}{p(\By \mid \Bx, \Bw^* )} \ .
\end{align}
However, we do not know the true parameters $\Bw^*$.
The posterior $p(\tilde{\Bw} \mid \cD )$ gives us the
likelihood of $\tilde{\Bw}$ being the true parameters $\Bw^*$.
We assume that the true model is revealed later.
Therefore we use the 
expected Kullback-Leibler divergence for the epistemic uncertainty:
\begin{align}
&\int_{\cW }    \KL{p(\By \mid \Bx, \Bw )}{p(\By \mid \Bx, \tilde{\Bw} )} 
 \ p(\tilde{\Bw} \mid \cD ) \ \Rd \tilde{\Bw}  \ .
\end{align}
Consequently, the total uncertainty is
\begin{align}
&\ENT{p(\By \mid \Bx, \Bw )} \ + \ 
 \int_{\cW }    \KL{p(\By \mid \Bx, \Bw )}{p(\By \mid \Bx, \tilde{\Bw} )} 
 \ p(\tilde{\Bw} \mid \cD ) \ \Rd \tilde{\Bw}  \ .
\end{align}
The total uncertainty can therefore be expressed by the expected cross-entropy as it was in setting (a) (see Eq.~\eqref{eq:setting_a_all_versions}), but between
$p(\By \mid \Bx, \Bw )$ and $p(\By \mid \Bx, \tilde{\Bw})$:
\begin{align}
 \int_{\cW } & \CE{p(\By \mid \Bx, \Bw )}{p(\By \mid \Bx, \tilde{\Bw} )}
 \ p(\tilde{\Bw} \mid \cD ) \ \Rd \tilde{\Bw}  \\ \nonumber
 &= \
\int_{\cW }   \left(  \ENT{p(\By \mid \Bx, \Bw )} \ + \ \KL{p(\By \mid \Bx, \Bw )}{p(\By \mid \Bx, \tilde{\Bw} )}
  \right) 
 \ p(\tilde{\Bw} \mid \cD ) \ \Rd \tilde{\Bw} \\ \nonumber
&= \  \ENT{p(\By \mid \Bx, \Bw )} \ + \ 
\int_{\cW }    \KL{p(\By \mid \Bx, \Bw )}{p(\By \mid \Bx, \tilde{\Bw} )}
 \ p(\tilde{\Bw} \mid \cD ) \ \Rd \tilde{\Bw} \ .
\end{align}

\subsubsection{Regression}
We follow \cite{Depeweg:18} and measure the predictive uncertainty in a regression setting using the differential entropy $\ENT{p(y \mid \Bx, \Bw)} = - \int_\cY p(y \mid \Bx, \Bw) \log p(y \mid \Bx, \Bw) \Rd y$ of the predictive distribution $p(y \mid \Bx, \Bw)$ of a probabilistic model.
In the following, we assume that we are modeling a Gaussian distribution, but other 
continuous probability distributions e.g. a Laplace lead to similar results.
The model thus has to provide estimators for the mean $\mu(\Bx, \Bw)$ and variance $\sigma^2(\Bx, \Bw)$ of the Gaussian.
The predictive distribution is given by 
\begin{align}
p(y \mid \Bx, \Bw ) &= (2 \pi \ \sigma^2(\Bx, \Bw))^{- \frac{1}{2}} \exp{\bigg\{-\frac{(y - \mu(\Bx, \Bw))^2}{2 \ \sigma^2(\Bx, \Bw)}\bigg\}} \ .
\end{align}
The differential entropy of a Gaussian distribution is given by
\begin{align}
     \ENT{p(y \mid \Bx, \Bw )} &= - \int_\cY \ p(y \mid \Bx, \Bw ) \log p(y \mid \Bx, \Bw ) \ \Rd y \\ \nonumber
     &= \frac{1}{2} \ \log(\sigma^2(\Bx, \Bw)) + \log(2 \pi) + \frac{1}{2} \ .
\end{align}
The KL divergence between two Gaussian distributions is given by
\begin{align}
    &\KL{p(y \mid \Bx, \Bw)}{p(y \mid \Bx, \tilde{\Bw})} \\ \nonumber
    &= - \int_\cY \ p(y \mid \Bx, \Bw ) \log \left( \frac{p(y \mid \Bx, \Bw )}{p(y \mid \Bx, \tilde{\Bw} )} \right) \ \Rd y \\ \nonumber
    &= \frac{1}{2} \ \log \left( \frac{\sigma^2(\Bx, \tilde{\Bw})}{\sigma^2(\Bx, \Bw)} \right) \ + \ \frac{\sigma^2(\Bx, \Bw) \ + \ \left( \mu(\Bx, \Bw) \ - \ \mu(\Bx, \tilde{\Bw}) \right)^2}{2 \ \sigma^2(\Bx, \tilde{\Bw})} \ - \ \frac{1}{2} \ .
\end{align}

\paragraph{Setting (a): Expected uncertainty when selecting a model.}
\cite{Depeweg:18} consider the differential entropy of the Bayesian model average $~{p(y \mid \Bx, \cD) = \int_W p(y \mid \Bx, \tilde{\Bw}) p(\tilde{\Bw} \mid \cD) \Rd \tilde{\Bw}}$, which is equal to the expected cross-entropy and can be decomposed into the expected differential entropy and Kullback-Leibler divergence.
Therefore, the expected uncertainty when selecting a model is given by
\begin{align}
    \int_{\cW } & \CE{p(y \mid \Bx, \tilde{\Bw} )}{p(y \mid \Bx, \cD )}
 \ p(\tilde{\Bw} \mid \cD ) \ \Rd \tilde{\Bw} \ = \ \ENT{p(y \mid \Bx, \cD )} \\ \nonumber
 &= \int_{\cW } \ENT{p(y \mid \Bx, \tilde{\Bw} )} \ p(\tilde{\Bw} \mid \cD ) \ \Rd \tilde{\Bw} + \int_{\cW } \KL{p(y \mid \Bx, \tilde{\Bw})}{p(y \mid \Bx, \cD)} \ p(\tilde{\Bw} \mid \cD ) \ \Rd \tilde{\Bw} \\ \nonumber
 &= \int_\cW \ \frac{1}{2} \ \log(\sigma^2(\Bx, \tilde{\Bw}))  \ p(\tilde{\Bw} \mid \cD ) \ \Rd \tilde{\Bw}  \ + \ \log(2 \pi) \\ \nonumber
 &+ \ \int_\cW \KL{p(y \mid \Bx, \tilde{\Bw})}{p(y \mid \Bx, \cD)} \ p(\tilde{\Bw} \mid \cD) \ \Rd \tilde{\Bw} \ .
\end{align}

\paragraph{Setting (b): Uncertainty of a given, pre-selected model.}
Synonymous to the classification setting, the uncertainty of a given, pre-selected model $\Bw$ is given by
\begin{align}
    \int_{\cW } & \CE{p(y \mid \Bx, \Bw )}{p(y \mid \Bx, \tilde{\Bw} )}
 \ p(\tilde{\Bw} \mid \cD ) \ \Rd \tilde{\Bw} \\ \nonumber
 &= \  \ENT{p(\By \mid \Bx, \Bw )} \ + \ 
\int_{\cW }    \KL{p(\By \mid \Bx, \Bw )}{p(\By \mid \Bx, \tilde{\Bw} )}
 \ p(\tilde{\Bw} \mid \cD ) \ \Rd \tilde{\Bw} \\ \nonumber
 &= \frac{1}{2} \ \log(\sigma^2(\Bx, \Bw)) \ + \ \log(2 \pi) \\ \nonumber
 &+ \ \int_\cW \ \frac{1}{2} \ \log \left( \frac{\sigma^2(\Bx, \tilde{\Bw})}{\sigma^2(\Bx, \Bw)} \right) \ + \ \frac{\sigma^2(\Bx, \Bw) \ + \ \left( \mu(\Bx, \Bw) \ - \ \mu(\Bx, \tilde{\Bw}) \right)^2}{2 \ \sigma^2(\Bx, \tilde{\Bw})} \ p(\tilde{\Bw} \mid \cD ) \ \Rd \tilde{\Bw} \ . %
\end{align}

\paragraph{Homoscedastic, Model Invariant Noise.}

We assume, that noise is homoscedastic for all inputs $\Bx \in \cX$, thus $\sigma^2(\Bx, \Bw) = \sigma^2(\Bw)$.
Furthermore, most models in regression do not explicitly model the variance in their training objective.
For such a model $\Bw$, we can estimate the variance on a validation dataset $\cD_{\text{val}} = \{(\Bx_n, y_n)\}|_{n=1}^N$ as
\begin{align}
    \hat{\sigma}^2(\Bw) = \frac{1}{N} \ \sum_{n=1}^N \ (y_n - \mu(\Bx_n, \Bw))^2 \ .
\end{align}
If we assume that all reasonable models under the posterior will have similar variances ($\hat{\sigma}^2(\Bw) \approx \sigma^2(\tilde{\Bw}) \, \text{for} \, \tilde{\Bw} \sim p(\tilde{\Bw} \mid \cD)$), the uncertainty of a prediction using the given, pre-selected model $\Bw$ is given by
\begin{align}
    \int_{\cW } & \CE{p(y \mid \Bx, \Bw )}{p(y \mid \Bx, \tilde{\Bw} )}
 \ p(\tilde{\Bw} \mid \cD ) \ \Rd \tilde{\Bw} \\ \nonumber
 &\approx \frac{1}{2} \ \log(\hat{\sigma}^2(\Bw)) \ + \ \log(2 \pi) \\ \nonumber
 &+ \ \int_\cW \ \frac{1}{2} \ \log \left( \frac{\hat{\sigma}^2(\Bw)}{\hat{\sigma}^2(\Bw)} \right) \ + \ \frac{\hat{\sigma}^2(\Bw) \ + \ \left( \mu(\Bx, \Bw) \ - \ \mu(\Bx, \tilde{\Bw}) \right)^2}{2 \ \hat{\sigma}^2(\Bw)} \ p(\tilde{\Bw} \mid \cD ) \ \Rd \tilde{\Bw} \\ \nonumber
 &= \frac{1}{2} \ \log(\hat{\sigma}^2(\Bw)) \ + \ \frac{1}{\hat{\sigma}^2(\Bw)} \ \int_\cW \ \left( \mu(\Bx, \Bw) \ - \ \mu(\Bx, \tilde{\Bw}) \right)^2 \ p(\tilde{\Bw} \mid \cD ) \ \Rd \tilde{\Bw} \ + \frac{1}{2} \ + \ \log(2 \pi)  \ . 
\end{align}

\subsection{Mixture Importance Sampling for Variance Reduction}\label{sec:importance_sampling_variance}

The epistemic uncertainties in Eq.~\eqref{eq:epistemic_setting_a} and Eq.~\eqref{eq:epistemic_setting_b}
are expectations of KL divergences over the posterior. 
We have to approximate these integrals.

If the posterior has different modes,
a concentrated importance sampling function
has a high variance of estimates,
therefore converges very slowly \citep{Steele:06}.
Thus, we use mixture importance 
sampling (MIS) \citep{Hesterberg:95}.
MIS uses a mixture model for sampling, 
instead of a unimodal model of standard importance sampling \citep{Owen:00}.
Multiple importance sampling \cite{Veach:95} is similar to MIS 
and equal to it for balanced heuristics \citep{Owen:00}.
More details on these and similar methods
can be found in \cite{Owen:00,Cappe:04,Elvira:15,Elvira:19,Steele:06,Raftery:10}.
MIS has been very successfully applied to estimate
multimodal densities.
For example, the evidence lower bound (ELBO) \citep{Kingma:14} has been improved
by multiple importance sampling ELBO \citep{Kviman:22}.
Using a mixture model
should ensure that at least one of its components 
will locally match the shape of the integrand.
Often, MIS
iteratively enrich the sampling distribution by new modes \citep{Raftery:10}.

In contrast to iterative enrichment, which finds modes by chance,
we are able to explicitly search for posterior modes,
where the integrand of the
definition of epistemic uncertainty is large.
For each of these modes, we define a component of the
mixture from which we then sample.
We have the huge advantage to have explicit expressions for the integrand.
The integrand of the epistemic uncertainty in
Eq.~\eqref{eq:epistemic_setting_a} and Eq.~\eqref{eq:epistemic_setting_b} has the form 
\begin{align}
   \DI{p(\By \mid \Bx, \Bw )}{p(\By \mid \Bx, \tilde{\Bw} )}
   \ p(\tilde{\Bw} \mid \cD ) \ ,
   \label{eq:distance_apdx}
\end{align}
where $\DI{\cdot}{\cdot}$ is a distance or divergence of distributions which
is computed using the parameters that determine those distributions.
The distance/divergence $\DI{\cdot}{\cdot}$ eliminates the aleatoric uncertainty, which
is present in $p(\By \mid \Bx, \Bw )$ and $p(\By \mid \Bx, \tilde{\Bw} )$.
Essentially, $\DI{\cdot}{\cdot}$ reduces distributions to functions of their parameters.

Importance sampling is applied to estimate integrals of the form
\begin{align}
    \label{eq:importance_sampling}
    s &= \int_\cX f(\Bx) \ p(\Bx) \ \Rd \Bx  = \int_\cX \frac{f(\Bx) \ p(\Bx)}{q(\Bx)} \ q(\Bx) \ \Rd \Bx \ ,
\end{align}
with integrand $f(x)$ and probability distributions $p(\Bx)$ and $q(\Bx)$, when it is easier to sample according to $q(\Bx)$ than $p(\Bx)$.
The estimator of Eq.~\eqref{eq:importance_sampling} when drawing $\Bx_n$ according to $q(\Bx)$ is given by
\begin{align}
    \hat s = \frac{1}{N} \sum_{n=1}^N \frac{f(\Bx_n) \ p(\Bx_n)}{q(\Bx_n)} \ .
\end{align}
The asymptotic variance $\sigma^2_{s}$ of importance sampling is given by (see e.g. \citet{Owen:00}):
\begin{align} \label{eq:is_variance}
    \sigma^2_{s} &= \int_\cX \left( \frac{f(\Bx) \ p(\Bx)}{q(\Bx)} \ - \ s \right)^2 \ q(\Bx) \ \Rd \Bx \\ \nonumber
    &= \int_\cX \left( \frac{f(\Bx) \ p(\Bx)}{q(\Bx)} \right)^2 \ q(\Bx) \ \Rd \Bx \ - \ s^2 \ ,
\end{align}
and its estimator when drawing $\Bx_n$ from $q(\Bx)$ is given by
\begin{align}
    \hat \sigma^2_{s} &= \frac{1}{N} \sum_{n=1}^N \left( \frac{f(\Bx_n) \ p(\Bx_n)}{q(\Bx_n)} \ - \ s \right)^2 \\ \nonumber
    &= \frac{1}{N} \sum_{n=1}^N \left( \frac{f(\Bx_n) \ p(\Bx_n)}{q(\Bx_n)} \right)^2 \ - \ s^2 \ .
\end{align}
We observe, that the variance is determined by the term $\frac{f(\Bx) p(\Bx)}{q(\Bx)}$, thus we want $q(\Bx)$ to be proportional to $f(\Bx) p(\Bx)$.
Most importantly, $q(\Bx)$ should not be close to zero for large $f(\Bx) p(\Bx)$.
To give an intuition about the severity of unmatched modes, we depict an educational example in Fig.~\ref{fig:variance_gaussians}.
Now we plug in the form of the integrand given by Eq.~\eqref{eq:distance_apdx} into Eq.~\eqref{eq:importance_sampling}, to calculate the expected divergence $\DI{\cdot}{\cdot}$ under the model posterior $p(\tilde{\Bw} \mid \cD )$.
This results in
\begin{align}
    v &= \int_\cW \frac{\DI{p(\By \mid \Bx, \Bw )}{p(\By \mid \Bx, \tilde{\Bw} )} \ p(\tilde{\Bw} \mid \cD )}{q(\tilde{\Bw})} \ q(\tilde{\Bw}) \ \Rd \tilde \Bw \ , 
\end{align}
with estimate 
\begin{align}
   \hat v &= \frac{1}{N} \ \sum_{n=1}^N \frac{\DI{p(\By \mid \Bx, \Bw )}{p(\By \mid \Bx, \tilde{\Bw}_n )}
   \ p(\tilde{\Bw}_n \mid \cD )}{q(\tilde \Bw_n)} \ .
\end{align}
The variance is given by
\begin{align} \label{eq:importance_sampling_variance}
    \sigma^2_{v} \ &= \ \int_\cW \left( \frac{\DI{p(\By \mid \Bx, \Bw )}{p(\By \mid \Bx, \tilde{\Bw})}
   \ p(\tilde{\Bw} \mid \cD )}{q(\tilde \Bw)} \ - \ v \right)^2 \
                   q(\tilde \Bw) \ \Rd \tilde \Bw \\ \nonumber
                   &= \  \int_\cW \left( \frac{\DI{p(\By \mid \Bx, \Bw )}{p(\By \mid \Bx, \tilde{\Bw})}
   \ p(\tilde{\Bw} \mid \cD )}{q(\tilde \Bw)}\right)^2 \
                   q(\tilde \Bw) \ \Rd \tilde \Bw 
                     \ - \ v^2 \ .
\end{align}
The estimate for the variance is given by
\begin{align} \label{eq:importance_sampling_variance_emp}
    \hat \sigma^2_{v} &= \frac{1}{N} \sum_{n=1}^N \left( \frac{\DI{p(\By \mid \Bx, \Bw )}{p(\By \mid \Bx, \tilde{\Bw}_n )}
   \ p(\tilde{\Bw}_n \mid \cD )}{q(\tilde \Bw_n)} \ - \ v \right)^2 \\ \nonumber
   &= \frac{1}{N} \sum_{n=1}^N \left( \frac{\DI{p(\By \mid \Bx, \Bw )}{p(\By \mid \Bx, \tilde{\Bw}_n )}
   \ p(\tilde{\Bw}_n \mid \cD )}{q(\tilde \Bw_n)} \right)^2  \ - \ v^2 \ ,
\end{align}
where $\tilde \Bw_n$ is drawn according to $q(\tilde \Bw)$.
The asymptotic ($N \to \infty$) confidence intervals are given by
\begin{align}
 &\lim_{N \to \infty} \PR \left( - \ a \ \frac{\sigma_{v}}{\sqrt{N}} 
  \ \leq \  \hat{v} \ - \ v  \ \leq \  b \ \frac{\sigma_{v}}{\sqrt{N}} \right) 
  \ = \
  \frac{1}{\sqrt{2 \ \pi}} \ \int_{-a}^b \ \exp ( - \ 1/2 \  t^2 ) \ \Rd t  \ .
\end{align}
Thus, $\hat{v}$ converges with $\frac{\sigma_{v}}{\sqrt{N}}$ to $v$.
The asymptotic confidence interval is proofed in \citet{Weinzierl:00}
and \citet{Hesterberg:96} using the Lindeberg–L\'{e}vy central limit theorem which
ensures the asymptotic normality of the estimate $\hat{v}$. 
The $q(\tilde \Bw)$ that minimizes the variance is
\begin{align}
\label{eq:AminVar}
q(\tilde \Bw) = \frac{\DI{p(\By \mid \Bx, \Bw )}{p(\By \mid \Bx, \tilde{\Bw})}
   \ p(\tilde{\Bw} \mid \cD )}{v} \ .
\end{align}
Thus we want to find a density $q(\tilde \Bw)$ that is proportional to $\DI{p(\By \mid \Bx, \Bw )}{p(\By \mid \Bx, \tilde{\Bw})}
   \ p(\tilde{\Bw} \mid \cD )$.
Only approximating the posterior $p(\tilde{\Bw} \mid \cD )$ as 
Deep Ensembles or MC dropout is insufficient to guarantee a low expected error,
since the sampling variance cannot be bounded, as $\sigma^2_v$ could get
arbitrarily big if the distance is large but the probability under the sampling distribution is very small.
For $q(\tilde \Bw) \propto  p(\tilde{\Bw} \mid \cD )$ and non-negative, unbounded,
but continuous $\DI{\cdot}{\cdot}$, the variance $ \sigma^2_{v}$ given by Eq.~\eqref{eq:importance_sampling_variance} cannot be
bounded.

For example, if $\DI{\cdot}{\cdot}$ is the KL-divergence and both $p(\By \mid \Bx, \Bw )$
and $p(\By \mid \Bx, \tilde{\Bw})$ are Gaussians where the 
means $\mu(\Bx, \Bw)$, $\mu(\Bx, \tilde{\Bw})$ 
and variances $\sigma^2(\Bx, \Bw)$, $\sigma^2(\Bx, \tilde{\Bw})$
are estimates provided by the models, the KL is unbounded.
The KL divergence between two Gaussian distributions is given by
\begin{align}
    &\KL{p(y \mid \Bx, \Bw)}{p(y \mid \Bx, \tilde{\Bw})} \\ \nonumber
    &= - \int_\cY \ p(y \mid \Bx, \Bw ) \log \left( \frac{p(y \mid \Bx, \Bw )}{p(y \mid \Bx, \tilde{\Bw} )} \right) \ \Rd y \\ \nonumber
    &= \frac{1}{2} \ \log \left( \frac{\sigma^2(\Bx, \tilde{\Bw})}{\sigma^2(\Bx, \Bw)} \right) \ + \ \frac{\sigma^2(\Bx, \Bw) \ + \ \left( \mu(\Bx, \Bw) \ - \ \mu(\Bx, \tilde{\Bw}) \right)^2}{2 \ \sigma^2(\Bx, \tilde{\Bw})} \ - \ \frac{1}{2} \ .
\end{align}
For $\sigma^2(\Bx, \tilde{\Bw})$ going towards zero and a non-zero difference of
the mean values, the KL-divergence can be arbitrarily large.
Therefore, methods that only consider the posterior $p(\tilde{\Bw}
\mid \cD )$ cannot bound the variance $ \sigma^2_{v}$ if $\DI{\cdot}{\cdot}$
is unbounded and the parameters $\tilde{\Bw}$ allow distributions
which can make $\DI{\cdot}{\cdot}$ arbitrary large.

\begin{figure}
    \centering
    \includegraphics[width=\textwidth]{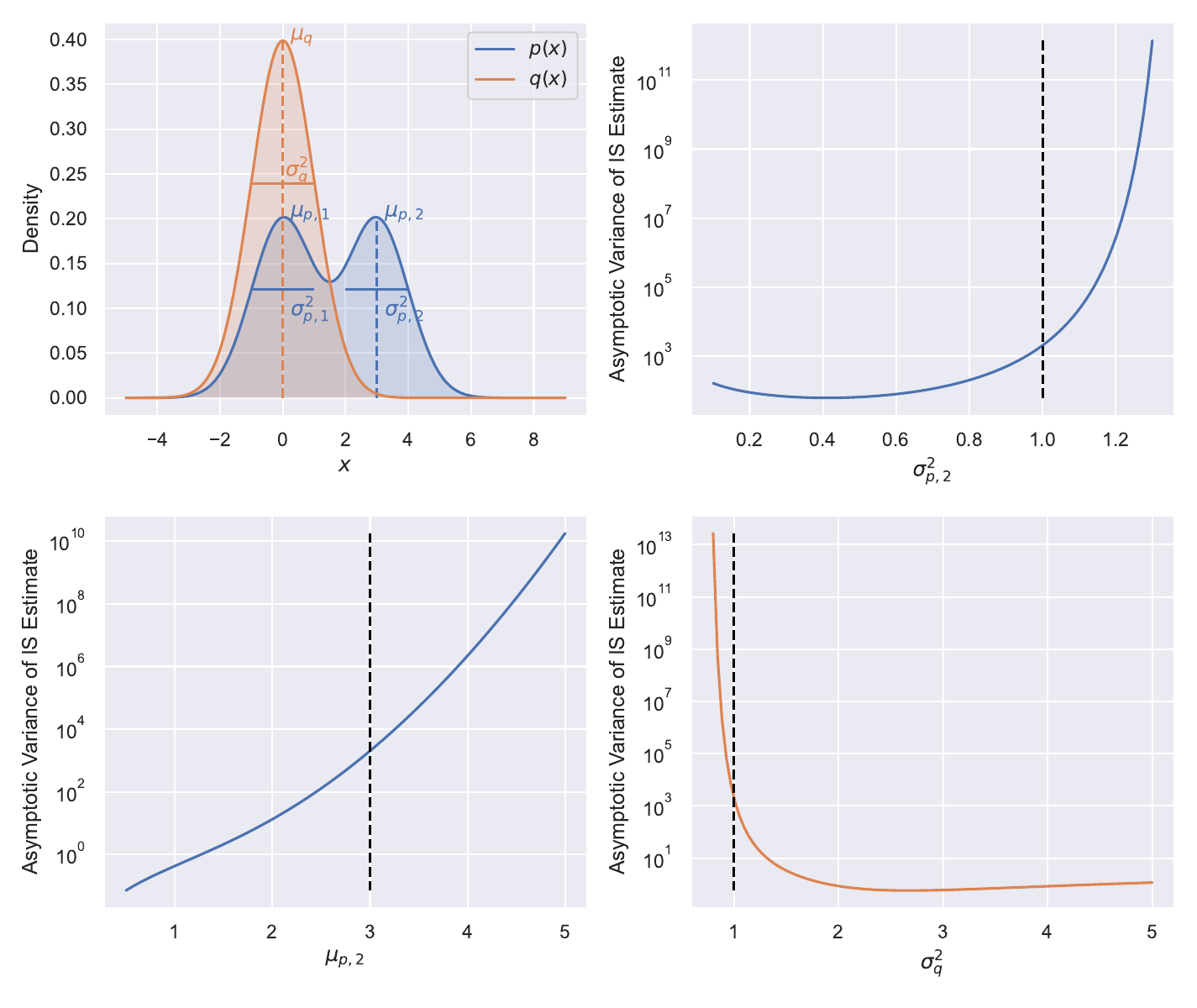}
    \caption[Asymptotic variance for multimodal target and unimodal sampling distribution]{Analysis of asymptotic variance of importance sampling for multimodal target distribution $p(x)$ and a unimodal sampling distribution $q(x)$. The target distribution is a mixture of two Gaussian distributions with means $\mu_{p, 1}, \mu_{p, 2}$ and variances $\sigma^2_{p, 1}, \sigma^2_{p, 2}$. The sampling distribution is a single Gaussian with mean $\mu_q$ and variance $\sigma^2_q$. $q(x)$ matches one of the modes of $p(x)$, but misses the other. Both distributions are visualized for their standard parameters $\mu_{p, 1} = \mu_q = 0$, $\mu_{p, 2} = 3$ and $\sigma^2_{p, 1} = \sigma^2_{p, 2} = \sigma^2_q = 1$, where both mixture components of $p(x)$ are equally weighted. We calculate the asymptotic variance (Eq.~\eqref{eq:is_variance} with $f(\Bx) = 1$) for different values of $\sigma^2_{p, 2}$, $\mu_{p, 2}$ and $\sigma^2_q$ and show the results in the top right, bottom left and bottom right plot respectively. The standard value for the varied parameter is indicated by the black dashed line. We observe, that slightly increasing the variance of the second mixture component of $p(x)$, which is not matched by the mode of $q(x)$, rapidly increases the asymptotic variance. Similarly, increasing the distance between the center of the unmatched mixture component of $p(x)$ and $q(x)$ strongly increases the asymptotic variance. On the contrary, increasing the variance of the sampling distribution $q(x)$ does not lead to a strong increase, as the worse approximation of the matched mode of $p(x)$ is counterbalanced by putting probability mass where the second mode of $p(x)$ is located. Note, that this issue is even more exacerbated if $f(x)$ is non-constant. Then, $q(x)$ has to match the modes of $f(x)$ as well.}
    \label{fig:variance_gaussians}
\end{figure}

\clearpage
\section{Experimental Details and Further Experiments} \label{apx:sec:experimental_details}

Our code is publicly available at \href{https://github.com/ml-jku/quam}{https://github.com/ml-jku/quam}.

\subsection{Details on the Adversarial Model Search} \label{apx:sec:ams_details}

\begin{figure}[b!]
\centering
\includegraphics[width=\textwidth]{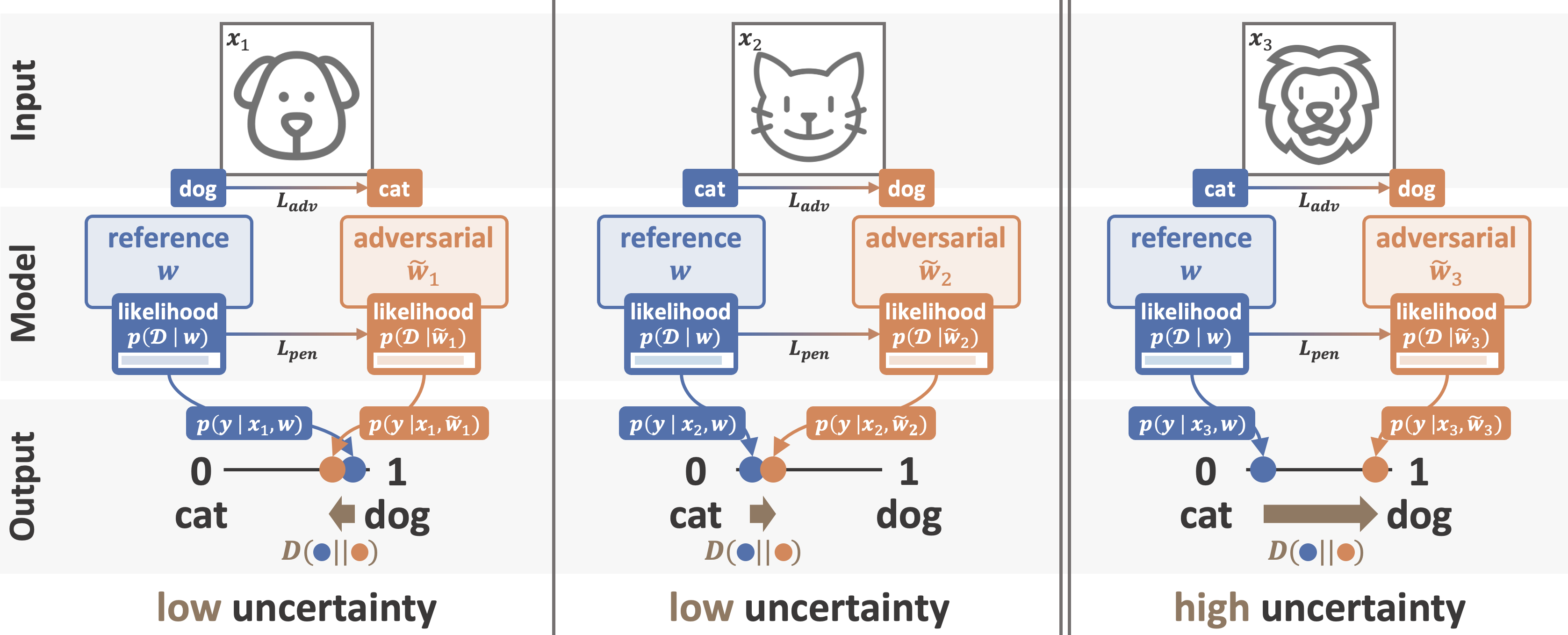}
\caption[Illustrative example of QUAM]{Illustrative example of QUAM. We illustrate quantifying the predictive uncertainty of a given, pre-selected model (blue), a classifier for images of cats and dogs. For each of the input images, we search for adversarial models (orange) that make different predictions than the given, pre-selected model while explaining the training data equally well (having a high likelihood). The adversarial models found for an image of a dog or a cat still make similar predictions (low epistemic uncertainty), while the adversarial model found for an image of a lion makes a highly different prediction (high epistemic uncertainty), as features present in images of both cats and dogs can be utilized to classify the image of a lion.
}
\label{fig:quam_overview}
\end{figure}

During the adversarial model search, we seek to maximize the KL divergence between the prediction of the reference model and adversarial models.
For an example, see Fig.~\ref{fig:quam_overview}.
We found that directly maximizing the KL divergence always leads to similar solutions to the optimization problem.
Therefore, we maximized the likelihood of a new test point to be in each possible class.
The optimization problem is very similar, considering the predictive distribution 
$p(\By \mid \Bx, \Bw)$ of a reference model and the predictive distribution $p(\By \mid \Bx, \tilde{\Bw})$ of an adversarial model, the model that is updated.
The KL divergence between those two is given by
\begin{align}
    &\KL{p(\By \mid \Bx, \Bw)}{p(\By \mid \Bx, \tilde{\Bw})} \\ \nonumber
    &= \sum p(\By \mid \Bx, \Bw) \log \left( \frac{p(\By \mid \Bx, \Bw)}{p(\By \mid \Bx, \tilde{\Bw})} \right) \\ \nonumber
    &= \sum p(\By \mid \Bx, \Bw) \log \left( p(\By \mid \Bx, \Bw)\right) - \sum p(\By \mid \Bx, \Bw) \log \left( p(\By \mid \Bx, \tilde{\Bw})\right) \\ \nonumber
    &= - \ENT{p(\By \mid \Bx, \Bw)} + \CE{p(\By \mid \Bx, \Bw)}{p(\By \mid \Bx, \tilde{\Bw})} \ .
\end{align}
Only the cross-entropy between the 
predictive distributions of the reference model parameterized by $\Bw$ and the adversarial model parameterized by $\tilde{\Bw}$
plays a role in the optimization, since the entropy of $p_{\Bw}$ stays constant during the adversarial model search.
Thus, the optimization target is equivalent to the cross-entropy loss, 
except that $\Bp_{\Bw}$ is generally not one-hot encoded but an arbitrary categorical
distribution.
This also relates to targeted / untargeted adversarial attacks on the input.
Targeted attacks try to maximize the output probability of a specific class.
Untargeted attacks try to minimize the probability of the originally predicted class,
by maximizing all other classes.
We found that attacking individual classes works better empirically, while directly maximizing the KL divergence always leads to similar solutions for different searches. The result often is a further increase of the probability associated with the most likely class.
Therefore, we conducted as many adversarial model searches for a new test point, as there are classes in the classification task. Thereby, we optimize the cross-entropy
loss for one specific class in each search.

For regression, we add a small perturbation to the bias of the output linear layer.
This is necessary to ensure a gradient in the first update step, as the model to optimize
is initialized with the reference model.
For regression, we perform the adversarial model search two times, as the output of an adversarial
model could be higher or lower than the reference model if we assume a scalar output.
We force, that the two adversarial model searches get higher or lower outputs than the
reference model respectively.
While the loss of the reference model on the training dataset $\rL_{\text{ref}}$ is
calculated on the full training dataset (as it has to be done only once),
we approximate $\rL_{\text{pen}}$ by randomly drawn mini-batches for each update step.
Therefore, the boundary condition might not be satisfied on the full training set,
even if the boundary condition is satisfied for the mini-batch estimate.

As described in the main paper, the resulting model of each adversarial model search
is used to define the location of a mixture component 
of a sampling distribution $q(\tilde{\Bw})$ (Eq.~\eqref{eq:mixture_distribution}).
The epistemic uncertainty is estimated by Eq.~\eqref{eq:epistemic_importance_sampling_estimator}, using models sampled from this
mixture distribution.
The simplest choice of distributions for each mixture distribution is a 
delta distribution at the location of the adversarial model $\breve{\Bw}_k$.
While this performs well empirically, we discard a lot of information
by not utilizing predictions of models obtained throughout the adversarial model search.
The intermediate solutions of the adversarial model search allow
to assess how easily models with highly divergent predictive distributions to the reference
model can be found.
Furthermore, the expected mean squared error (Eq.~\eqref{eq:expected_mse}) decreases
with $\frac{1}{N}$ with the number of samples $N$ and the expected 
variance of the estimator (Eq.~\eqref{eq:importance_sampling_variance_emp}) decreases
with $\frac{1}{\sqrt{N}}$.
Therefore, using more samples is beneficial empirically, even though 
we potentially introduce a bias to the estimator.

Consequently, we utilize all sampled models during the adversarial model search as an empirical sampling distribution for our experiments. This is the same as how members of an ensemble can be seen as an empirical sampling distribution \citep{Gustafsson:20} and conceptually similar to Snapshot ensembling \citep{Huang:17}.
To compute Eq.~\eqref{eq:epistemic_importance_sampling_estimator}, we use the negative exponential training loss of each model to approximate its posterior probability $p(\tilde{\Bw} \mid \cD)$. Note that the training loss is the negative log-likelihood, which in turn is proportional to the posterior probability. Note we temperature-scale the approximate posterior probability by $p(\tilde{\Bw} \mid \cD)^{\frac{1}{T}}$, with the temperature parameter $T$ set as a hyperparameter.

\subsection{Simplex Example}\label{apx:sec:simplex}

\begin{figure}[b!]
\centering
\begin{subfigure}{0.32\textwidth}
\includegraphics[width=\textwidth]{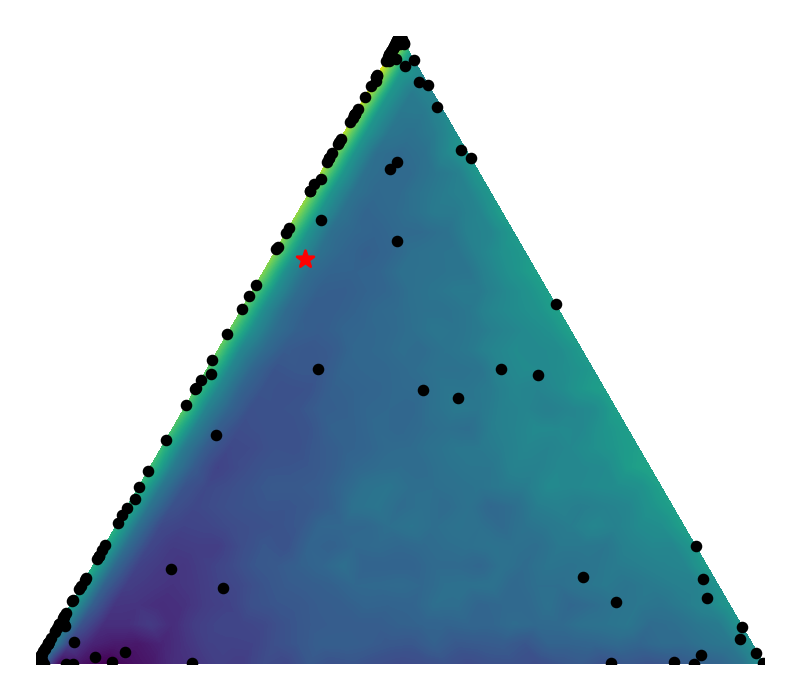}
\subcaption{HMC}
\end{subfigure}
\begin{subfigure}{0.32\textwidth}
\includegraphics[width=\textwidth]{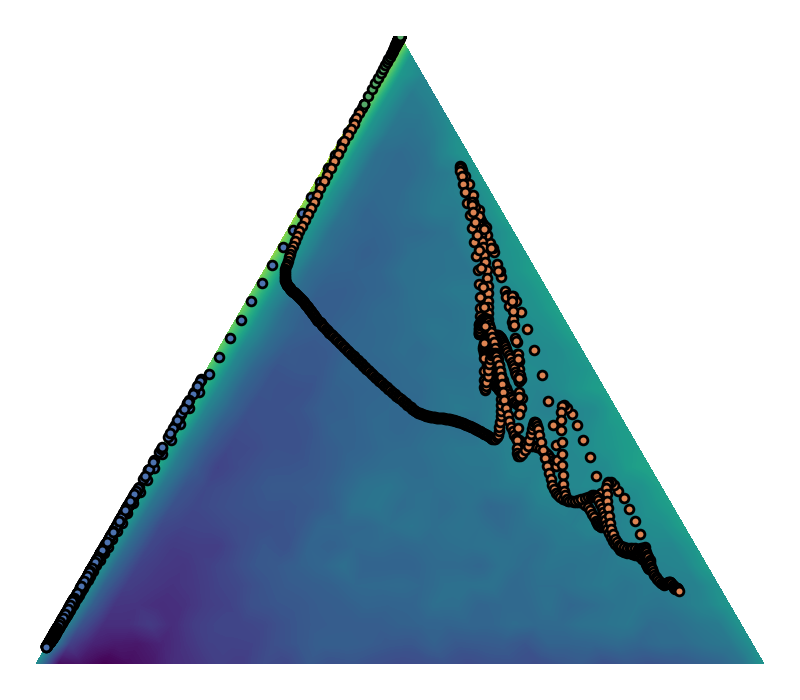}
\subcaption{Adversarial Model Search}
\end{subfigure}
\begin{subfigure}{0.32\textwidth}
\includegraphics[width=\textwidth]{figures/explanations/introexample/dataset.pdf}
\subcaption{Training data + new test point}
\end{subfigure}
\caption[HMC and Adversarial Model Search on simplex]{Softmax outputs (black) of individual models of HMC (a) as well as their average output (red) on a probability simplex.
Softmax outputs of models found throughout the adversarial model search (b), colored by the attacked class.
Left, right and top corners denote 100\% probability mass at the blue, orange and green class in (c) respectively.
Models were selected on the training data, and evaluated on the new test point (red) depicted in (c).
The background color denotes the maximum likelihood of the training data that is achievable by a model having equal softmax output as the respective location on the simplex.}
\label{apx:fig:simplex}
\end{figure}

We sample the training dataset $\cD = \{(\Bx_k, \By_k)\}_{k=1}^{K}$ from three Gaussian distributions (21 datapoints from each Gaussian) at locations $\Bmu_1 = (-4, -2)^T$, $\Bmu_2 = (4, -2)^T$, $\Bmu_3 = (0, 2 \ \sqrt{2})^T$ and the same two-dimensional covariance with $\sigma^2 = 1.5$ on both entries of the diagonal and zero on the off-diagonals.
The labels $\By_k$ are one-hot encoded vectors, signifying which Gaussian the input $\Bx_k$ was sampled from.
The new test point $\Bx$ we evaluate for is located at $(-6, 2)$.
To attain the likelihood for each position on the probability simplex, we train a two-layer fully connected neural network (with parameters $\Bw$) 
with hidden size of 10 on this dataset.
We minimize the combined loss 
\begin{align}
    \rL = \frac{1}{K} \sum_{k=1}^K l(p(\By \mid \Bx_k, \Bw), \By_k) + l(p(\By \mid \Bx, \Bw), \breve{\By}) \ ,
\end{align}
where $l$ is the cross-entropy loss function and $\breve{\By}$ 
is the desired categorical distribution for the output of the network.
We report the likelihood on the training dataset upon convergence of the training
procedure for $\breve{\By}$ on the probability simplex.
To average over different initializations of $\Bw$ and alleviate the influence of 
potentially bad local minima, 
we use the median over 20 independent runs to calculate the maximum.

For all methods, we utilize the same two-layer fully connected neural network with hidden
size of 10; for MC dropout we additionally added dropout with dropout probability $0.2$ after
every intermediate layer.
We trained 50 networks for the Deep Ensemble results.
For MC dropout we sampled predictive distributions using 1000 forward passes.

Fig.~\ref{apx:fig:simplex} (a) shows models sampled using HMC, 
which is widely regarded as the best approximation to the 
ground truth for predictive uncertainty estimation.
Furthermore, Fig.~\ref{apx:fig:simplex} (b) shows models obtained by executing the
adversarial model search for the given training dataset and test point depicted in
Fig.~\ref{apx:fig:simplex} (c).
HMC also provides models that put more probability mass on the orange class.
Those are missed by Deep Ensembles and MC dropout (see Fig.~\ref{fig:simplex} (a) and  (b)).
The adversarial model search used by QUAM helps to identify those regions.

\subsection{Epistemic Uncertainty on Synthetic Dataset}\label{apx:sec:synthetic}

We create the two-moons dataset using the implementation of \cite{Pedregosa:11}.
All experiments were performed on a three-layer fully connected neural network with hidden size 100 and ReLU activations.
For MC dropout, dropout with dropout probability of 0.2 was applied after the intermediate layers.
We assume to have a trained reference model $\Bw$ of this architecture.
Results of the same runs as in the main paper, but calculated for the epistemic uncertainty in setting (b) (see Eq.~\eqref{eq:epistemic_setting_b}) are depicted in Fig.~\ref{fig:res:classification_twomoon_setting_b}.
Again, QUAM matches the ground truth best.

Furthermore, we conducted experiments on a synthetic regression dataset, where the input feature $x$ is drawn randomly between $ \left[-\pi , \pi \right]$ and the target is $y = \sin(x) + \epsilon$, with $\epsilon \sim \cN(0, 0.1)$.
The results are depicted in Fig.~\ref{fig:res:regression}.
As for the classification results, the estimate of QUAM is closest to the ground truth
provided by HMC.

The HMC implementation of \cite{Cobb:21} was used to obtain 
the ground truth epistemic uncertainties.
For the Laplace approximation, we used the implementation of \cite{Daxberger:21}.
For SG-MCMC we used the python package of \cite{Kapoor:23}.

\begin{figure}
\captionsetup[subfigure]{aboveskip=-0.5pt,belowskip=-1pt}
\centering
\begin{subfigure}{0.32\textwidth}
\includegraphics[width=\textwidth]{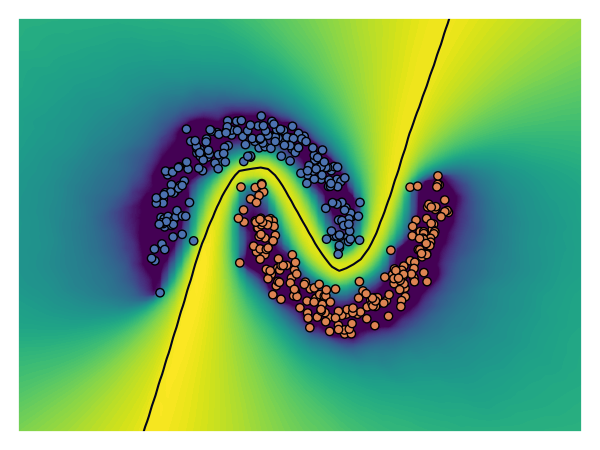}
\subcaption{\textbf{Ground Truth} - HMC}
\end{subfigure}
\begin{subfigure}{0.32\textwidth}
\includegraphics[width=\textwidth]{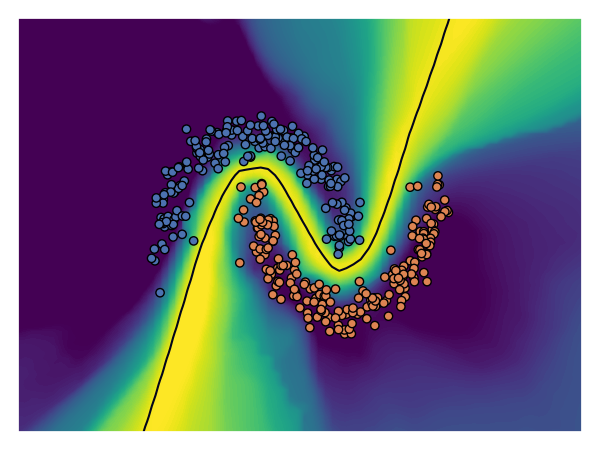}
\subcaption{cSG-HMC}
\end{subfigure}
\begin{subfigure}{0.32\textwidth}
\includegraphics[width=\textwidth]{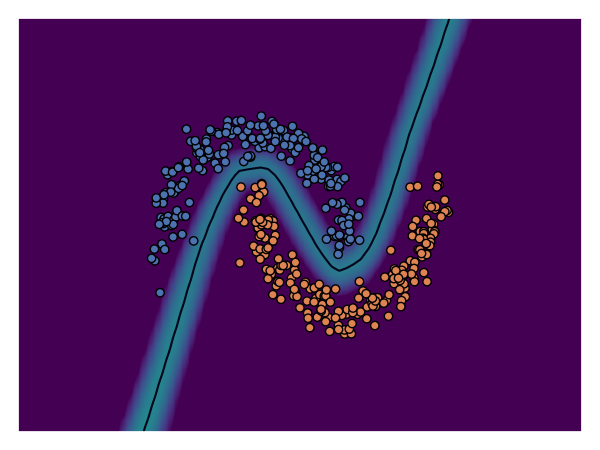}
\subcaption{Laplace}
\end{subfigure}
\begin{subfigure}{0.32\textwidth}
\includegraphics[width=\textwidth]{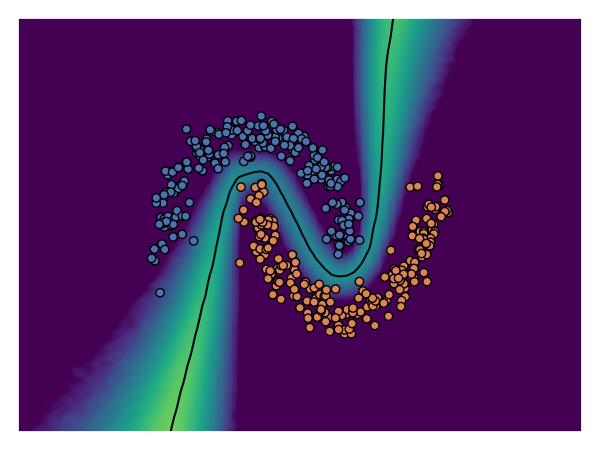}
\subcaption{MC dropout}
\end{subfigure}
\begin{subfigure}{0.32\textwidth}
\includegraphics[width=\textwidth]{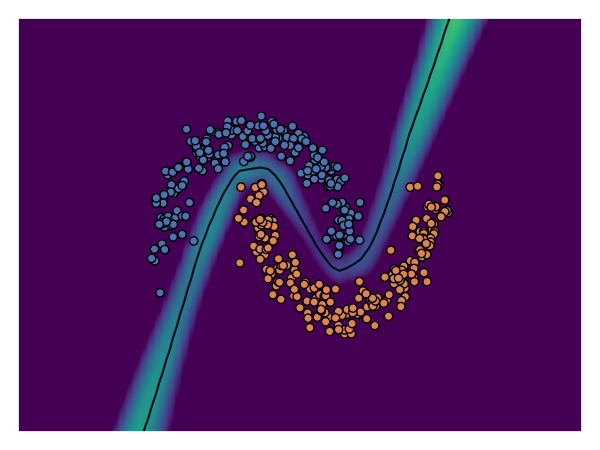}
\subcaption{Deep Ensembles}
\end{subfigure}
\begin{subfigure}{0.32\textwidth}
\includegraphics[width=\textwidth]{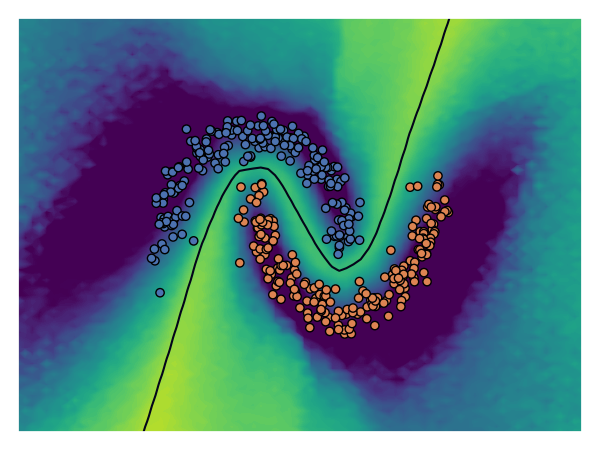}
\subcaption{\textbf{Our Method} - QUAM}
\end{subfigure}
\caption[Epistemic uncertainty (setting (b)) on synthetic classification dataset]{Epistemic uncertainty as in Eq.~\eqref{eq:epistemic_setting_b}. Yellow denotes high epistemic uncertainty. Purple denotes low epistemic uncertainty. The black lines show the decision boundary of the reference model $\Bw$. HMC is considered to be  the ground truth epistemic uncertainty. The estimate of QUAM is closest to the ground truth. All other methods underestimate the epistemic uncertainty in the top left and bottom right corner, as all models sampled by those predict the same class with high confidence for those regions.

}
\label{fig:res:classification_twomoon_setting_b}

\bigskip

\begin{subfigure}{0.32\textwidth}
\includegraphics[width=\textwidth]{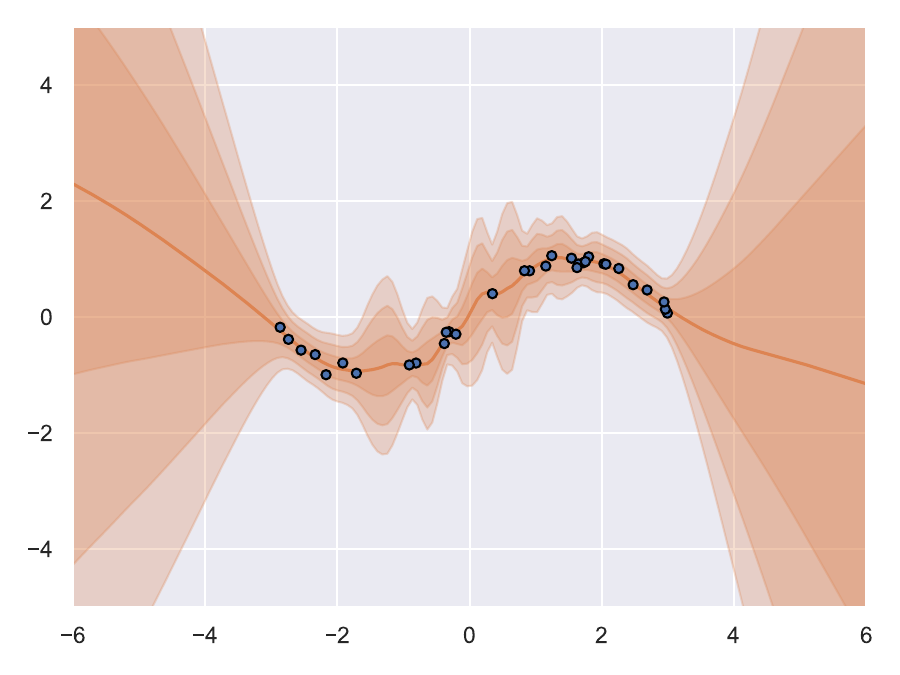}
\subcaption{\textbf{Ground Truth} - HMC}
\end{subfigure}
\begin{subfigure}{0.32\textwidth}
\includegraphics[width=\textwidth]{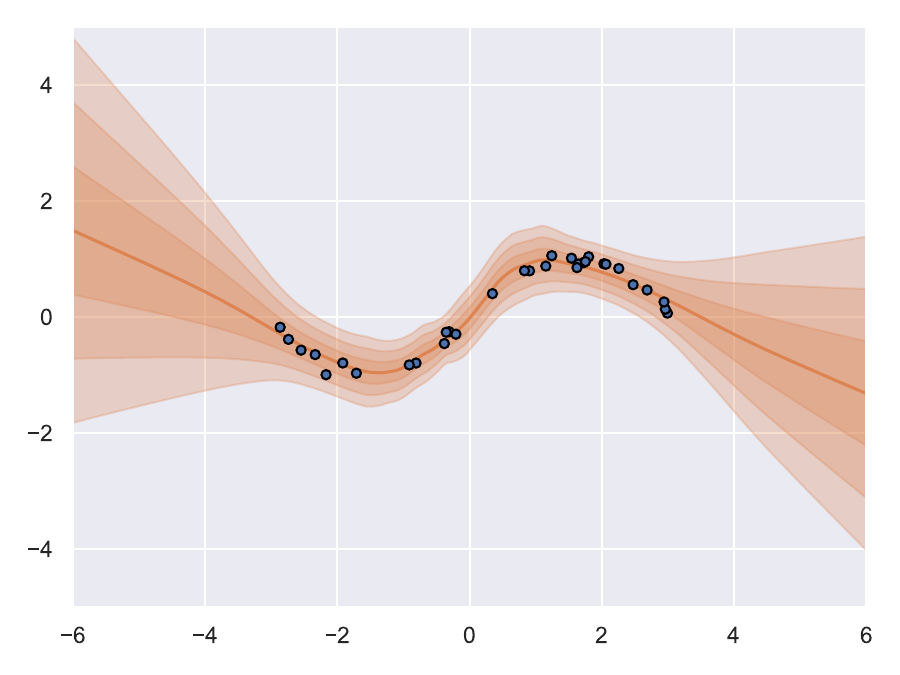}
\subcaption{cSG-HMC}
\end{subfigure}
\begin{subfigure}{0.32\textwidth}
\includegraphics[width=\textwidth]{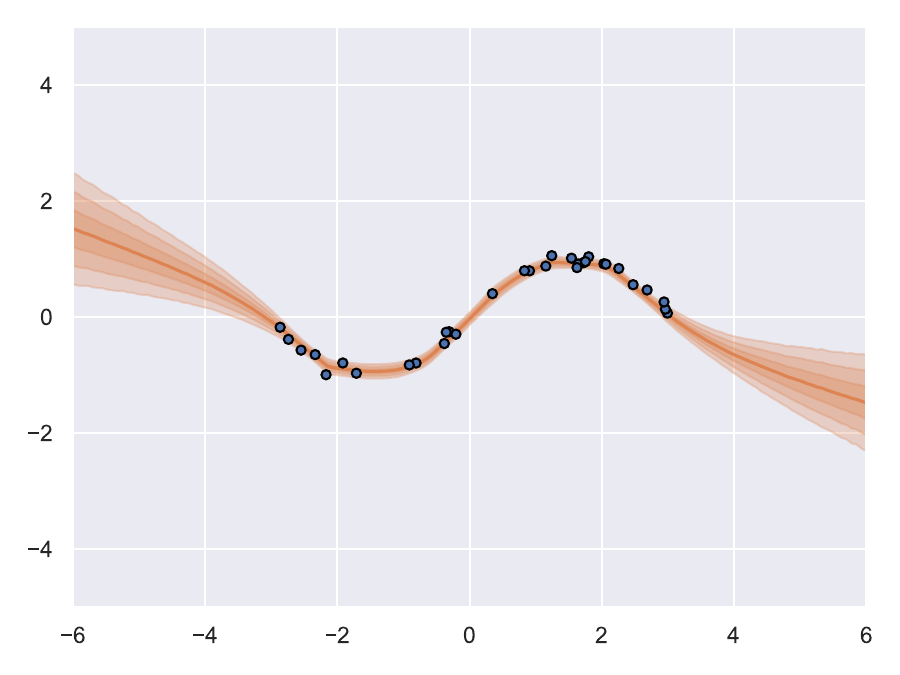}
\subcaption{Laplace}
\end{subfigure}
\begin{subfigure}{0.32\textwidth}
\includegraphics[width=\textwidth]{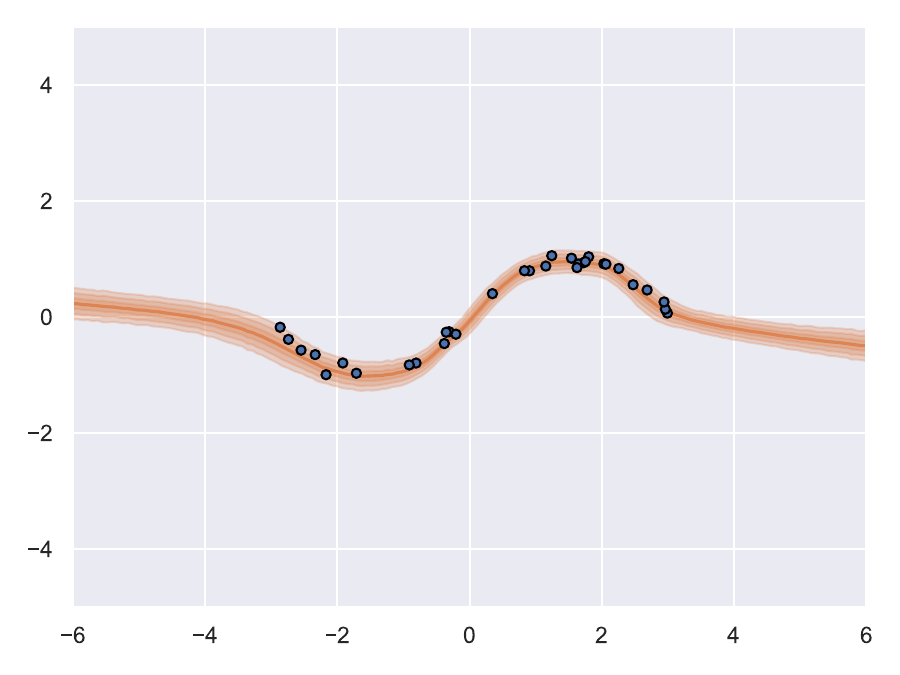}
\subcaption{MC dropout}
\end{subfigure}
\begin{subfigure}{0.32\textwidth}
\includegraphics[width=\textwidth]{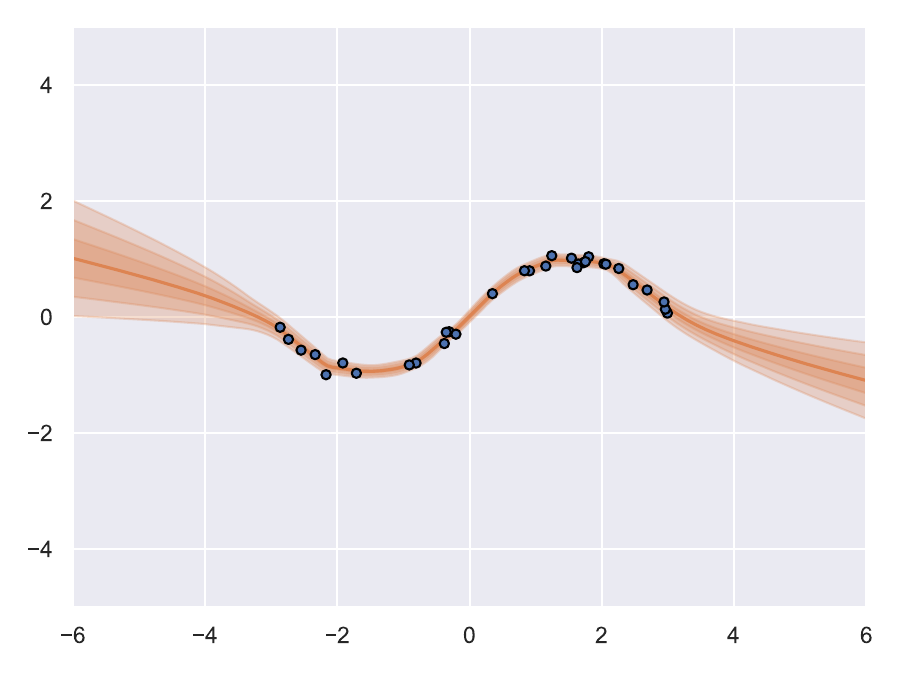}
\subcaption{Deep Ensembles}
\end{subfigure}
\begin{subfigure}{0.32\textwidth}
\includegraphics[width=\textwidth]{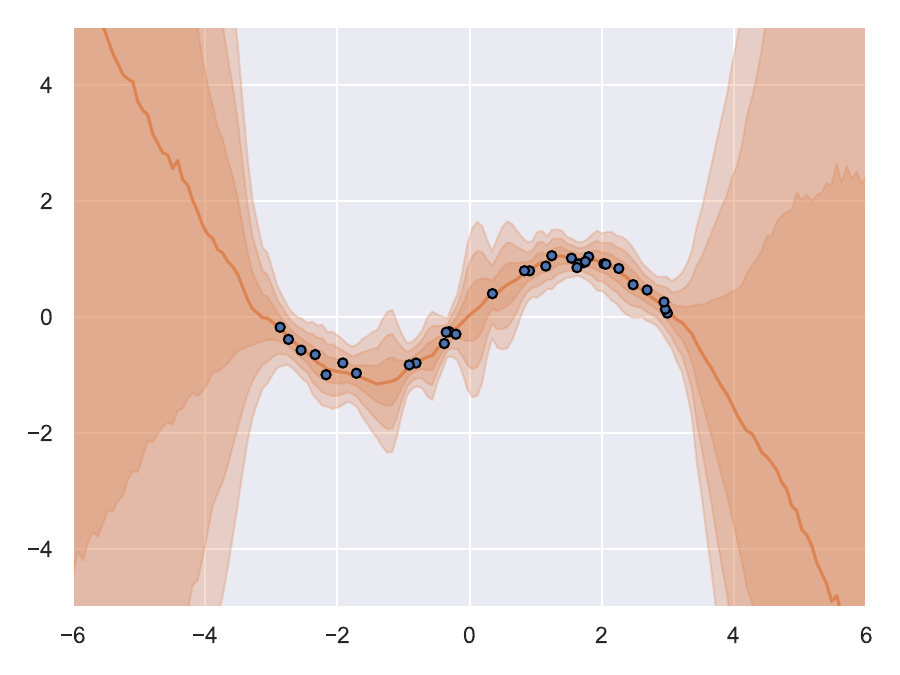}
\subcaption{\textbf{Our Method} - QUAM}
\end{subfigure}
\caption[Model variance of different methods on toy regression dataset]{Variance between different models found by different methods on synthetic \texttt{sine} dataset. Orange line denotes the empirical mean of the averaged models, shades denote one, two and three standard deviations respectively.  HMC is considered to be  the ground truth epistemic uncertainty. The estimate of QUAM is closest to the ground truth. All other methods fail to capture the variance between points as well as the variance left outside the region ($[-\pi , \pi]$) datapoints are sampled from.}
\label{fig:res:regression}
\end{figure}

\subsection{Epistemic Uncertainty on Vision Datasets}\label{sec:apx:vision_experiments}

Several vision datasets and their corresponding OOD datasets are commonly used for benchmarking predictive uncertainty quantification in the literature, e.g.\ in \cite{Blundell:15, Gal:16, Malinin:18, Ovadia:19, vAmersfoort:20, Mukhoti:21, Postels:21, Band:22}. 
Our experiments focused on two of those: MNIST \citep{LeCun:98} and its OOD derivatives as the most basic benchmark and ImageNet1K \citep{Deng:09} to demonstrate our method's ability to perform on a larger scale. 
Four types of experiments were performed: (i) OOD detection (ii) adversarial example detection, (iii) misclassification detection and (iv) selective prediction. 
Our experiments on adversarial example detection did not 
utilize a specific adversarial attack on the
input images, but natural adversarial examples \citep{Hendrycks:21},
which are images from the ID classes, 
but wrongly classified by standard ImageNet classifiers.
Misclassification detection and selective prediction was only performed for Imagenet1K, 
since MNIST classifiers easily reach accuracies of 99\% on the test set,
thus hardly misclassifying any samples. 
In all cases except selective prediciton, we measured AUROC, FPR at TPR of 95\% and AUPR of classifying ID vs. OOD, 
non-adversarial vs. adversarial and 
correctly classified vs. misclassified samples (on ID test set), 
using the epistemic uncertainty estimate provided by the different methods. 
For selective prediction, we utilized the epistemic uncertainty estimate to select a subset of samples on the ID test set.

\subsubsection{MNIST}\label{apx:sec:mnist}

OOD detection experiments were performed on MNIST with FashionMNIST (FMNIST) \citep{Xiao:17}, EMNIST \citep{Cohen:17}, KMNIST \citep{Clanuwat:18} and OMNIGLOT \citep{Lake:15} as OOD datasets. 
In case of EMNIST, we only used the ''letters'' subset, thus excluding classes overlapping with MNIST (digits).
We used the MNIST (test set) vs FMNIST (train set) OOD detection task to tune 
hyperparameters for all methods. 
The evaluation was performed using the complete test sets of the above-mentioned datasets ($n=10000$). 

For each seed, a separate set of Deep Ensembles was trained. 
Ensembles with the size of 10 were found to perform best.
MC dropout was used with a number of samples set to 2048. 
This hyperparameter setting was found to perform well. 
A higher sampling size would increase the performance marginally while increasing the computational load. 
Noteworthy is the fact, that with these settings the 
computational requirements of MC dropout surpassed those of QUAM.
Laplace approximation was performed only for the last layer,
due to the computational demand making it infeasible on the full network
with our computational capacities.
Mixture of Laplace approximations \cite{Eschenhagen:21} was evaluated as well using the parameters provided in the original work. Notably, the results from the original work suggesting improved performance compared to the Deep Ensembles on these tasks could not be reproduced. Comparison is provided in Table \ref{tab:res:mola_comparison}.
SG-HMC was performed on the full network using the Python package from \cite{Kapoor:23}. 
Parameters were set in accordance with those of the original authors \citep{Zhang:20}. 
For QUAM, the initial penalty parameter found by tuning was $c_0 = 6$, 
which was exponentially increased ($c_{t+1} = \eta c_{t}$) with $\eta = 2$ 
every $14$ gradient steps for a total of two epochs through the training dataset.
Gradient steps were performed using Adam \citep{Kingma:14b} with a learning rate of 
$5.e\text{-}3$ and weight decay of $1.e\text{-}3$, 
chosen equivalent to the original training parameters of the model. 
A temperature of $1.e\text{-}3$ was used for scaling the cross-entropy loss, an approximation for the posterior probabilities when calculating Eq.~\eqref{eq:epistemic_importance_sampling_estimator}.
Detailed results and additional metrics and replicates of the experiments can be found in Tab.~\ref{tab:appendix:detailed_mnist}. 
Experiments were performed three times with seeds: \{42, 142, 242\} to provide confidence intervals.
Histograms of the scores on the ID dataset and the OOD datasets for different methods are depicted in Fig.~\ref{fig:mnist_histogram}.

\setlength{\tabcolsep}{17pt}
\renewcommand{\arraystretch}{1.2}
\begin{table}[t!]
\centering
\caption[Results for additional baseline (MoLA)]{Additional baseline MoLA: AUROC using the epistemic uncertainty of a given, pre-selected model as a score to distinguish between ID (MNIST) and OOD samples. Results for additional baseline method MoLA, comparing to Laplace approximation, Deep Ensembles (DE) and QUAM. Results are averaged over three independent runs.\\
}
\label{tab:res:mola_comparison}
\begin{tabular}{ccccccc}
\hline
$\cD_{\text{ood}}$ & Laplace             & {MoLA}         & DE         & QUAM                \\ \hline
FMNIST              &               $.978_{\pm .004}$      & {$.986_{\pm .002}$} & $.988_{\pm .001}$ & $\boldsymbol{.994}_{\pm .001}$ \\
KMNIST              &                   $.959_{\pm .006}$  & {$.984_{\pm .000}$} & $.990_{\pm .001}$ & $\boldsymbol{.994}_{\pm .001}$ \\
EMNIST              & $.877_{\pm .011}$ & {$.920_{\pm .002}$} & $.924_{\pm .003}$ & $\boldsymbol{.937}_{\pm .008}$ \\
OMNIGLOT            & $.963_{\pm .003}$ &           {$.979_{\pm .000}$}          & $.983_{\pm .001}$ &  $\boldsymbol{.992}_{\pm .001}$ \\ \hline
\end{tabular}
\end{table}
\renewcommand{\arraystretch}{1}

\renewcommand{\arraystretch}{0.85}
\setlength{\tabcolsep}{9pt}
\begin{table}
\centering
\caption[Detailed results of MNIST OOD detection experiments]{Detailed results of MNIST OOD detection experiments, reporting AUROC, AUPR and FPR@TPR=95\% for individual seeds.\\
}
\label{tab:appendix:detailed_mnist}
\begin{tabular}{ccccccc}
\hline \\ [-6pt]
OOD dataset &  Method & Seed & $\uparrow$ AUPR  & $\uparrow$ AUROC & $\downarrow$ FPR@TPR=95\% \\[2pt] \hline  \\ [-6pt]
\multirow{15}{*}{EMNIST}&\multirow{3}{*}{cSG-HMC}&42&0.8859&0.8823&0.5449 \\ 
 & &142&0.8714&0.8568&0.8543 \\ 
 & &242&0.8797&0.8673&0.7293 \\ 
 &\multirow{3}{*}{Laplace}&42&0.8901&0.8861&0.5273 \\ 
 & &142&0.8762&0.8642&0.7062 \\ 
 & &242&0.8903&0.8794&0.6812 \\ 
 &\multirow{3}{*}{Deep Ensembles}&42&0.9344&0.9239&0.4604 \\ 
 & &142&0.9325&0.9236&0.4581 \\ 
 & &242&0.9354&0.9267&0.4239 \\ 
 &\multirow{3}{*}{MC dropout}&42&0.8854&0.8787&0.5636 \\ 
 & &142&0.8769&0.8630&0.6718 \\ 
 & &242&0.8881&0.8751&0.6855 \\ 
 &\multirow{3}{*}{QUAM}&42&0.9519&0.9454&0.3405 \\ 
 & &142&0.9449&0.9327&0.4538 \\ 
 & &242&0.9437&0.9317&0.4325 \\ 
\hline \\ [-6pt]
\multirow{15}{*}{FMNIST}&\multirow{3}{*}{cSG-HMC}&42&0.9532&0.9759&0.0654 \\ 
 & &142&0.9610&0.9731&0.0893 \\ 
 & &242&0.9635&0.9827&0.0463 \\ 
 &\multirow{3}{*}{Laplace}&42&0.9524&0.9754&0.0679 \\ 
 & &142&0.9565&0.9739&0.0788 \\ 
 & &242&0.9613&0.9824&0.0410 \\ 
 &\multirow{3}{*}{Deep Ensembles}&42&0.9846&0.9894&0.0319 \\ 
 & &142&0.9776&0.9865&0.0325 \\ 
 & &242&0.9815&0.9881&0.0338 \\ 
 &\multirow{3}{*}{MC dropout}&42&0.9595&0.9776&0.0644 \\ 
 & &142&0.9641&0.9748&0.0809 \\ 
 & &242&0.9696&0.9848&0.0393 \\ 
 &\multirow{3}{*}{QUAM}&42&0.9896&0.9932&0.0188 \\ 
 & &142&0.9909&0.9937&0.0210 \\ 
 & &242&0.9925&0.9952&0.0132 \\ 
\hline \\ [-6pt]
\multirow{15}{*}{KMNIST}&\multirow{3}{*}{cSG-HMC}&42&0.9412&0.9501&0.2092 \\ 
 & &142&0.9489&0.9591&0.1551 \\ 
 & &242&0.9505&0.9613&0.1390 \\ 
 &\multirow{3}{*}{Laplace}&42&0.9420&0.9520&0.1915 \\ 
 & &142&0.9485&0.9617&0.1378 \\ 
 & &242&0.9526&0.9640&0.1165 \\ 
 &\multirow{3}{*}{Deep Ensembles}&42&0.9885&0.9899&0.0417 \\ 
 & &142&0.9875&0.9891&0.0458 \\ 
 & &242&0.9884&0.9896&0.0473 \\ 
 &\multirow{3}{*}{MC dropout}&42&0.9424&0.9506&0.2109 \\ 
 & &142&0.9531&0.9618&0.1494 \\ 
 & &242&0.9565&0.9651&0.1293 \\ 
 &\multirow{3}{*}{QUAM}&42&0.9928&0.9932&0.0250 \\ 
 & &142&0.9945&0.9952&0.0194 \\ 
 & &242&0.9925&0.9932&0.0260 \\ 
\hline \\ [-6pt]
\multirow{15}{*}{OMNIGLOT}&\multirow{3}{*}{cSG-HMC}&42&0.9499&0.9658&0.1242 \\ 
 & &142&0.9459&0.9591&0.1498 \\ 
 & &242&0.9511&0.9637&0.1222 \\ 
 &\multirow{3}{*}{Laplace}&42&0.9485&0.9647&0.1238 \\ 
 & &142&0.9451&0.9597&0.1345 \\ 
 & &242&0.9526&0.9656&0.1077 \\ 
 &\multirow{3}{*}{Deep Ensembles}&42&0.9771&0.9822&0.0621 \\ 
 & &142&0.9765&0.9821&0.0659 \\ 
 & &242&0.9797&0.9840&0.0581 \\ 
 &\multirow{3}{*}{MC dropout}&42&0.9534&0.9663&0.1248 \\ 
 & &142&0.9520&0.9619&0.1322 \\ 
 & &242&0.9574&0.9677&0.1063 \\ 
 &\multirow{3}{*}{QUAM}&42&0.9920&0.9930&0.0274 \\ 
 & &142&0.9900&0.9909&0.0348 \\ 
 & &242&0.9906&0.9915&0.0306 \\ 
\hline
\end{tabular}
\end{table}
\renewcommand{\arraystretch}{1}

\begin{figure}[b!]
\centering
\includegraphics[width=\textwidth]{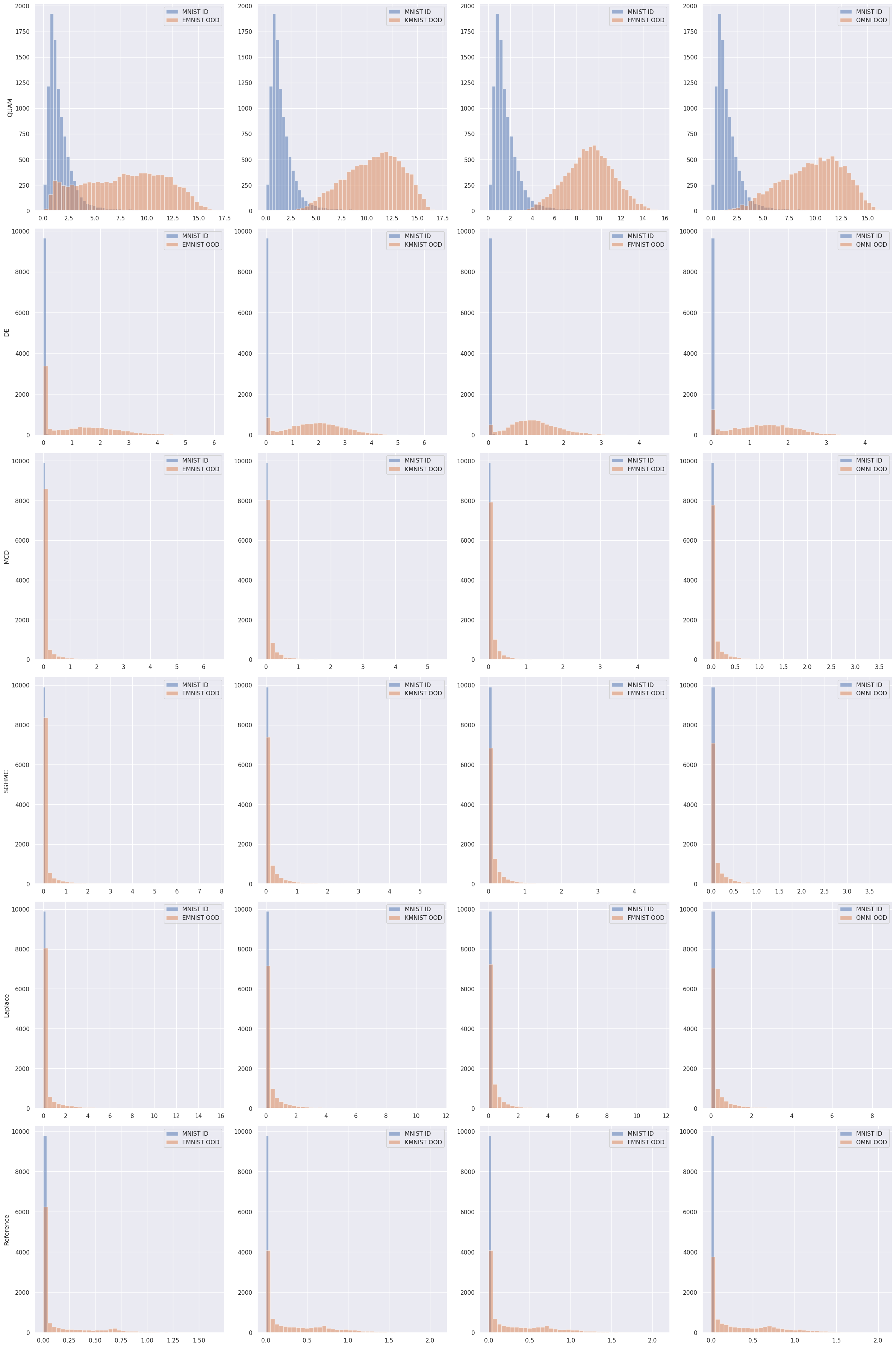}
\caption[Histograms MNIST]{MNIST: Histograms of uncertainty scores calculated for test set samples of the specified datasets.}
\label{fig:mnist_histogram}
\end{figure}

\subsubsection{ImageNet} \label{apx:sec:imagenet}

For ImageNet1K \citep{Deng:09}, 
OOD detection experiments were performed with ImageNet-O \citep{Hendrycks:21}, 
adversarial example detection experiments with ImageNet-A \citep{Hendrycks:21}, 
and misclassification detection as well as selective prediction experiments on the official validation set of ImageNet1K. 
For each experiment, we utilized a pre-trained EfficientNet \citep{Tan:19} architecture with 21.5 million 
trainable weights available through PyTorch \citep{Paszke:19}, achieving a top-1 accuracy of $84.2\%$ as well as a top-5 accuracy of $96.9\%$. 

cSG-HMC was performed on the last layer using the best hyperparameters that resulted from a hyperparameter search around the ones suggested by the original authors \citep{Zhang:20}. 
The Laplace approximation with the implementation of \citep{Daxberger:21} was not feasible to compute for this problem on our hardware, 
even only for the last layer.
Similarly to the experiments in section \ref{apx:sec:mnist}, we compare against a Deep Ensemble consisting of 10 pre-trained EfficientNet architectures ranging from 5.3 million to 66.3 million trainable weights (DE (all)). 
Also, we retrained the last layer of 10 ensemble members (DE (LL)) given the same base network. 
We also compare against MC dropout used with 2048 samples with a dropout probability of $20\%$.
The EfficientNet architectures utilize dropout only before the last layer.
The adversarial model search for QUAM was performed on the last layer of the EfficientNet, which has 1.3 million trainable parameters. 
To enhance the computational efficiency, the output of the second-to-last layer was
computed once for all samples, and this output was subsequently used as input for 
the final layer when performing the adversarial model search. 
We fixed $c_{0}$ to $1$ and exponentially updated it at every of the 256 update steps. 
Also, weight decay was fixed to $1.e\text{-}4$
for the Adam optimizer \citep{Kingma:14b}.

Two hyperparameters have jointly been optimized on ImageNet-O and ImageNet-A 
using a small grid search, with learning rate $\alpha \in \{5.e\text{-}3, 1.e\text{-}3, 5.e\text{-}4, 1.e\text{-}4\}$ 
and the exponential schedule update constant $\eta \in \{1.15, 1.01, 1.005, 1.001\}$. 
The hyperparameters $\alpha = 1.e\text{-}3$ and $\eta = 1.01$ resulted in the overall highest performance 
and have thus jointly been used for each of the three experiments. 
This implies that $c_{0}$ increases by $1\%$ after each update step. 
We additionally searched for the best temperature and the best number of update steps 
for each experiment separately.  
The best temperature for scaling the cross-entropy loss when calculating Eq.~\eqref{eq:epistemic_importance_sampling_estimator} was identified as $0.05$, $0.005$, and $0.0005$, 
while the best number of update steps was identified as $50$, $100$, and $100$ for 
ImageNet-O OOD detection, ImageNet-A adversarial example detection, 
and ImageNet1K misclassification detection, respectively. Selective prediction was performed using the same hyperparameters as misclassification detection. We observed that the adversarial model search is relatively stable 
with respect to these hyperparameters.

The detailed results on various metrics and replicates of the experiments can be found in \ref{tab:appendix:detailed_imagenet}. 
Histograms of the scores on the ID dataset and the OOD dataset, the adversarial example dataset and the correctly and incorrectly classified samples are depicted in Fig.~\ref{fig:imagenet_histogram} for all methods.
ROC curves, as well as accuracy over retained sample curves, are depicted in Fig.~\ref{fig:roc}.
To provide confidence intervals, we performed all experiments on three distinct dataset splits of the ID datasets, matching the number of OOD samples.
Therefore we used three times 2000 ID samples for Imagenet-O and three times 7000 ID samples for Imagenet-A and misclassification detection as well as selective prediction.

\paragraph{Calibration.} Additionally, we analyze the calibration of QUAM compared to other baseline methods. Therefore, we compute the expected calibration error (ECE) \citep{Guo:17} on the ImageNet-1K validation dataset using the expected predictive distribution. Regarding QUAM, the predictive distribution was optained using the same hyperparameters as for misclassification detection reported above.
We find that QUAM improves upon the other considered baseline methods, although it was not directly designed to improve the calibration of the predictive distribution. Tab.~\ref{tab:calibration} states the ECE of considered uncertainty quantification methods and in Fig.~\ref{fig:calibration} the accuracy and number of samples (depicted by the size) for specific confidence bins is depicted.

\setlength{\tabcolsep}{19pt} %
\renewcommand{\arraystretch}{1.2}
\begin{table}[t!]
\centering
\caption[Results for calibration on ImageNet]{Calibration: expected calibration error (ECE) based on the weighted average predictive distribution. Reference refers to the predictive distribution of the given, pre-selected model. Experiment was performed on three distinct splits, each containing 7000 ImageNet-1K validation samples.\\}
\begin{tabular}{cccccc}
\hline
Reference & cSG-HMC & MCD & DE & QUAM \\ \hline
$.159_{\pm .004}$ & $.364_{\pm .001}$ & $.166_{\pm .004}$  & $.194_{\pm .004}$ & $\boldsymbol{.096}_{\pm .006}$ \\
\hline
\end{tabular}
\label{tab:calibration}
\end{table}
\renewcommand{\arraystretch}{1}

\begin{figure*}
\centering
\begin{subfigure}{0.19\textwidth}
\includegraphics[width=\textwidth, trim={12pt 12pt 12pt 12pt}, clip]{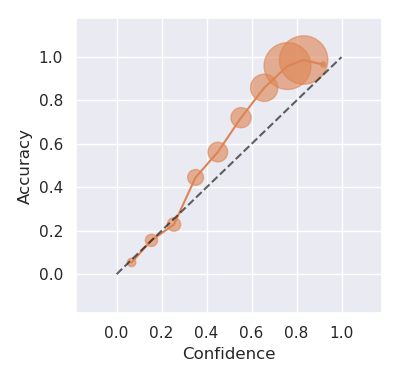}
\subcaption{Reference}
\end{subfigure}
\begin{subfigure}{0.19\textwidth}
\includegraphics[width=\textwidth, trim={12pt 12pt 12pt 12pt}, clip]{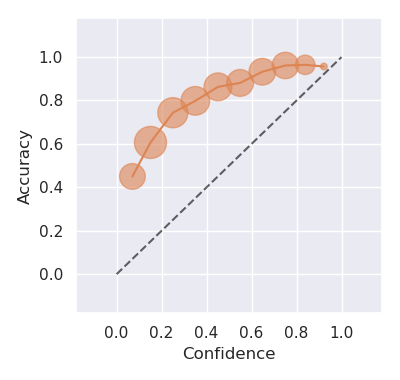}
\subcaption{cSG-HMC}
\end{subfigure}
\begin{subfigure}{0.19\textwidth}
\includegraphics[width=\textwidth, trim={12pt 12pt 12pt 12pt}, clip]{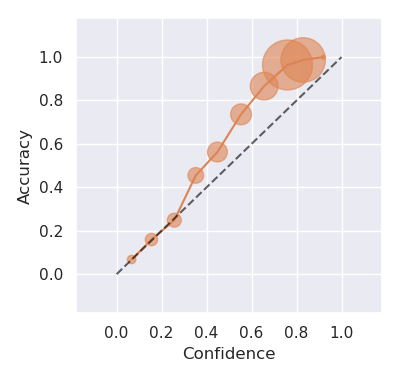}
\subcaption{MCD}
\end{subfigure}
\begin{subfigure}{0.19\textwidth}
\includegraphics[width=\textwidth, trim={12pt 12pt 12pt 12pt}, clip]{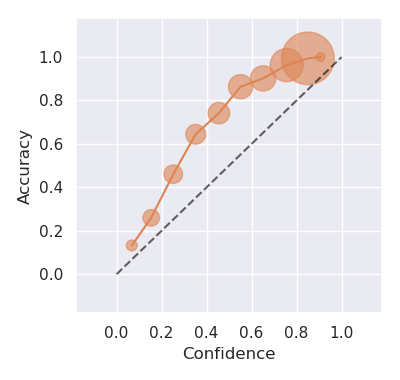}
\subcaption{DE}
\end{subfigure}
\begin{subfigure}{0.19\textwidth}
\includegraphics[width=\textwidth, trim={12pt 12pt 12pt 12pt}, clip]{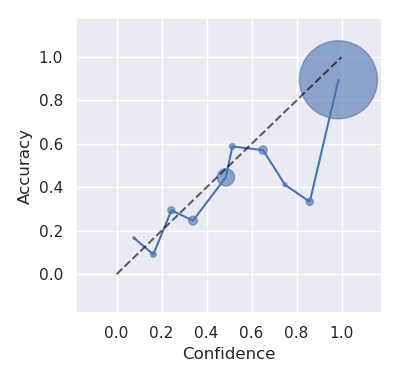}
\subcaption{QUAM}
\end{subfigure}
\caption[Calibration on ImageNet]{Calibration: confidence vs. accuracy based on (weighted) average predictive distribution of different uncertainty quantification methods. Point size indicates number of samples in the bin.}
\label{fig:calibration}
\end{figure*}

\begin{figure}[b!]
\centering
\includegraphics[width=\textwidth, height=0.9\textheight]{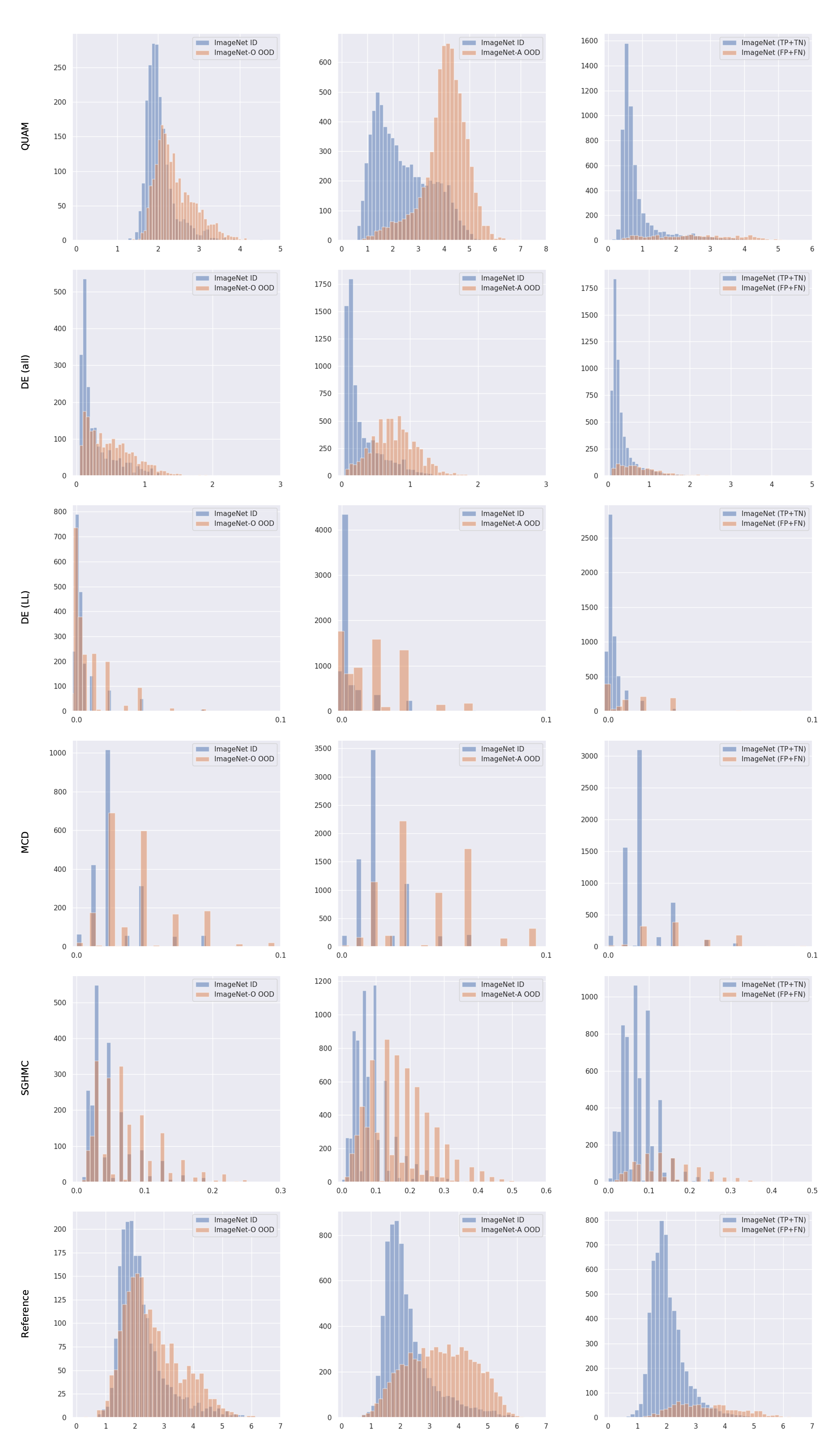}
\caption[Histograms ImageNet]{ImageNet: Histograms of uncertainty scores calculated for test set samples of the specified datasets.}
\label{fig:imagenet_histogram}
\end{figure}

\begin{figure*}[b!]
\centering
\begin{subfigure}{0.49\textwidth}
\includegraphics[width=\textwidth, trim={0 0 0 13pt}, clip]{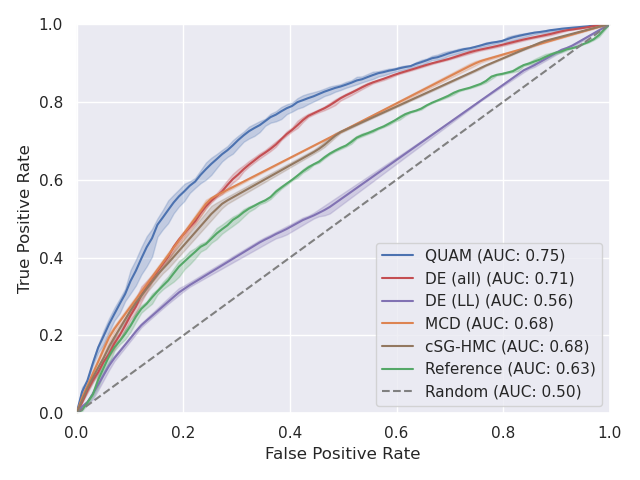}
\subcaption{OOD detection}
\end{subfigure}
\begin{subfigure}{0.49\textwidth}
\includegraphics[width=\textwidth, trim={0 0 0 13pt}, clip]{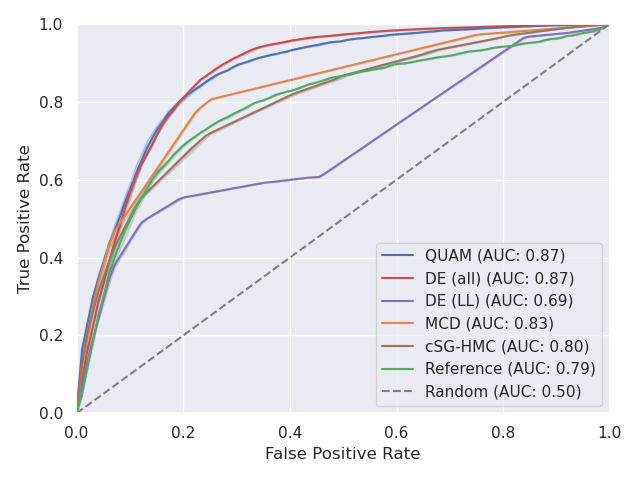}
\subcaption{Adversarial example detection}
\end{subfigure}
\begin{subfigure}{0.49\textwidth}
\includegraphics[width=\textwidth]{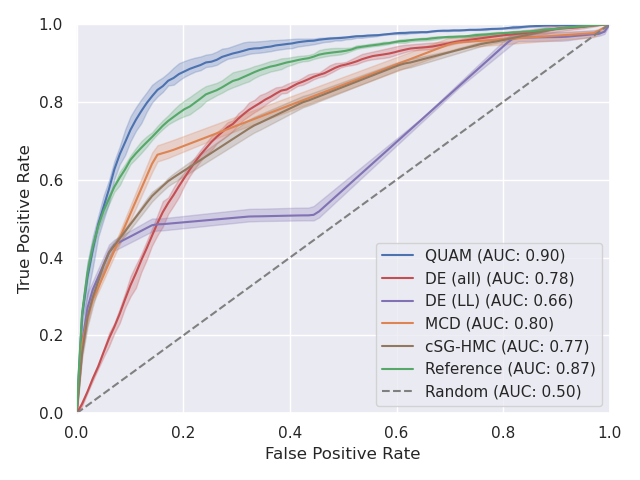}
\subcaption{Misclassification}
\end{subfigure}
\begin{subfigure}{0.49\textwidth}
\includegraphics[width=\textwidth]{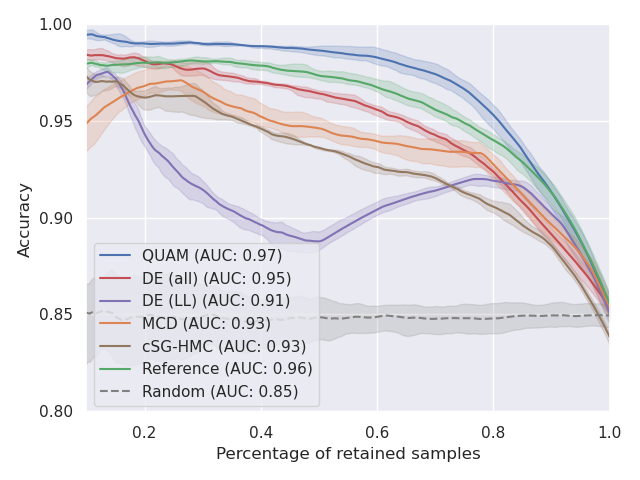}
\subcaption{Selective prediction}
\end{subfigure}
\caption[Detailed results of ImageNet experiments]{ImageNet-1K OOD detection results on ImageNet-O, adversarial example detection results on ImageNet-A, misclassification detection and selective prediction results on the validation dataset. ROC curves using the epistemic uncertainty of a given, pre-selected model (as in Eq.~\eqref{eq:epistemic_setting_b}) to distinguish between (a) the ImageNet-1K validation dataset and ImageNet-O, (b) the ImageNet-1K validation dataset and ImageNet-A and (c) the reference model's correctly and incorrectly classified samples. (d) Accuracy of reference model on subset composed of samples that exhibit lowest epistemic uncertainty.}
\label{fig:roc}
\end{figure*}

\renewcommand{\arraystretch}{0.85}
\setlength{\tabcolsep}{5pt}
\begin{table}
\centering
\caption[Detailed results of ImageNet experiments]{Detailed results of ImageNet OOD detection, adversarial example detection and misclassification experiments, reporting AUROC, AUPR and FPR@TPR=95\% for individual splits.\\
}
\label{tab:appendix:detailed_imagenet}
\begin{tabular}{ccccccc}
\hline \\ [-6pt]
OOD dataset / task &  Method & Split & $\uparrow$ AUPR  & $\uparrow$ AUROC & $\downarrow$ FPR@TPR=95\% \\[2pt] \hline  \\ [-6pt]
\multirow{18}{*}{ImageNet-O} & \multirow{3}{*}{Reference} & I & 0.615 & 0.629 & 0.952 \\
  &  & II &  0.600 & 0.622 & 0.953 \\
  &  & III &  0.613 & 0.628 & 0.954 \\
  & \multirow{3}{*}{cSG-HMC} & I & 0.671 & 0.682 & 0.855 \\
  &  & II & 0.661 & 0.671 & 0.876  \\
  &  & III & 0.674 & 0.679 & 0.872 \\
  & \multirow{3}{*}{MC dropout} & I &  0.684 & 0.681 & 0.975 \\
  &  &  II &  0.675 & 0.677 & 0.974 \\
  & & III &  0.689 & 0.681 & 0.972 \\
  & \multirow{3}{*}{Deep Ensembles (LL)} & I &  0.573 & 0.557 & 0.920 \\
  &  &  II &  0.566 & 0.562 & 0.916 \\
  &  & III &  0.573 & 0.566 & 0.928 \\
  & \multirow{3}{*}{Deep Ensembles (all)} & I &  0.679 & 0.713 & 0.779 \\
  &  &  II &  0.667 & 0.703 & 0.787 \\
  &  & III &  0.674 & 0.710 & 0.786 \\
  & \multirow{3}{*}{QUAM} & I &  0.729 & 0.758 & 0.766 \\
  &  &  II &  0.713 & 0.740 & 0.786 \\
  &  & III &  0.734 & 0.761 & 0.764 \\
\hline \\ [-6pt]
\multirow{18}{*}{ImageNet-A} & \multirow{3}{*}{Reference} & I &  0.779 & 0.795 & 0.837 \\
  &   &  II &  0.774 & 0.791 & 0.838 \\
  &   & III &  0.771 & 0.790 & 0.844 \\
  & \multirow{3}{*}{cSG-HMC} & I & 0.800 & 0.800 & 0.785 \\
  &  & II & 0.803 & 0.800 & 0.785 \\
  &  & III & 0.799 & 0.798 & 0.783 \\
  & \multirow{3}{*}{MC dropout} & I &  0.835 & 0.828 & 0.748 \\
  &  &  II &  0.832 & 0.828 & 0.740 \\
  &  & III &  0.826 & 0.825 & 0.740 \\
  & \multirow{3}{*}{Deep Ensembles (LL)} & I &  0.724 & 0.687 & 0.844 \\
  &  &  II &  0.723 & 0.685 & 0.840 \\
  &  & IIII &  0.721 & 0.686 & 0.838 \\
  & \multirow{3}{*}{Deep Ensembles (all)} & I &  0.824 & 0.870 & 0.385 \\
  &  &  II &  0.837 & 0.877 & 0.374 \\
  &  & III &  0.832 & 0.875 & 0.375 \\
  & \multirow{3}{*}{QUAM} & I &  0.859 & 0.875 & 0.470 \\
  &  &  II &  0.856 & 0.872 & 0.466 \\
  &  & III &  0.850 & 0.870 & 0.461 \\
\hline \\ [-6pt]
\multirow{18}{*}{Misclassification} & \multirow{3}{*}{Reference} & I &  0.623 & 0.863 & 0.590 \\
  &   &  II &  0.627 & 0.875 & 0.554 \\
  &   & III &  0.628 & 0.864 & 0.595 \\
  & \multirow{3}{*}{cSG-HMC} & I & 0.478 & 0.779 & 0.755 \\
  &  & II & 0.483 &	0.779 & 0.752 \\
  &  & III & 0.458 & 0.759 & 0.780 \\
  & \multirow{3}{*}{MC dropout} & I &  0.514 & 0.788 & 0.719 \\
  &  &  II &  0.500 & 0.812 & 0.704 \\
  &  & III &  0.491 & 0.788 & 0.703 \\
  & \multirow{3}{*}{Deep Ensembles (LL)} & I &  0.452 & 0.665 & 0.824 \\
  &  &  II &  0.421 & 0.657 & 0.816 \\
  &  & III &  0.425 & 0.647 & 0.815 \\
  & \multirow{3}{*}{Deep Ensembles (all)} & I &  0.282 & 0.770 & 0.663 \\
  &  &  II &  0.308 & 0.784 & 0.650 \\
  &  & III &  0.310 & 0.786 & 0.617 \\
  & \multirow{3}{*}{QUAM} & I &  0.644 & 0.901 & 0.451 \\
  &  &  II &  0.668 & 0.914 & 0.305 \\
  &  & III &  0.639 & 0.898 & 0.399 \\
\hline \\ [-6pt]
\end{tabular}
\end{table}
\renewcommand{\arraystretch}{1}

\clearpage
\subsection{Comparing Mechanistic Similarity of Deep Ensembles vs. Adversarial Models}\label{sec:apx_mechanistic}

The experiments were performed on MNIST, EMNIST, and KMNIST test datasets, 
using $512$ images of each using Deep Ensembles, and the reference model $\Bw$, trained on MNIST. 
Results are depicted in Fig.~\ref{fig:res:emnist:mecha}.
For each image and each ensemble member, gradients were integrated over $64$ steps from $64$ different random normal sampled baselines for extra robustness \citep{Sundararajan:17}.  
Since the procedure was also performed on the OOD sets as well as our general focus on uncertainty estimation, no true labels were used for the gradient computation. 
Instead, predictions of ensemble members for which the attributions were computed were used as targets. 
Principal Component Analysis (PCA) was performed for the attributions of each image separately, where for each pixel the attributions from different ensemble members were treated as features.
The ratios of explained variance, which are normalized to sum up to one, are collected from each component. 
If all ensemble members would utilize mutually exclusive features for their prediction, all components would be weighted equally, leading to a straight line in the plots in the top row in Fig.~\ref{fig:res:emnist:mecha}.
Comparatively high values of the first principal component to the other components in the top row plots in Fig.~\ref{fig:res:emnist:mecha} indicate low diversity in features used by Deep Ensembles.

The procedure was performed similarly for an ensemble of adversarial models. 
The main difference was that for each image an ensemble produced as a result of an adversarial model search on that specific image was used. 
We observe, that ensembles of adversarial models utilize more dissimilar features, indicated by
the decreased variance contribution of the first principal component.
This is especially strong for ID data, but also noticeable for OOD data.

\subsection{Prediction Space Similarity of Deep Ensembles and Adversarial Models}\label{sec:apx_outputsimilarity}

In the following, ensembles members and adversarial models are analyzed in prediction space. 
We used the same Deep Ensembles as the one trained on MNIST for the OOD detection task described in Sec.~\ref{apx:sec:mnist}. 
Also, 10 adversarial models were retrieved from the reference model $\Bw$ and a single OOD sample (KMNIST), following the same procedure as described in Sec.~\ref{apx:sec:mnist}.

For the analysis, PCA was applied to the flattened softmax output vectors of each of the 20 models applied to ID validation data. 
The resulting points represent the variance of the model’s predictions across different principal components \citep{Fort:19}. 
The results in Fig.~\ref{fig:res:PCA_val_samples} show, that the convex hull of blue points representing adversarial models, in general, is much bigger than the convex hull of orange points representing ensemble members across the first four principal components, which explain 99.99\% of the variance in prediction space. 
This implies that even though adversarial models achieve similar accuracy as Deep Ensembles on the validation set, they are capable of capturing more diversity in prediction space.

\begin{figure}
\centering
\includegraphics[width=\textwidth, trim={0 2pt 0 2pt}, clip]{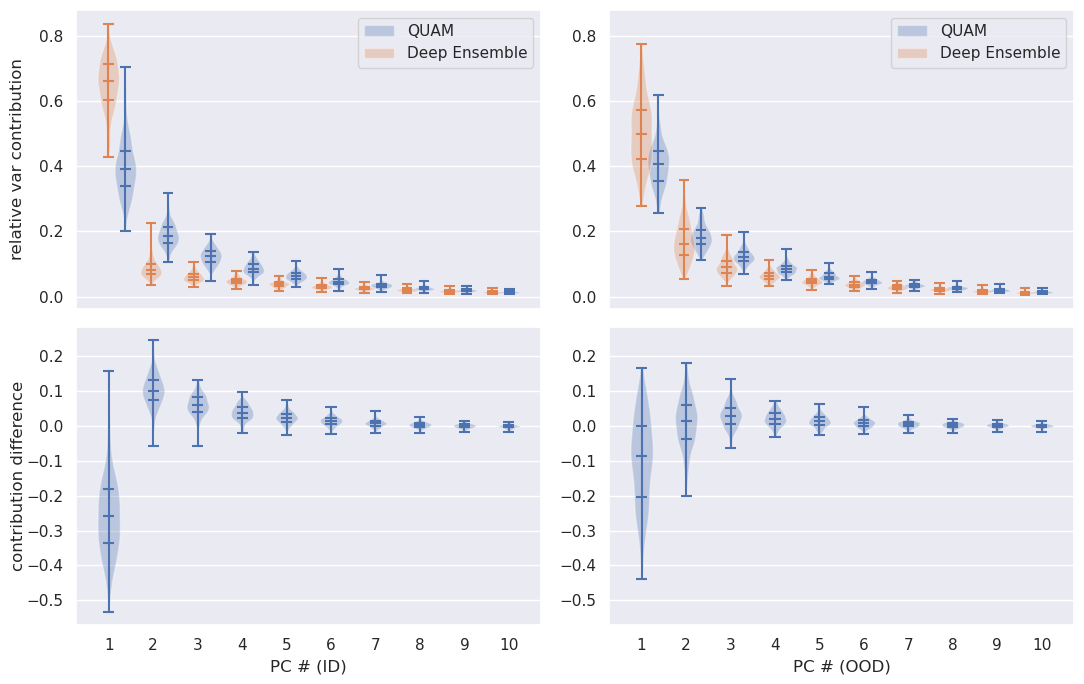}
\caption[Comparing mechanistic similarity of Deep Ensembles vs. Adversarial Models]{The differences between significant component distribution are marginal on OOD data but pronounced on the ID data. The ID data would be subject to optimization by gradient descent during training, therefore the features are learned greedily and models are similar to each other mechanistically. We observe, that the members of Deep Ensembles show higher mechanistic similarity than the members of ensembles obtained from adversarial model search. }
\label{fig:res:emnist:mecha}

\bigskip

\includegraphics[width=\textwidth, trim={0 2pt 0 2pt}, clip]{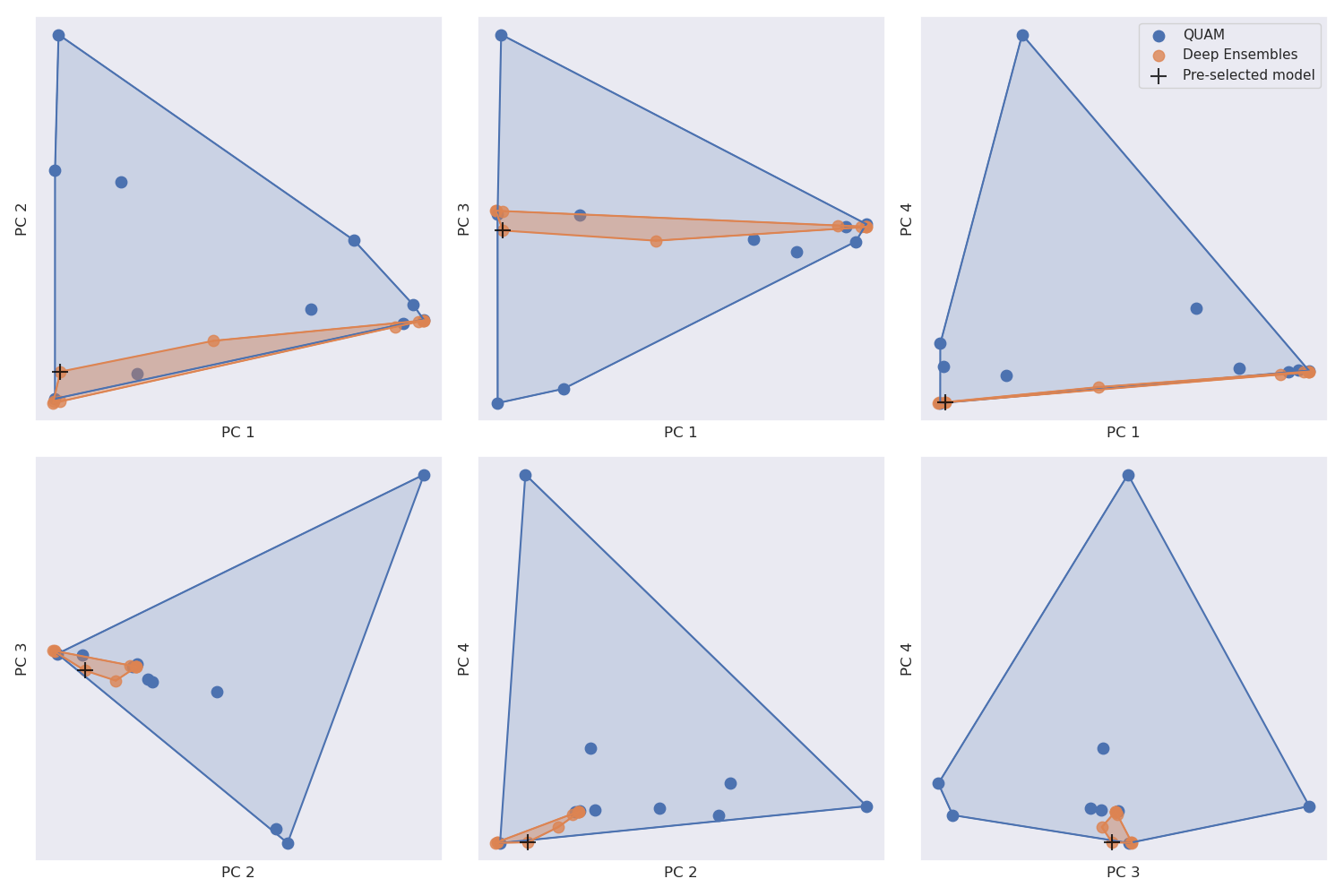}
\caption[Prediction Space Similarity of Deep Ensembles vs. Adversarial Models]{Convex hull of the down-projected softmax output from 10 Ensemble Members (orange) as well as 10 adversarial models (blue). PCA is used for down-projection, all combinations of the first four principal components (99.99\% variance explained) are plotted against each other. Softmax outputs are obtained on a batch of 10 random samples from the ID validation dataset. The black cross marks the given, pre-selected model $\Bw$.}
\label{fig:res:PCA_val_samples}
\end{figure}

\subsection{Computational Expenses}

\paragraph{Experiments on Synthetic Datasets}

The example in Sec.~\ref{apx:sec:simplex} was computed within half an hour on a GTX 1080 Ti.
Experiments on synthetic datasets shown in Sec.~\ref{apx:sec:synthetic} were also
performed on a single GTX 1080 Ti.
Note that the HMC baseline took approximately 14 hours on 36 CPU cores for the classification task.
All other methods except QUAM finish within minutes.
QUAM scales with the number of test samples.
Under the utilized parameters and 6400 test samples, QUAM computation took approximately 6 hours on a single GPU and under one hour for the regression task, where the number of test points is much smaller.

\paragraph{Experiments on Vision Datasets}

Computational Requirements for the vision domain experiments depend a lot on the exact utilization of the baseline methods.
While Deep Ensembles can take a long time to train, depending on the ensemble size,
we utilized either pre-trained networks for ensembling or only trained last layers,
which significantly reduces the runtime.
Noteworthy, MC-dropout can result in extremely high runtimes depending on the number of
forward passes and depending on the realizable batch size for inputs.
The same holds for SG-HMC.
Executing the QUAM experiments on MNIST (Sec.~\ref{apx:sec:mnist}) took a grand total of around 120 GPU-hours on a variety of mostly older generation and low-power GPUs (P40, Titan V, T4), corresponding to roughly 4 GPU-seconds per sample.
Executing the experiments on ImageNet (Sec.~\ref{apx:sec:imagenet}) took about 100 GPU-hours on a mix of A100 and A40 GPUs, corresponding to around 45 GPU-seconds per sample.
The experiments presented in Sec.\ref{sec:apx_mechanistic} and \ref{sec:apx_outputsimilarity} took around 2 hours each on 4 GTX 1080 Ti.

\end{document}